\documentclass{article}

     \PassOptionsToPackage{numbers, compress}{natbib}
\usepackage[final]{neurips_2020}

\usepackage{graphicx}
\usepackage{amsmath}
\usepackage{amssymb}
\usepackage{amsthm}
\usepackage{caption}
\usepackage{subcaption}
\usepackage[colorlinks,linktocpage]{hyperref}
\hypersetup{
  colorlinks,
  linkcolor={red!50!black},
  citecolor={blue!50!black},
  urlcolor={blue!80!black}
}
\usepackage{bbm}
\usepackage{float}
\usepackage{url}
\usepackage{comment}
\usepackage{xspace}
\usepackage{enumitem}

\usepackage{textcomp}

\usepackage[capitalize]{cleveref}  
\crefname{appendix}{App.}{Apps.}
\crefname{proposition}{Prop.}{Props.}
\crefname{section}{Sec.}{Secs.}
\crefname{theorem}{Thm.}{Thms.}
\crefname{definition}{Def.}{Defs.}
\crefname{assumption}{Assumption}{Assumptions}
\crefname{equation}{}{}
\crefname{talign}{}{}
\usepackage{datetime}
\RequirePackage[usenames,dvipsnames]{color}
\usepackage{mathrsfs}
\usepackage{autonum} 
\usepackage{todonotes}
\allowdisplaybreaks[3]

\def\balign#1\ealign{\begin{align}#1\end{align}}
\def\baligns#1\ealigns{\begin{align*}#1\end{align*}}
\def\balignat#1\ealign{\begin{alignat}#1\end{alignat}}
\def\balignats#1\ealigns{\begin{alignat*}#1\end{alignat*}}
\def\bitemize#1\eitemize{\begin{itemize}#1\end{itemize}}
\def\benumerate#1\eenumerate{\begin{enumerate}#1\end{enumerate}}

\newenvironment{talign*}
 {\csname align*\endcsname}
 {\endalign}
\newenvironment{talign}
 {\csname align\endcsname}
 {\endalign}

\def\balignst#1\ealignst{\begin{talign*}#1\end{talign*}}
\def\balignt#1\ealignt{\begin{talign}#1\end{talign}}
\newcommand{\qtext}[1]{\quad\text{#1}\quad} 

\let\originalleft\left
\let\originalright\right
\renewcommand{\left}{\mathopen{}\mathclose\bgroup\originalleft}
\renewcommand{\right}{\aftergroup\egroup\originalright}

\def\tinycitep*#1{{\tiny\citep*{#1}}}
\def\tinycitealt*#1{{\tiny\citealt*{#1}}}
\def\tinycite*#1{{\tiny\cite*{#1}}}
\def\smallcitep*#1{{\scriptsize\citep*{#1}}}
\def\smallcitealt*#1{{\scriptsize\citealt*{#1}}}
\def\smallcite*#1{{\scriptsize\cite*{#1}}}

\def\mbb#1{\mathbb{#1}}
\def\mc#1{\mathcal{#1}}
\def\mrm#1{\mathrm{#1}}
\def\tbf#1{\textbf{#1}}

\def\reals{\mathbb{R}} %
\def\R{\mathbb{R}}

\def\naturals{\mathbb{N}} %
\def\N{\mathbb{N}}
\def\<{\left\langle} %
\def\>{\right\rangle}

\def\iff{\Leftrightarrow}

\def\defeq{\triangleq} %
\def\half{\frac{1}{2}}

\newcommand{\textfrac}[2]{{\textstyle\frac{#1}{#2}}}

\newcommand{\inner}[2]{\langle{#1},{#2}\rangle} %
\def\indic#1{\mathbbm{1}\left[{#1}\right]} %
\def\staticindic#1{\mathbbm{1}[{#1}]} %
\def\E{\mbb{E}} %
\def\Earg#1{\E\left[{#1}\right]}
\def\Esub#1{\E_{#1}}
\def\Esubarg#1#2{\E_{#1}\left[{#2}\right]}
\def\P{\mbb{P}} %

\def\Var{\mrm{Var}} %

\def\Varsubarg#1#2{\Var_{#1}\left({#2}\right)}

\newcommand{\Ber}{\textnormal{Ber}}

\newcommand{\todist}{\stackrel{d}{\to}}
\newcommand{\toprob}{\stackrel{p}{\to}}
\newcommand{\toL}[1]{\stackrel{L^{#1}}{\to}}

\newcommand{\iid}{\textrm{i.i.d.}\xspace}

\providecommand{\sign}{\mathop\mathrm{sign}}
\ifdefined\nonewproofenvironments\else
\ifdefined\ispres\else
\newtheorem{theorem}{Theorem}
\newtheorem{lemma}{Lemma}

\newtheorem{definition}{Definition}

\newtheorem{example}{Example}

\renewenvironment{proof}{\noindent\textbf{Proof}\hspace*{1em}}{\qed\\}
\newenvironment{proof-sketch}{\noindent\textbf{Proof Sketch}
  \hspace*{1em}}{\qed\bigskip\\}
\newenvironment{proof-idea}{\noindent\textbf{Proof Idea}
  \hspace*{1em}}{\qed\bigskip\\}
\newenvironment{proof-of-lemma}[1][{}]{\noindent\textbf{Proof of Lemma {#1}}
  \hspace*{1em}}{\qed\\}
\newenvironment{proof-of-theorem}[1][{}]{\noindent\textbf{Proof of Theorem {#1}}
  \hspace*{1em}}{\qed\\}
\newenvironment{proof-attempt}{\noindent\textbf{Proof Attempt}
  \hspace*{1em}}{\qed\bigskip\\}

\fi

\newtheorem{proposition}{Proposition}

\fi

\newcommand{\var}{\ensuremath{\operatorname{Var}}}

\def\R{\mathbb{R}}

\renewcommand{\N}{\mathcal{N}}

\newcommand{\Zset}{\mathcal{Z}}
\newcommand{\D}{\mathcal{D}}
\newcommand{\Rhat}{\hat{R}_n}
\newcommand{\Rhatn}{\hat{R}_n}
\newcommand{\Rcondcv}{R_n}
\newcommand{\Rcondcvn}{R_n}

\newcommand{\lstab}{\gamma_{loss}} %
\newcommand{\msstab}{\gamma_{ms}} %
\newcommand{\sig}{\sigma}
\newcommand{\sigin}{\hat{\sigma}_{n,in}} %
\newcommand{\siginapprox}{\hat{\sig}_{n,in,approx}}
\newcommand{\sigout}{\hat{\sigma}_{n,out}}

\newcommand{\sigoutapprox}{\hat{\sig}_{n,out,approx,1}}
\newcommand{\sigoutapproxx}{\hat{\sig}_{n,out,approx,2}}
\newcommand{\ntrain}{m} %
\newcommand{\err}[2]{\text{Err}(\alg_{#1}) < \text{Err}(\alg_{#2})} %

\newtheorem{remark}{Remark}

\graphicspath{{./fig/}}

\numberwithin{equation}{section}

\title{Cross-validation Confidence Intervals for Test Error}
\author{Pierre Bayle\thanks{Equal contribution}\\
Princeton University\\
\texttt{pbayle@princeton.edu}\\
\And  
Alexandre Bayle\footnotemark[1]\\
Harvard University\\
\texttt{alexandre\_bayle@g.harvard.edu}\\
\AND
Lucas Janson\\
Harvard University\\
\texttt{ljanson@fas.harvard.edu}\\
\And
Lester Mackey\\
Microsoft Research New England\\
\texttt{lmackey@microsoft.com}
}
\begin{document}

\maketitle
\begin{abstract}
    This work develops central limit theorems for cross-validation and consistent estimators of its asymptotic variance under weak stability conditions on the learning algorithm. Together, these results provide practical, asymptotically-exact confidence intervals for $k$-fold test error and valid, powerful hypothesis tests of whether one learning algorithm has smaller $k$-fold test error than another.
    These results are also the first of their kind for the popular choice of leave-one-out cross-validation. 
    In our real-data experiments with 
    diverse 
    learning algorithms, the resulting intervals and tests outperform the most popular alternative methods from the literature.

\end{abstract}

\section{Introduction}\label{sec:intro}
Cross-validation (CV) \citep{Stone1974,geisser1975predictive} is a de facto standard for estimating the test error of a prediction rule.
By partitioning a dataset into $k$ equal-sized validation sets, fitting a prediction rule with each validation set held out, evaluating each prediction rule on its corresponding held-out set, and averaging the $k$ error estimates, CV produces an unbiased estimate of the test error with lower variance than a single train-validation split could provide.
However, these properties alone are insufficient for high-stakes applications in which the uncertainty of an error estimate impacts decision-making.
In predictive cancer prognosis and mortality prediction for instance, scientists and clinicians rely on \emph{test error confidence intervals (CIs)} based on CV and other repeated sample splitting estimators to avoid spurious findings and improve reproducibility \citep{michiels-etal:2005,pirracchio2015mortality}.  
Unfortunately, the CIs most often used have no correctness guarantees and can be severely misleading \citep{jiang-etal:2008}.
The difficulty comes from the dependence across the $k$ averaged error estimates: if the estimates were independent, one could derive an \emph{asymptotically-exact CI} (i.e., a CI with coverage converging exactly to the target level) for test error using a standard central limit theorem.
However, the error estimates are seldom independent, due to the overlap amongst training sets and between different training and validation sets.
Thus, new tools are needed to develop valid, informative CIs based on CV.

The same uncertainty considerations are relevant when comparing two machine learning methods: before selecting a prediction rule for deployment, one would like to be confident that its test error is better than a baseline or an available alternative.
The standard practice amongst both method developers and consumers is to conduct a formal \emph{hypothesis test} for a difference in test error between two prediction rules \cite{dietterich:1998,lim-etal:2000,nadeau-bengio:2003,bouckaert-frank:2004,demsar:2006}.
Unfortunately, the most popular tests from the literature like the cross-validated $t$-test \citep{dietterich:1998}, the repeated train-validation $t$-test \citep{nadeau-bengio:2003}, and the $5 \times 2$ CV test \citep{dietterich:1998} have no correctness guarantees and hence can produce misleading conclusions.
The difficulty parallels that of the confidence interval setting: standard tests assume independence and do not appropriately account for the dependencies across CV error estimates. 
Therefore, new tools are also needed to develop valid, powerful tests for test error improvement based on CV.

\paragraph{Our contributions}
To meet these needs, we characterize the asymptotic distribution of CV error and develop consistent estimates of its variance under weak stability conditions on the learning algorithm.
Together, these results provide practical, asymptotically-exact confidence intervals for test error as well as valid and powerful hypothesis tests of whether one learning algorithm has smaller test error than another.
In more detail, we prove in \cref{sec:clt} that $k$-fold CV error is asymptotically normal around its test error under an abstract asymptotic linearity condition.
We then give in \cref{sec:suff} two different stability conditions that hold for large classes of learning algorithms and losses and that  individually imply the asymptotic linearity condition. 
In \cref{sec:var-est-and-CI}, we propose two estimators of the asymptotic variance of CV and prove them to be consistent under similar stability conditions; our second estimator accommodates any choice of $k$ and appears to be the first consistent variance estimator for leave-one-out CV.  
To validate our theory in \cref{sec:experiments}, we apply our intervals and tests to a diverse collection of classification and regression methods on particle physics and flight delay data and observe consistent improvements in width and power over the most popular alternative methods from the literature.
\paragraph{Related work}
Despite the ubiquity of CV, 
we are only aware of three prior efforts to 
characterize the precise distribution of cross-validation error.
The CV central limit theorem (CLT) of \citet{SD-MvdL:2005} requires considerably stronger assumptions than our own and is not paired with the consistent estimate of variance needed to construct a valid confidence interval or test. 
\citet{ledell2015computationally} derive both a CLT and a consistent estimate of variance for CV, but these apply only to the area under the ROC curve (AUC) performance measure.
Finally, in very recent work, \citet{MA-WZ:2020} derive a CLT and a consistent estimate of variance for CV under more stringent assumptions than our own.
We compare our results with each of these works in detail in \cref{sec:comp-lit}.
We note also that another work \citep{lei2019cross} aims to test the difference in test error between two learning algorithms using cross-validation but only proves the validity of their procedure for a single train-validation split rather than for CV.
Many other works have studied the problem of bounding or estimating the variance of the cross-validation error \citep{Blum-ea:1999,nadeau-bengio:2003,bengio-grandvalet:2004,markatou:2005,SK-RK-SV:2011,RK-ea:2013,celisse-guedj:2016,pmlr-v98-abou-moustafa19a,karim_csaba_aaai19}, but none have established the consistency of their variance estimators. Among these, \citet{SK-RK-SV:2011,RK-ea:2013,celisse-guedj:2016} introduce relevant notions of algorithmic stability to which we link our results in \cref{sec:loss-stability}.
Moreover, non-asymptotic CIs can be derived from the CV concentration inequalities of \citep{bousquet-elisseeff:2002,cornec:2010,celisse-guedj:2016,pmlr-v98-abou-moustafa19a,karim_csaba_aaai19}, but these CIs are more difficult to deploy as they require (1) stronger stability assumptions, (2) a known upper bound on stability, and (3) either a known upper bound on the loss or a known uniform bound on the covariates and a known sub-Gaussianity constant for the response variable.
In addition, the reliance on somewhat loose inequalities typically leads to overly large, relatively uninformative CIs.
For example, we implemented the ridge regression CI of \citep[Thm.\ 3]{celisse-guedj:2016} for the regression experiment of \cref{sec:ci-experiment} (see \cref{sec:concentration-inequalities}).
When the features are standardized, the narrowest concentration-based interval is 91 times wider than our widest CLT interval in \cref{fig:test-error-CI-reg}. Without standardization, the narrowest concentration-based interval is $5\times 10^{14}$ times wider.

\paragraph{Notation}
Let $\todist$, $\toprob$, and $\toL{q}$ for $q > 0$, denote convergence in distribution, in probability, and in $L^q$ norm (i.e., $X_n \toL{q} X \iff \E[|X_n - X|^q] \to 0$), respectively.
For each $m,n\in\naturals$ with $m\le n$, we define the set $[n] \defeq \{1,\dots, n\}$ and the vector $m \colon \hspace{-0.08cm} n\defeq (m,\dots,n)$.
When considering independent random elements $(X,Y)$, we use $\Esub{X}$ and $\Var_X$ to indicate expectation and variance only over $X$, respectively; that is,  $\Esubarg{X}{f(X,Y)} \defeq \E[f(X,Y)\mid Y]$ and $\Var_X (f(X,Y)) \defeq \Var(f(X,Y)\mid Y)$. %
We will refer to the Euclidean norm of a vector as the $\ell^2$ norm in the context of $\ell^2$ regularization.

\section{A Central Limit Theorem for Cross-validation}\label{sec:clt}
In this section, we present a new central limit theorem for $k$-fold cross-validation. Throughout, any asymptotic statement will take $n\to\infty$, and while we allow the number of folds $k_n$ to depend on the sample size $n$ (e.g., $k_n=n$ for leave-one-out cross-validation), we will write $k$ in place of $k_n$ to simplify our notation. We will also present our main results assuming that $k$ evenly divides $n$, but we address the indivisible setting in the appendix. %

Hereafter, we will refer to a sequence $(Z_i)_{i\geq 1}$ of random datapoints taking values in a set $\Zset$. Notably, $(Z_i)_{i\geq 1}$ need not be independent or identically distributed. We let $Z_{1:n}$ designate the first $n$ points, and, for any vector $B$ of indices in $[n]$, we let $Z_B$ denote the subvector of $Z_{1:n}$ corresponding to ordered indices in $B$.
We will also refer to \emph{train-validation splits} $(B, B')$. These are vectors of indices in $[n]$ representing the ordered points assigned to the training set and validation set.\footnote{We keep track of index order to support asymmetric learning algorithms like stochastic gradient descent.} As is typical in CV, we will assume that $B$ and $B'$ partition $[n]$, so that every datapoint is either in the training or validation set.

Given a scalar loss function $h_n(Z_i,Z_B)$ and a set of $k$ train-validation splits
$\{(B_j,B_j')\}_{j=1}^k$ with validation indices $\{B_j'\}_{j=1}^k$ partitioning $[n]$ into $k$ folds, 
we will use the \emph{$k$-fold cross-validation error}
\begin{talign}\label{eq:cv-err}
\Rhat 
    \defeq 
    \frac{1}{n}
    \sum_{j=1}^k
    \sum_{i\in B_j'} h_n(Z_i, Z_{B_j})
\end{talign}
to draw inferences about the \emph{$k$-fold test error} 
\begin{talign}\label{eq:cond-cv-err}
    \Rcondcv
    \defeq
    \frac{1}{n}
    \sum_{j=1}^k
    \sum_{i\in B_j'}
    \E[h_n(Z_i, Z_{B_j})\mid Z_{B_j}].
\end{talign}
A prototypical example of $h_n$ is squared error or 0-1 loss, 
\[
h_n(Z_i, Z_{B}) = (Y_i - \hat{f}(X_i; Z_{B}))^2
\qtext{or}
h_n(Z_i, Z_{B}) = \staticindic{Y_i \neq \hat{f}(X_i; Z_{B})},
\]
composed with an algorithm for fitting a prediction rule $\hat{f}(\cdot; Z_{B})$ to training data $Z_{B}$ and predicting the response value of a test point $Z_i = (X_i, Y_i)$.\footnote{For randomized learning algorithms (such as random forests or stochastic gradient descent), all statements in this paper should be treated as holding conditional on the external source of randomness.}
In this setting, the $k$-fold test error is a standard inferential target \citep{Blum-ea:1999, SD-MvdL:2005, SK-RK-SV:2011, RK-ea:2013, MA-WZ:2020} and represents the average test error of the $k$ prediction rules $\hat{f}(\cdot; Z_{B_j})$.
When comparing the performance of two algorithms in \cref{sec:var-est-and-CI,sec:experiments}, we will choose $h_n$ to be the difference between the losses of two prediction rules.

\subsection{Asymptotic linearity of cross-validation}
The key to our central limit theorem is establishing that the $k$-fold CV error asymptotically behaves like the $k$-fold test error plus an average of functions applied to single datapoints. %
The following proposition provides a convenient characterization of this \emph{asymptotic linearity} property.

\begin{proposition}[Asymptotic linearity of $k$-fold CV]\label{cv-asymp-linear}
For any sequence of datapoints $(Z_i)_{i\geq 1}$,
\balignt
\frac{\sqrt{n}}{\sigma_n}(\Rhat-\Rcondcv) 
- \frac{1}{\sigma_n\sqrt{n}}\sum_{i=1}^n \left(\bar{h}_n(Z_i) - \E\left[\bar{h}_n(Z_i)\right]\right)
\toprob \Big(\text{resp.\,} \toL{q}\hspace{-0.08cm}\Big)\; 0
\ealignt
for a function $\bar{h}_n$ with $\sigma_n^2 \defeq \frac{1}{n}\Var(\sum_{i=1}^n \bar{h}_n(Z_i))$ if and only if
\balignt\label{eq:asylincond}
\textfrac{1}{\sigma_n\sqrt{n}}
\sum_{j=1}^k
\sum_{i\in B_j'} 
&\Big(h_n\left(Z_i,Z_{B_j}\right) - \E\left[h_n\left(Z_i,Z_{B_j}\right)\mid Z_{B_j}\right]\\
&- \left(\bar{h}_n(Z_i)
-\E\left[\bar{h}_n(Z_i)\right]\right)\Big)
\toprob \Big(\text{resp.\,}\toL{q}\hspace{-0.08cm}\Big)\; 0,\;
\ealignt
where the parenthetical convergence indicates that the same statement holds when both convergences in probability are replaced with convergences in $L^q$ for the same $q > 0$.
\end{proposition}

Typically, one will choose $\bar{h}_n(z) = \E[h_n(z, Z_{1:n(1-1/k)})]$.
With this choice, we see that the difference of differences in \cref{eq:asylincond} is small whenever $h_n(Z_i,Z_{B_j})$ is close to \emph{either} its expectation given $Z_i$ \emph{or} its expectation given $Z_{B_j}$, but it need not be close to both. 
As the asymptotic linearity condition \cref{eq:asylincond} is still quite abstract, we devote all of \cref{sec:suff} to establishing sufficient conditions for \cref{eq:asylincond} that are interpretable, broadly applicable, and simple to verify.
\cref{cv-asymp-linear} follows from a more general asymptotic linearity characterization for repeated sample-splitting estimators proved in \cref{sec:proof-cv-asymp-linear-general}. 

\subsection{From asymptotic linearity to asymptotic normality}
So far, we have assumed nothing about the dependencies amongst the datapoints $Z_i$.
If we additionally assume that the datapoints are \iid, the average $\frac{1}{\sigma_n\sqrt{n}}\sum_{i=1}^n \left(\bar{h}_n(Z_i) - \E\left[\bar{h}_n(Z_i)\right]\right)$ converges to a standard normal under a mild integrability condition, and we obtain the following CLT for CV.

\begin{theorem}[Asymptotic normality of $k$-fold CV with i.i.d.\ data]\label{iid-cv-normal}
In the notation of \cref{cv-asymp-linear},
suppose that the datapoints  $(Z_i)_{i\geq 1}$ are i.i.d.\ copies of a random element $Z_0$ and that the sequence of
$(\bar{h}_n(Z_0)-\E[\bar{h}_n(Z_0)])^2/\sigma_n^2$ with $\sigma_n^2 = \var(\bar{h}_n(Z_0))$
is uniformly integrable (UI).
If the asymptotic linearity condition \cref{eq:asylincond} holds in probability
then
\balignt \frac{\sqrt{n}}{\sigma_n}(\Rhat-\Rcondcv)\todist \N(0,1).\ealignt
\end{theorem}
\cref{iid-cv-normal} is a special case of a more general result, proved in \cref{sec:proof-iid-cv-normal},  that applies when the datapoints are independent but not necessarily identically distributed.
A simple sufficient condition for the required uniform integrability is that $\sup_n \E[|(\bar{h}_n(Z_0)-\E[\bar{h}_n(Z_0)])/\sigma_n|^\alpha] < \infty$ for some $\alpha>2$.
This holds, for example, whenever
$\bar{h}_n(Z_0)$ has uniformly bounded $\alpha$ moments (e.g., the 0-1 loss has all moments uniformly bounded) and does not converge to a degenerate distribution.
We now turn our attention to the asymptotic linearity condition.

\section{Sufficient Conditions for Asymptotic Linearity}\label{sec:suff}
\subsection{Asymptotic linearity from loss stability} \label{sec:loss-stability}
Our first result relates the asymptotic linearity of CV to a specific notion of algorithmic stability, termed \emph{loss stability}.
\begin{definition}[Mean-square stability and loss stability]
\label{def:msstab}
For $\ntrain > 0$, let $Z_0$ and $Z_0', Z_1, \ldots, Z_{\ntrain}$ be \iid test and training points with $Z^{\backslash i}_{1:\ntrain}$ representing $Z_{1:\ntrain}$ with $Z_i$ replaced by $Z_0'$. For any function $h : \Zset \times \Zset^{\ntrain} \to \reals$, the \emph{mean-square stability} \citep{SK-RK-SV:2011} is defined as
\begin{talign}\label{eq:ms_stab}
\msstab(h) \defeq \frac{1}{\ntrain}\sum_{i=1}^\ntrain \mathbb{E}[(h(Z_0,Z_{1:\ntrain})-h(Z_0,Z^{\backslash i}_{1:\ntrain}))^2] 
\end{talign}
and the \emph{loss stability} \citep{RK-ea:2013} as $\lstab(h) \defeq \msstab(h')$, where 
\balignt\label{eq:hprime}
h'(Z_0,Z_{1:\ntrain}) \defeq h(Z_0,Z_{1:\ntrain})-\E\left[h(Z_0,Z_{1:\ntrain})\mid Z_{1:\ntrain}\right].
\ealignt
\end{definition}
\citet{RK-ea:2013} introduced loss stability 
to bound the variance of CV in terms of the variance of a single hold-out set estimate.
Here we show that a suitable decay in loss stability is also sufficient for $L^2$ asymptotic linearity, via a non-asymptotic bound on the departure from linearity.

\begin{theorem}[Approximate linearity from loss stability]\label{asymp-from-lstability}
In the notation of \cref{cv-asymp-linear} and \cref{def:msstab}, suppose that the datapoints $(Z_i)_{i\geq 1}$ are i.i.d.\ copies of a random element $Z_0$.
Then
\begin{talign}\label{eq:suffcondforasylincond}
\Var(\frac{1}{\sqrt{n}}\sum_{j=1}^k\sum_{i\in B_j'}
(h'_n(Z_i,Z_{B_j})-\E[h'_n(Z_i,Z_{B_j})\mid Z_i]))
\leq \frac{3}{2}n\left(1-\frac{1}{k}\right)\lstab(h_n).
\end{talign}
Hence the $L^2$ asymptotic linearity condition \cref{eq:asylincond} holds with $\bar{h}_n(z)=\E[h_n(z,Z_{1:n(1-1/k)})]$
if the loss stability satisfies $\lstab(h_n) = o(\sigma_n^2/n)$.
\end{theorem}
The proof of \cref{asymp-from-lstability} is given in \cref{sec:proof-asymp-from-lstability}.
Recall that in a typical learning context, we have $h_n(Z_0, Z_{1:m}) =  \ell(Y_0, \hat{f}(X_0; Z_{1:m}))$ for a fixed loss $\ell$, a learned prediction rule $\hat{f}(\cdot ; Z_{1:m})$,  a test point $Z_0 = (X_0, Y_0)$, and $\ntrain=n\left(1-1/k\right)$.
When $\hat{f}(\cdot; Z_{1:m})$ converges to an imperfect prediction rule, we will commonly have $\sigma_n^2 = \Var(\E\left[h_n(Z_0,Z_{1:m})\mid Z_0\right]) = \Omega(1)$ so that $\lstab(h_n)=o(1/n)$ loss stability is sufficient.
However, \cref{asymp-from-lstability} also accommodates the cases of non-convergent $\hat{f}(\cdot; Z_{1:m})$ and of $\hat{f}(\cdot; Z_{1:m})$ converging to a perfect prediction rule, so that
$\sigma_n^2 = o(1)$.

Many learning algorithms are known to enjoy decaying loss stability \citep{bousquet-elisseeff:2002,elisseeff-etal:2005,hardt-etal:2016,celisse-guedj:2016,arsov-etal:2019}, in part because loss stability is upper-bounded by a variety of algorithmic stability notions studied in the literature.
For example, stochastic gradient descent on convex and non-convex objectives \citep{hardt-etal:2016} and the empirical risk minimization of a strongly convex and Lipschitz objective both have $O(1/n)$ \emph{uniform stability} \citep{bousquet-elisseeff:2002} which implies a loss stability of $O(1/n^2)=o(1/n)$ by \citep[Lem. 1]{SK-RK-SV:2011} and \citep[Lem. 2]{RK-ea:2013}.
However, we emphasize that the loss $h_n$ need not be convex and need not coincide with a loss function used to train a learning method.
Indeed, our stability assumptions also cover $k$-nearest neighbor methods \citep{devroye-wagner:1979a}, decision tree methods \citep{arsov-etal:2019}, and ensemble methods \citep{elisseeff-etal:2005} %
and can even hold when training error is a poor proxy for test error due to overfitting (e.g., 1-nearest neighbor has training error 0 but is still suitably stable \citep{devroye-wagner:1979a}).
In addition, for any loss function, loss stability is upper-bounded by mean-square stability \citep{SK-RK-SV:2011} and all $L^q$ stabilities \citep{celisse-guedj:2016} for $q\geq 2$.
For bounded loss functions such as the 0-1 loss, loss stability is also weaker than hypothesis stability (also called $L^1$ stability) \citep{devroye-wagner:1979a, kearns-ron:1999}, weak-hypothesis stability \citep{devroye-wagner:1979b}, and weak-$L^1$ stability \citep{kutin-niyogi:2002}. 

\subsection{Asymptotic linearity from conditional variance convergence}\label{sec:asymp-from-cond-var}
We can also guarantee asymptotic linearity under weaker moment conditions than \cref{asymp-from-lstability} at the expense of stronger requirements on the number of folds $k$.
\begin{theorem}[Asymptotic linearity from conditional variance convergence]\label{asymp-from-cond-var}
In the notation of \cref{cv-asymp-linear}, suppose that the datapoints $(Z_i)_{i\geq 1}$ are i.i.d.\ copies of a random element $Z_0$.
If %
\balignt\label{eq:variance_convergence_assump_k}
\max(k^{q/2}, k^{1-q/2})
    \E\Big[\big(\frac{1}{\sigma_n^2}\Var_{Z_0}\left(h_n(Z_0,Z_{1:n(1-1/k)})-\bar{h}_n(Z_0)\right)\big)^{q/2}\Big] 
    \to 0
\ealignt
for a function $\bar{h}_n$ and $q \in (0,2]$, then $\bar{h}_n$ satisfies the $L^q$ asymptotic linearity condition \cref{eq:asylincond}.
If%
\balignt\label{eq:variance_convergence_assump_prob_k}
\Earg{\min\Big(k, \frac{\sqrt{k}}{\sigma_n}
    \sqrt{\Var_{Z_0}\left(h_n(Z_0,Z_{1:n(1-1/k)})-\bar{h}_n(Z_0)\right)}\Big)} 
    \to 0.
\ealignt
for a function $\bar{h}_n$, then $\bar{h}_n$ satisfies the in-probability asymptotic linearity condition \cref{eq:asylincond}.

\end{theorem}
\begin{remark}
When $k=O(1)$, %
\cref{eq:variance_convergence_assump_prob_k} holds $\iff$ $\frac{1}{\sigma_n}\sqrt{\Var_{Z_0}\left(h_n(Z_0,Z_{1:n(1-1/k)})-\bar{h}_n(Z_0)\right)}
    \toprob 0$.
\end{remark}
\cref{asymp-from-cond-var} follows from a more general statement proved in \cref{sec:proof-asymp-from-cond-var}.
When $k$ is bounded, as in $10$-fold CV, the conditions of \cref{asymp-from-cond-var} are considerably weaker than those of \cref{asymp-from-lstability} (see \cref{sec:proof-cond-var-conv-from-loss-stability}), granting asymptotic linearity whenever the conditional variance converges in probability rather than in $L^2$. Indeed in \cref{sec:proof-loss-vs-mean-square-stability}, we detail a simple learning problem in which the loss stability is infinite but \cref{asymp-from-cond-var,iid-cv-normal} together provide a valid CLT with convergent variance $\sigma_n^2$.

\subsection{Comparison with prior work}\label{sec:comp-lit}

Our sufficient conditions for asymptotic normality are significantly less restrictive and more broadly applicable than the three prior distributional characterizations of CV error \citep{SD-MvdL:2005,ledell2015computationally,MA-WZ:2020}. 
In particular, the CLT of \citet[Thm.\ 3]{SD-MvdL:2005} assumes a bounded loss function, excludes the popular case of leave-one-out cross-validation,  and requires the prediction rule to be loss-consistent for a risk-minimizing prediction rule. Similarly, the CLT of \citet[Thm.\ 4.1]{ledell2015computationally} applies only to AUC loss, requires the prediction rule to be loss-consistent for a deterministic prediction rule, and requires a bounded number of folds.

Moreover, in our notation, the recent CLT of \citet[Thm.\ 1]{MA-WZ:2020} restricts focus to learning algorithms that treat all training points symmetrically, assumes that its variance parameter
\begin{align}\label{eq:MAWZ-variance}
\tilde{\sigma}_n^2 \defeq \E[\Var(h_n(Z_0, Z_{1:m}) \mid Z_{1:m})]
\end{align}
converges to a non-zero limit, requires mean-square stability $\msstab(h_n)=o(1/n)$, and places a $o(1/n^2)$ constraint on the second-order mean-square stability
\balignt\label{second-order-mss}
\mathbb{E}[((h_n(Z_0,Z_{1:\ntrain})-h_n(Z_0,Z^{\backslash 1}_{1:\ntrain})) - (h_n(Z_0,Z^{\backslash 2}_{1:\ntrain})-h_n(Z_0,Z^{\backslash 1,2}_{1:\ntrain})))^2]
= o(1/n^2),
\ealignt
where $Z^{\backslash 1,2}_{1:\ntrain}$ represents $Z_{1:\ntrain}$ with $Z_1$, $Z_2$ replaced by \iid\ copies $Z'_1$, $Z'_2$. 
\citet{RK-ea:2013} showed that the mean-square stability is always an upper bound for the loss stability required by our \cref{asymp-from-lstability}, and in \cref{sec:proof-loss-vs-mean-square-stability-excess,sec:proof-loss-vs-mean-square-stability} we exhibit two simple learning tasks in which $\lstab(h_n)=O(1/n^2)$ but $\msstab(h_n) = \infty$. 
Furthermore, when $k$ is constant, as in $10$-fold CV, our conditional variance assumptions in \cref{sec:asymp-from-cond-var} are weaker still and hold even for algorithms with infinite loss stability (see \cref{sec:proof-loss-vs-mean-square-stability}).
In addition, our results allow for asymmetric learning algorithms (like stochastic gradient descent), accommodate growing, vanishing, and non-convergent variance parameters $\sigma_n^2$, and do not require the second-order mean-square stability condition \cref{second-order-mss}.

Finally, we note that the asymptotic variance parameter $\sig_n^2$ appearing in \cref{iid-cv-normal} is never larger and sometimes smaller than the variance parameter $\tilde{\sigma}_n^2$ in \citep[Thm.\ 1]{MA-WZ:2020}.
\begin{proposition}[Variance comparison]\label{var-comp}
Let $\sigma_n^2 = \Var(\E[h_n(Z_0, Z_{1:m})\mid Z_0])$ be the variance appearing in \cref{iid-cv-normal}, with the choice $\bar h_n(z) = \E[h_n(z,Z_{1:\ntrain})]$, and $\tilde{\sigma}_n^2 = \E[\Var(h_n(Z_0, Z_{1:m}) \mid Z_{1:m})]$
be the variance parameter of \citep[Eq. (15)]{MA-WZ:2020} for $m = n\left(1 - 1/k\right)$. Then
\balignt
\sigma_n^2\leq \tilde{\sigma}_n^2 \leq \sigma_n^2 + \frac{m}{2} \lstab(h_n),
\ealignt
and the first inequality is strict whenever 
$h(Z_0,Z_{1:\ntrain})-\E\left[h(Z_0,Z_{1:\ntrain})\mid Z_{1:\ntrain}\right]$
depends on $Z_{1:m}$.
\end{proposition}
The proof of \cref{var-comp} can be found in \cref{sec:proof-var-comp}.
In \cref{sec:proof-loss-vs-mean-square-stability}, we present a simple learning task for which our central limit theorem provably holds with $\sigma_n^2$ converging to a non-zero constant, but the CLT in \citep[Eq. (15)]{MA-WZ:2020} is inapplicable because the variance parameter $\tilde{\sigma}_n^2$ is infinite.

\section{
Confidence Intervals and Tests for $k$-fold Test Error
}\label{sec:var-est-and-CI}
A primary application of our central limit theorems is the construction of asymptotically-exact confidence intervals for the unknown $k$-fold test error.
For example, under the assumptions and notation of \cref{iid-cv-normal}, any sample statistic $\hat{\sigma}_n^2$ satisfying relative error consistency, 
$\hat{\sigma}_n^2/\sigma_n^2\toprob 1$, gives rise to an asymptotically-exact $(1-\alpha)$-confidence interval, 
\balignt\label{eq:ci}
C_\alpha \defeq \Rhat \pm q_{1-\alpha/2} \hat{\sigma}_n/\sqrt{n}
\qtext{satisfying}
\lim_{n\to\infty} \P(\Rcondcv \in C_\alpha) = 1 - \alpha,
\ealignt
where $q_{1-\alpha/2}$ is the $(1-\alpha/2)$-quantile of a standard normal distribution.

\newcommand{\alg}{\mc{A}}
A second, related application of our central limit theorems is testing whether, given a dataset $Z_{1:n}$, a $k$-fold partition $\{B_j'\}_{j=1}^k$, and two algorithms $\mc{A}_1$, $\mc{A}_2$ for fitting prediction rules, $\mc{A}_2$ has larger $k$-fold test error than $\mc{A}_1$.
In this circumstance, we may define
\balignt
h_n(Z_0, Z_{B}) = \ell(Y_0, \hat{f}_1(X_0; Z_{B})) - \ell(Y_0, \hat{f}_2(X_0; Z_{B}))
\ealignt
to be the difference of the loss functions of two prediction rules trained on $Z_{B}$ and tested on $Z_0=(X_0,Y_0)$. Our aim is to test whether $\mc{A}_1$ improves upon $\mc{A}_2$ on the fold partition, that is to test the null $H_0: \Rcondcv \geq 0$ against the alternative hypothesis $H_1: \Rcondcv < 0$.
Under the assumptions and notation of \cref{iid-cv-normal}, an asymptotically-exact level-$\alpha$ test is given by\footnote{The test \eqref{eq:one-sided-test} is equivalent to rejecting when the one-sided interval $(-\infty,\Rhat - q_{\alpha} \hat{\sigma}_n/\sqrt{n}\,]$ excludes $0$.} 
\balignt
\label{eq:one-sided-test}
\textsc{reject } H_0 \iff
\Rhat < q_{\alpha} \hat{\sigma}_n/\sqrt{n} 
\ealignt
where $q_{\alpha}$ is the $\alpha$-quantile of a standard normal distribution and $\hat{\sigma}_n^2$ is any variance estimator satisfying relative error consistency, 
$\hat{\sigma}_n^2/\sigma_n^2 \toprob 1$.
Fortunately, our next theorem describes how to compute such a consistent estimate of $\sigma_n^2$ under weak conditions.

\begin{theorem}[Consistent within-fold estimate of asymptotic variance]
\label{consistent-variance-est-in}
In the notation of \cref{iid-cv-normal} with $\ntrain = n(1-1/k)$, $\bar{h}_n(z) = \E[h_n(z,Z_{1:\ntrain})]$, and $k < n$, define the within-fold variance estimator
\balignt
\sigin^2 \defeq
\frac{1}{k}\sum_{j=1}^k
\frac{1}{(n/k)-1}\sum_{i\in B_j'}
\left(
h_n(Z_i,Z_{B_j}) - \frac{k}{n}\sum_{i'\in B_j'}h_n(Z_{i'},Z_{B_j})
\right)^2.
\ealignt
Suppose $(Z_i)_{i\geq 1}$ are \iid\ copies of a random element $Z_0$.
Then
$\sigin^2/\sigma_n^2 \toL{1} 1$
whenever $\lstab(h_n) = o(\sigma_n^2/n)$ and the sequence of
$(\bar{h}_n(Z_0)-\E[\bar{h}_n(Z_0)])^2/\sigma_n^2$
is uniformly integrable (UI). 
Moreover, $\sigin^2/\sigma_n^2 \toL{2} 1$
whenever $\E[((\bar{h}_n(Z_0)-\E[\bar{h}_n(Z_0)])/\sigma_n)^4] = o(n)$ and the \emph{fourth-moment loss stability} $\gamma_{4}(h_n') \defeq \frac{1}{\ntrain}\sum_{i=1}^\ntrain\E[(h_n'(Z_0,Z_{1:\ntrain})-h_n'(Z_0,Z_{1:\ntrain}^{\backslash i}))^4] = o(\sigma_n^4/n^2)$.
Here, 
$Z^{\backslash i}_{1:\ntrain}$ denotes $Z_{1:\ntrain}$ with $Z_i$ replaced by an identically distributed copy independent of $Z_{0:\ntrain}$.
\end{theorem}

\cref{consistent-variance-est-in} follows from explicit error bounds proved in \cref{sec:proof-consistent-variance-est-in}.
A notable take-away is that the same two conditions---loss stability $\lstab(h_n) = o(\sigma_n^2/n)$ and a UI sequence of
$(\bar{h}_n(Z_0)-\E[\bar{h}_n(Z_0)])^2/\sigma_n^2$---grant both a central limit theorem for CV (by \cref{iid-cv-normal,asymp-from-lstability}) and an $L^1$-consistent estimate of $\sig_n^2$ (by \cref{consistent-variance-est-in}).
Moreover, the $L^2$-consistency bound of \cref{sec:proof-consistent-variance-est-in} can be viewed as a strengthening of the consistency result of \citep[Prop.~1]{MA-WZ:2020} which analyzes the same variance estimator under more stringent assumptions.
In our notation, to establish $L^2$ consistency, \citep[Prop.~1]{MA-WZ:2020} additionally requires $h_n$ symmetric in its training points, convergence of the variance parameter  $\tilde{\sigma}_n^2$ \cref{eq:MAWZ-variance} to a non-zero constant, control over a fourth-moment analogue of mean-square stability $\gamma_4(h_n) = o(\sig_n^4/n^2)$ instead of the smaller fourth-moment loss stability $\gamma_4(h_n')$, and the more restrictive fourth-moment condition $\E[(h_n(Z_0,Z_{1:\ntrain})/\sig_n)^4]=O(1)$.\footnote{The result \citep[Prop.~1]{MA-WZ:2020} also assumes a fourth moment second-order stability condition similar to \cref{second-order-mss}, but this appears to not be used in the proof.}
By \cref{var-comp}, their assumptions further imply that $\sigma_n^2$ converges to a non-zero constant.
In contrast, \cref{consistent-variance-est-in} accommodates growing, vanishing, and non-convergent variance parameters $\sigma_n^2$ and a wider variety of learning procedures and losses.

Since \cref{consistent-variance-est-in} necessarily excludes the case of leave-one-out CV ($k = n$),
we propose a second estimator
with consistency guarantees for any $k$ and only slightly stronger stability conditions than \cref{consistent-variance-est-in} when $k=\Omega(n)$.
Notably, \citet{MA-WZ:2020} do not provide a consistent variance estimator for $k=n$, and \citet{SD-MvdL:2005} do not establish the consistency of any variance estimator.

\begin{theorem}[Consistent all-pairs estimate of asymptotic variance]
\label{consistent-variance-est-out}
Under the notation of \cref{iid-cv-normal} with $\ntrain = n(1-1/k)$, and $\bar{h}_n(z) = \E[h_n(z,Z_{1:\ntrain})]$, define the all-pairs variance estimator
\balignt
\sigout^2 \defeq
\frac{1}{k}\sum_{j=1}^k
\frac{k}{n}\sum_{i\in B_j'}
(
h_n(Z_i,Z_{B_j}) - \Rhat
)^2.
\ealignt
If $(Z_i)_{i\geq 1}$ are \iid copies of a random element $Z_0$,
then
$\sigout^2/\sigma_n^2 \toL{1} 1$
whenever $\lstab(h_n) = o(\sigma_n^2/n)$, $\msstab(h_n) = o(k\sigma_n^2/n)$, and the sequence of
$(\bar{h}_n(Z_0)-\E[\bar{h}_n(Z_0)])^2/\sigma_n^2$
is UI.
\end{theorem}
\cref{consistent-variance-est-out} follows from an explicit error bound proved in \cref{sec:proof-consistent-variance-est-out} and differs from the $L^1$-consistency result of \cref{consistent-variance-est-in} only in the added requirement $\msstab(h_n) = o(k\sigma_n^2/n)$.  
This mean-square stability condition is especially mild when $k = \Omega(n)$ (as in the case of leave-one-out CV) and ensures that two training sets differing in only $n/k$ points produce prediction rules with comparable test losses. 
Importantly, both $\sigin^2$ and $\sigout^2$ can be computed in $O(n)$ time using just the individual datapoint losses $h_n(Z_i,Z_{B_j})$ outputted by a run of $k$-fold cross-validation.
Moreover, when $h_n$ is binary, as in the case of 0-1 loss, one can compute $\sigout^2 = \Rhat(1-\Rhat)$ in  $O(1)$ time given access to the overall cross-validation error $\Rhat$ and $\sigin^2 = \frac{1}{k}\sum_{j=1}^k
\frac{(n/k)}{(n/k)-1} \hat{R}_{n,j}(1-\hat{R}_{n,j})$ in $O(k)$ time given access to the $k$ average fold errors $\hat{R}_{n,j} \defeq \frac{k}{n}\sum_{i\in B_j'}h_n(Z_i,Z_{B_j})$.
\section{Numerical Experiments}\label{sec:experiments}

In this section, we compare our test error confidence intervals \cref{eq:ci} and tests for algorithm improvement \cref{eq:one-sided-test} with the most popular alternatives from the literature: the hold-out test described in \citep[Eq. (17)]{MA-WZ:2020} based on a single train-validation split, the cross-validated $t$-test \citep{dietterich:1998}, the repeated train-validation $t$-test \citep{nadeau-bengio:2003} (with and without correction), and the $5\times 2$-fold CV test \citep{dietterich:1998}.\footnote{We exclude McNemar's test \citep{mcnemar:1947} and the difference-of-proportions test which \citet{dietterich:1998} found to be less powerful than $5\times2$-fold CV and the conservative $Z$-test which \citet{nadeau-bengio:2003} found less powerful and more expensive than  corrected repeated train-validation splitting.} These procedures are commonly used and admit both two-sided CIs and one-sided tests, but, unlike our proposals, none except the hold-out method are known to be valid.
Our aim is to verify whether our proposed procedures outperform these popular heuristics across a diversity of settings encountered in real learning problems.
We fix $k=10$, use $90$-$10$ train-validation splits for all tests save $5\times2$-fold CV, and report our results using $\sigout^2$ (as $\sigin^2$ results are nearly identical).

Evaluating the quality of CIs and tests requires knowledge of the target test error.%
\footnote{Generalizing the notion of $k$-fold test error \cref{eq:cond-cv-err}, we define the target test error %
for each testing procedure to be the average test error of the learned prediction rules; see \cref{sec:proc-list} for more details.}
In each experiment, we use points subsampled from a large real dataset to form a surrogate ground-truth estimate of 
the test error. %
Then, we evaluate the CIs and tests constructed from $500$ training sets of sample sizes $n$ ranging from $700$ to $11,000$ subsampled from the same dataset.
Each mean width estimate is displayed with a $\pm$ 2 standard error confidence band. The surrounding confidence bands for the coverage, size, and power estimates are $95\%$ Wilson intervals \citep{wilson:1927}, which are known to provide more accurate coverage for binomial proportions than a $\pm$ 2 standard error interval \citep{brown-etal:2001}.
We use the \texttt{Higgs} dataset of \cite{baldi:2014,dataset:higgs} to study the classification error of random forest, neural network, and $\ell^2$-penalized logistic regression classifiers and the Kaggle \texttt{FlightDelays} dataset of \cite{dataset:flight-delays} to study the mean-squared regression error of random forest, neural network, and ridge regression.
In each case, we focus on stable settings of these learning algorithms with sufficiently strong $\ell^2$ regularization for the neural network, logistic, and ridge learners and small depths for the random forest trees.
Complete experimental details are available in \cref{sec:gen-exper-setup}, and code replicating all experiments can be found at \url{https://github.com/alexandre-bayle/cvci}.

\subsection{Confidence intervals for test error}
\label{sec:ci-experiment}
In \cref{sec:additional-results-CI}, we compare the coverage and width of each procedure's $95\%$ CI for each of the described algorithms, datasets, and training set sizes.
Two representative examples---logistic regression classification and random forest regression---are displayed in 
\cref{fig:test-error-CI}. 
While the repeated train-validation CI significantly undercovers in all cases, all remaining CIs have coverage near the $95\%$ target, even for the smallest training set size of $n=700$.
The hold-out CI, while valid, is substantially wider and less informative than the other intervals as it is based on only a single train-validation split.
Meanwhile, our CLT-based CI delivers the smallest width\footnote{All widths in \cref{fig:test-error-CI} are displayed with $\pm 2$ standard error bars, but some bars are too small to be visible.} (and hence greatest precision) for both learning tasks and every dataset size.

\newcommand{\subfigfracout}{0.5}
\newcommand{\subfigfracin}{0.49}
\newcommand{\imgspace}{.01}
\newcommand{\imgfrac}{1}
\newcommand{\vgap}{.01}
\begin{figure}[tb!]
\centering
    \begin{subfigure}{\subfigfracin\linewidth}
        \includegraphics[width=\imgfrac\linewidth]{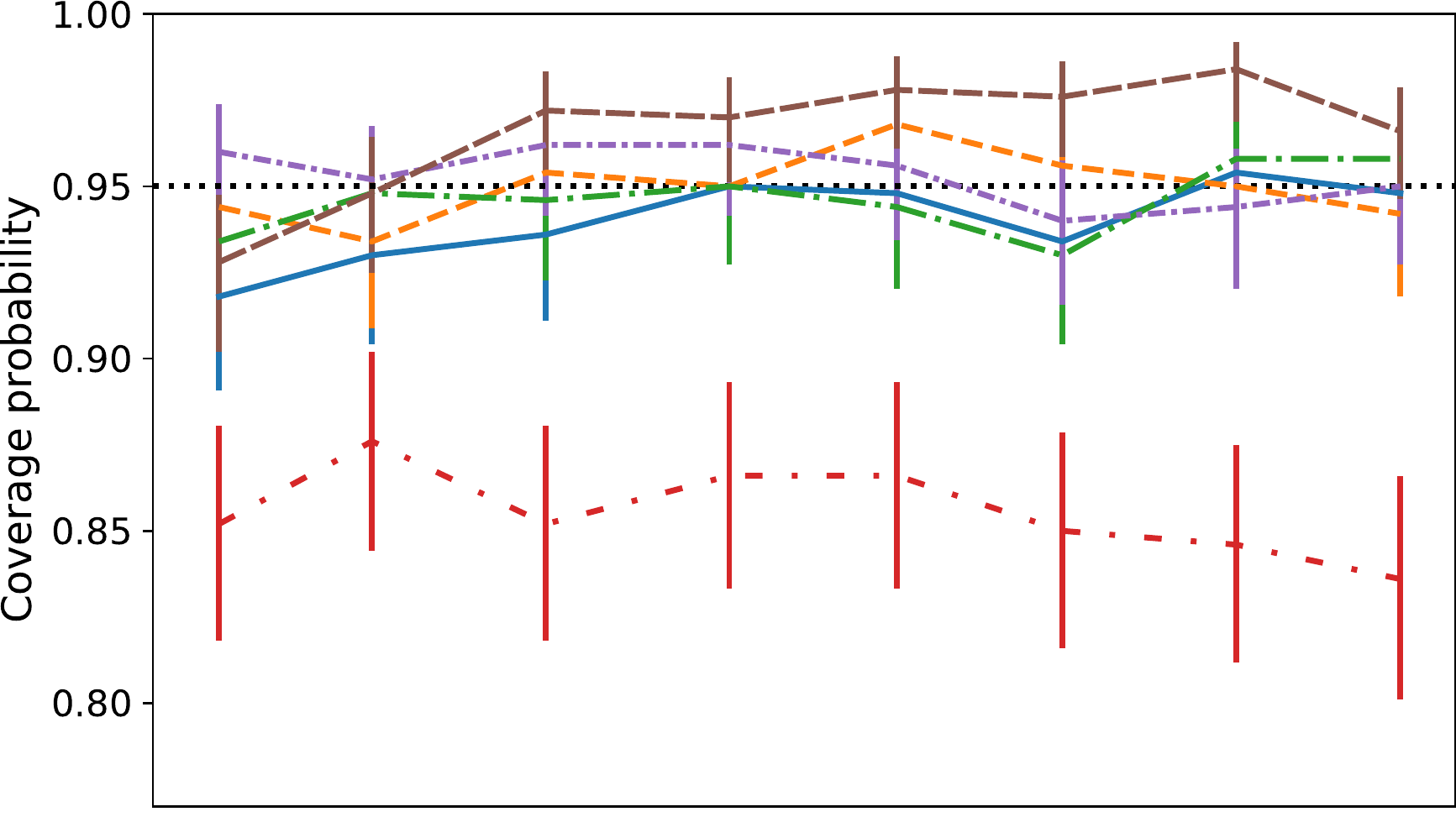}
    \end{subfigure}\hspace{\imgspace\linewidth}%
     \begin{subfigure}{\subfigfracin\linewidth}
             \includegraphics[width=\imgfrac\linewidth]{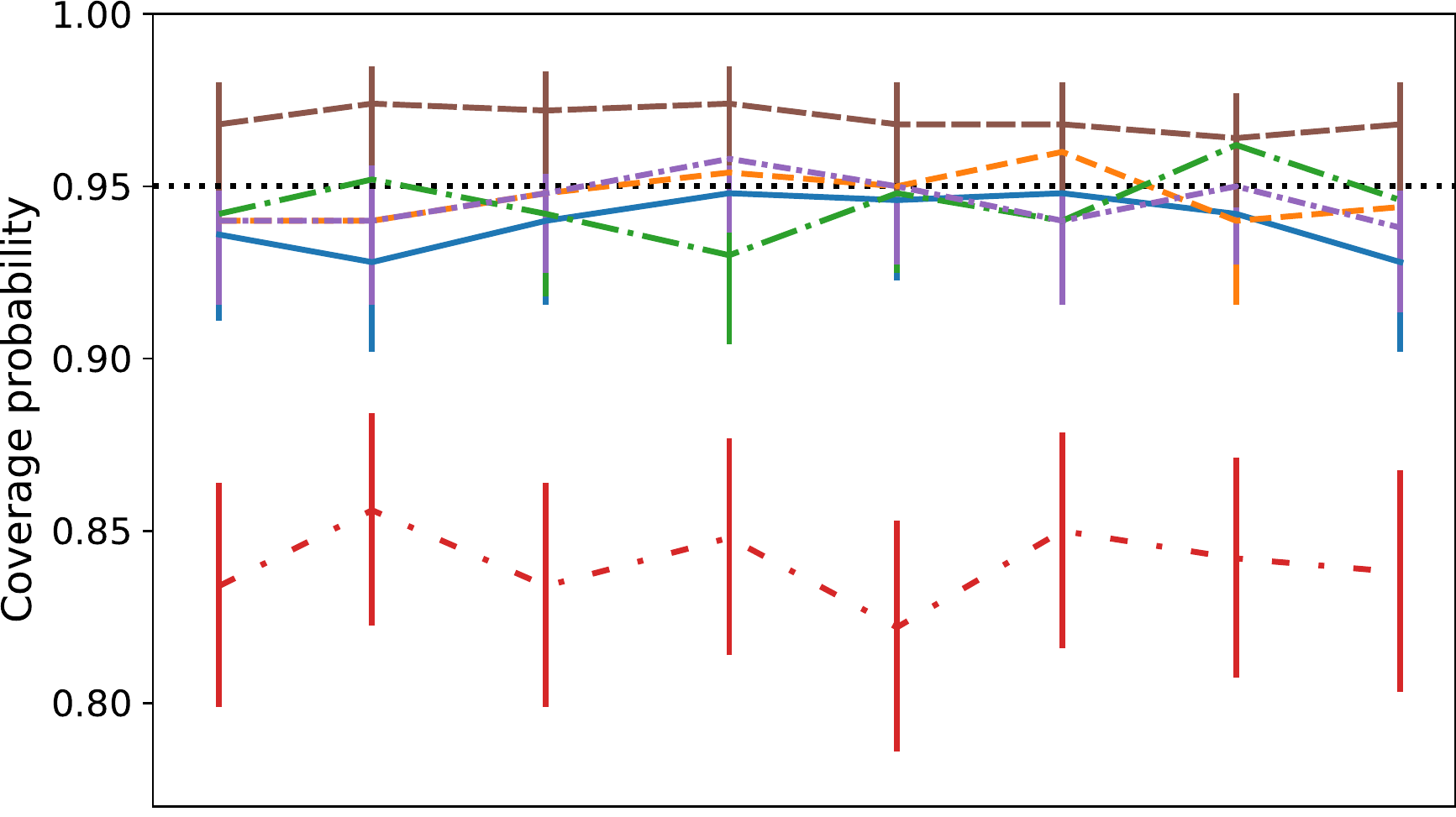}
    \end{subfigure}

    \vspace{\vgap\linewidth}
    \begin{subfigure}{\subfigfracin\linewidth}
        \includegraphics[width=\imgfrac\linewidth]{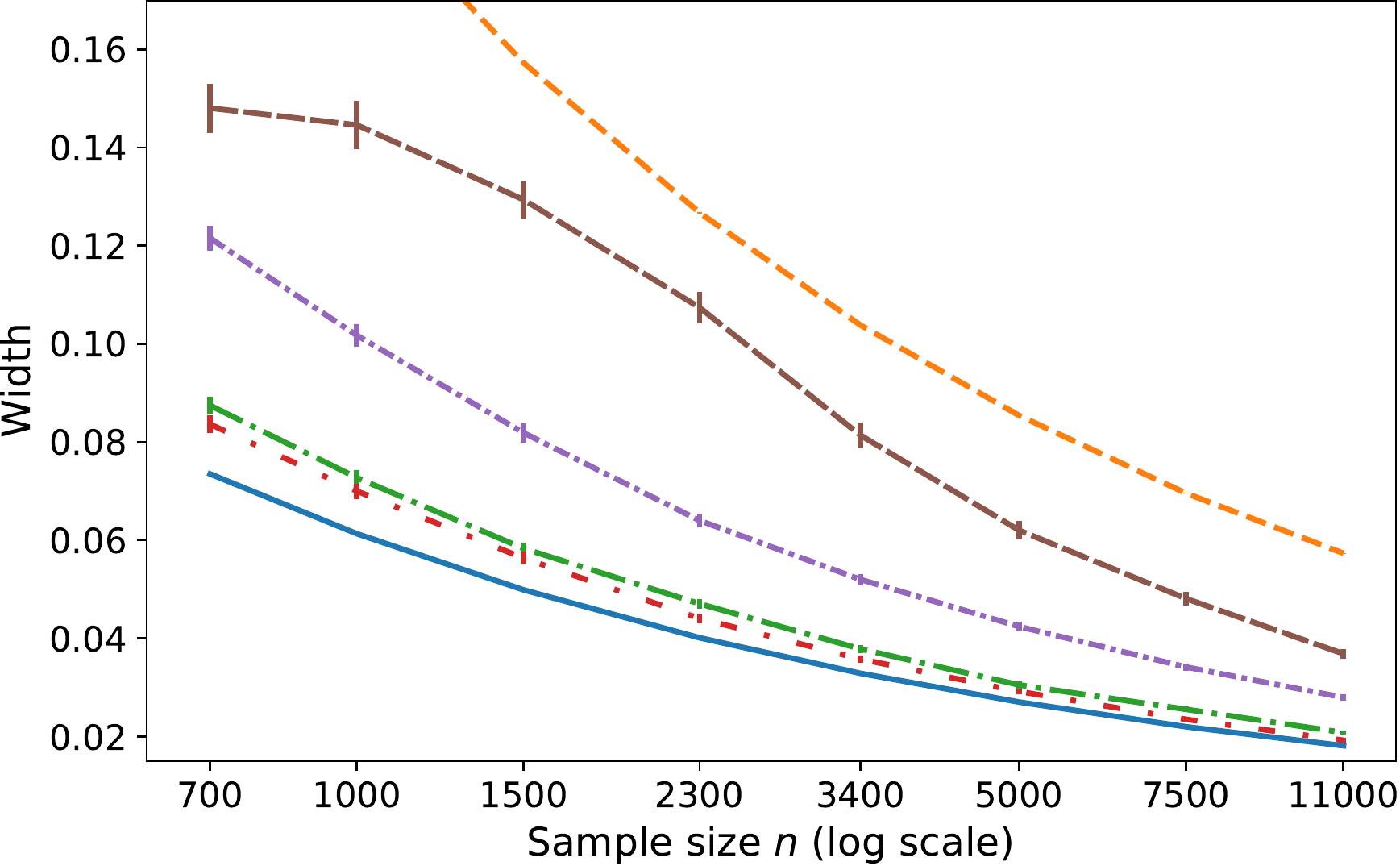}
    \end{subfigure}\hspace{\imgspace\linewidth}%
    \begin{subfigure}{\subfigfracin\linewidth}
        \includegraphics[width=\imgfrac\linewidth]{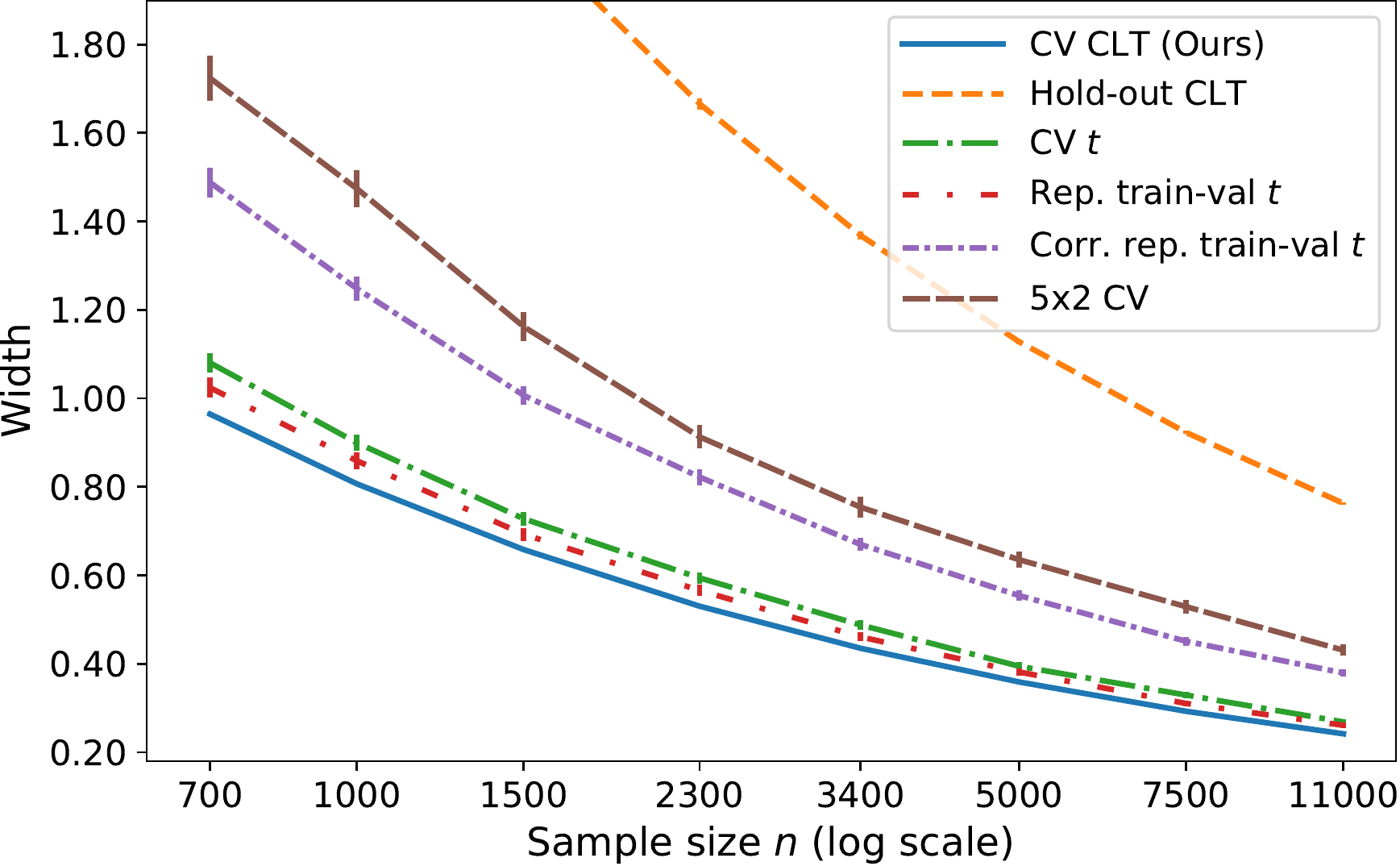}
    \end{subfigure}
    \caption{Test error coverage (top) and width (bottom) of $95\%$ confidence intervals (see \cref{sec:ci-experiment}). \tbf{Left:} $\ell^2$-regularized logistic regression classifier. \tbf{Right:} Random forest regression.}
    \label{fig:test-error-CI}
\end{figure}

\begin{figure}[tb!]
\centering
    \begin{subfigure}{\subfigfracin\linewidth}
        \includegraphics[width=\imgfrac\linewidth]{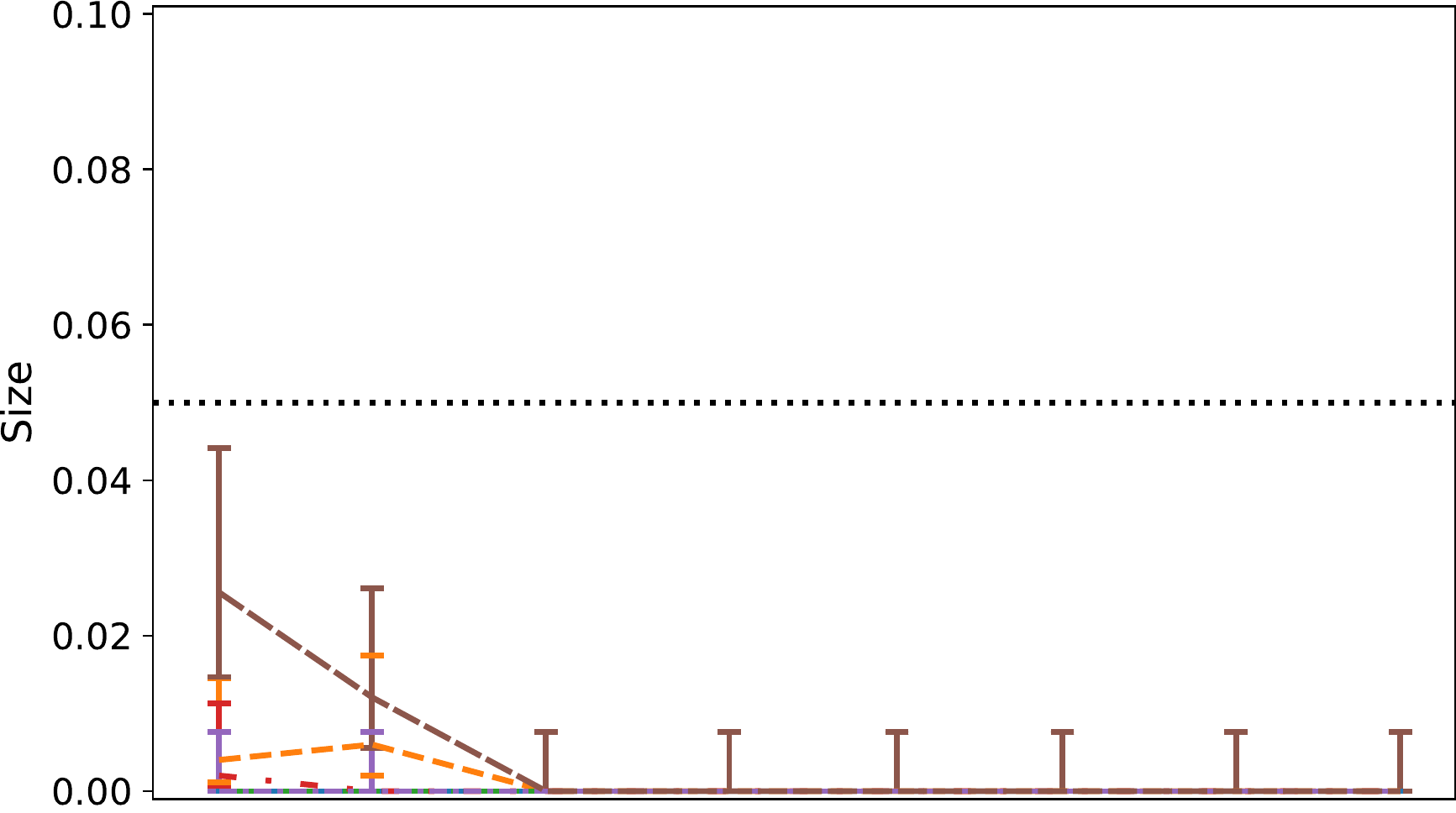}
    \end{subfigure}\hspace{\imgspace\linewidth}%
     \begin{subfigure}{\subfigfracin\linewidth}
        \includegraphics[width=\imgfrac\linewidth]{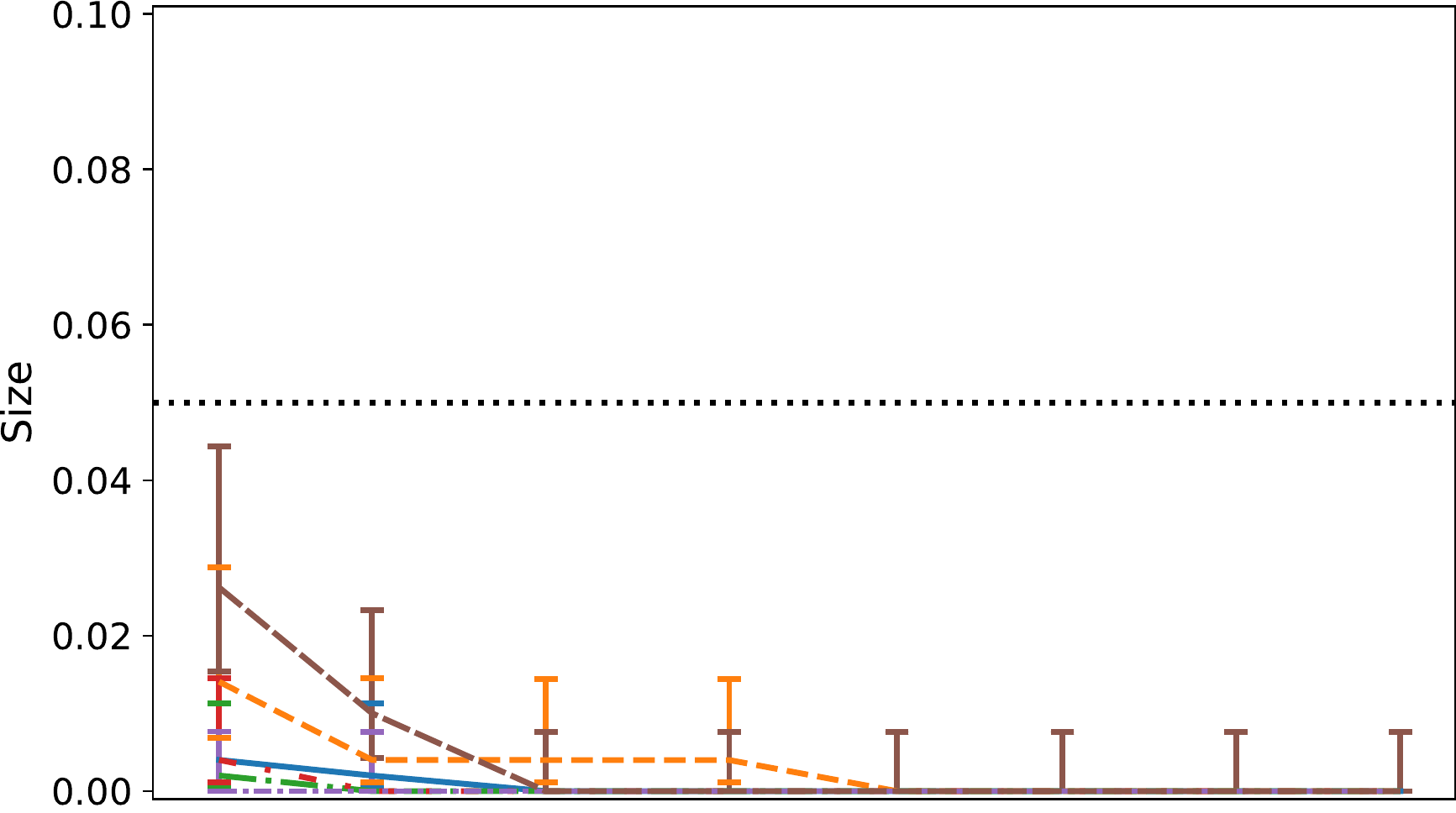}
    \end{subfigure}
    
    \vspace{\vgap\linewidth}
    \begin{subfigure}{\subfigfracin\linewidth}
        \includegraphics[width=\imgfrac\linewidth]{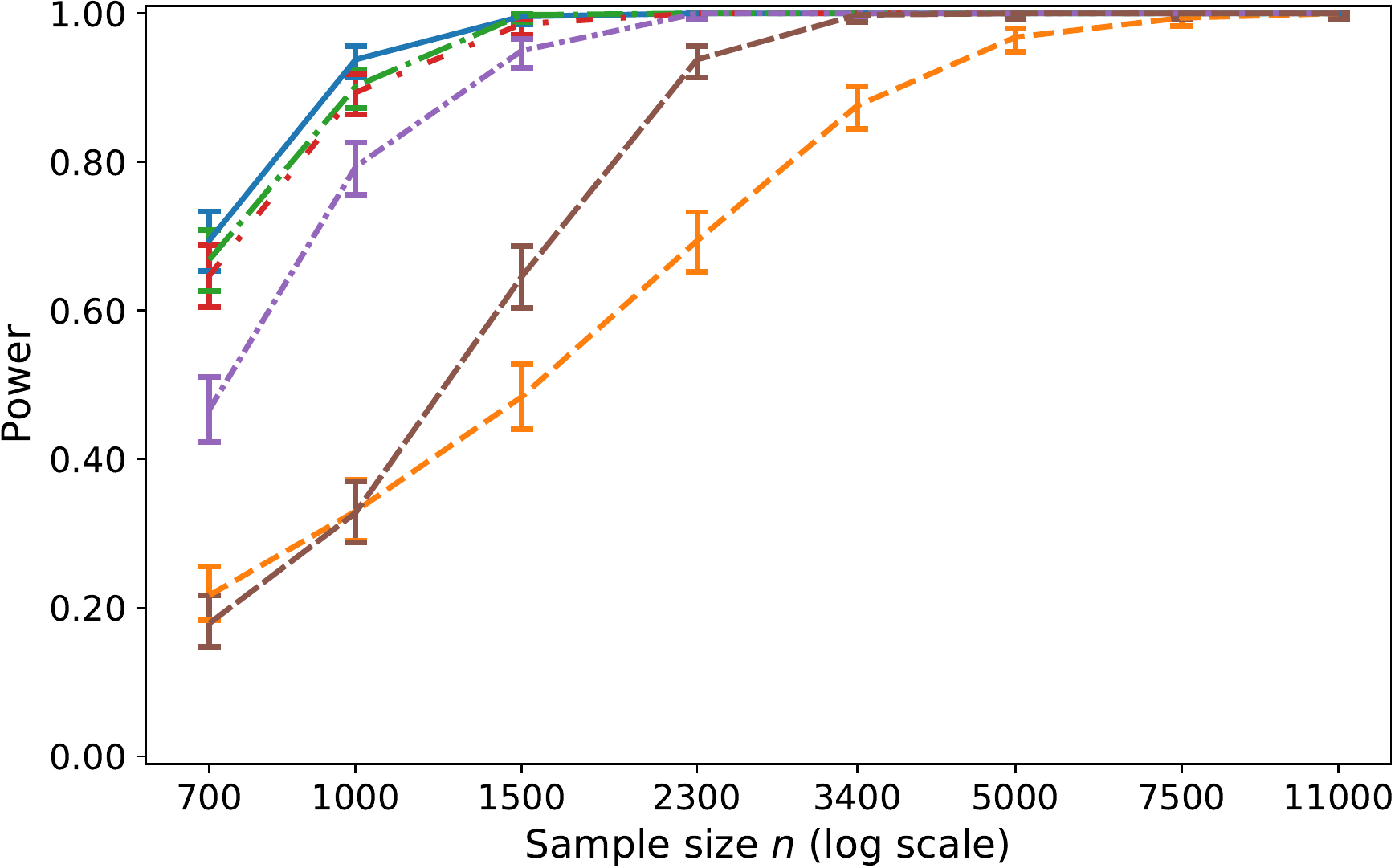}
    \end{subfigure}\hspace{\imgspace\linewidth}%
    \begin{subfigure}{\subfigfracin\linewidth}
        \includegraphics[width=\imgfrac\linewidth]{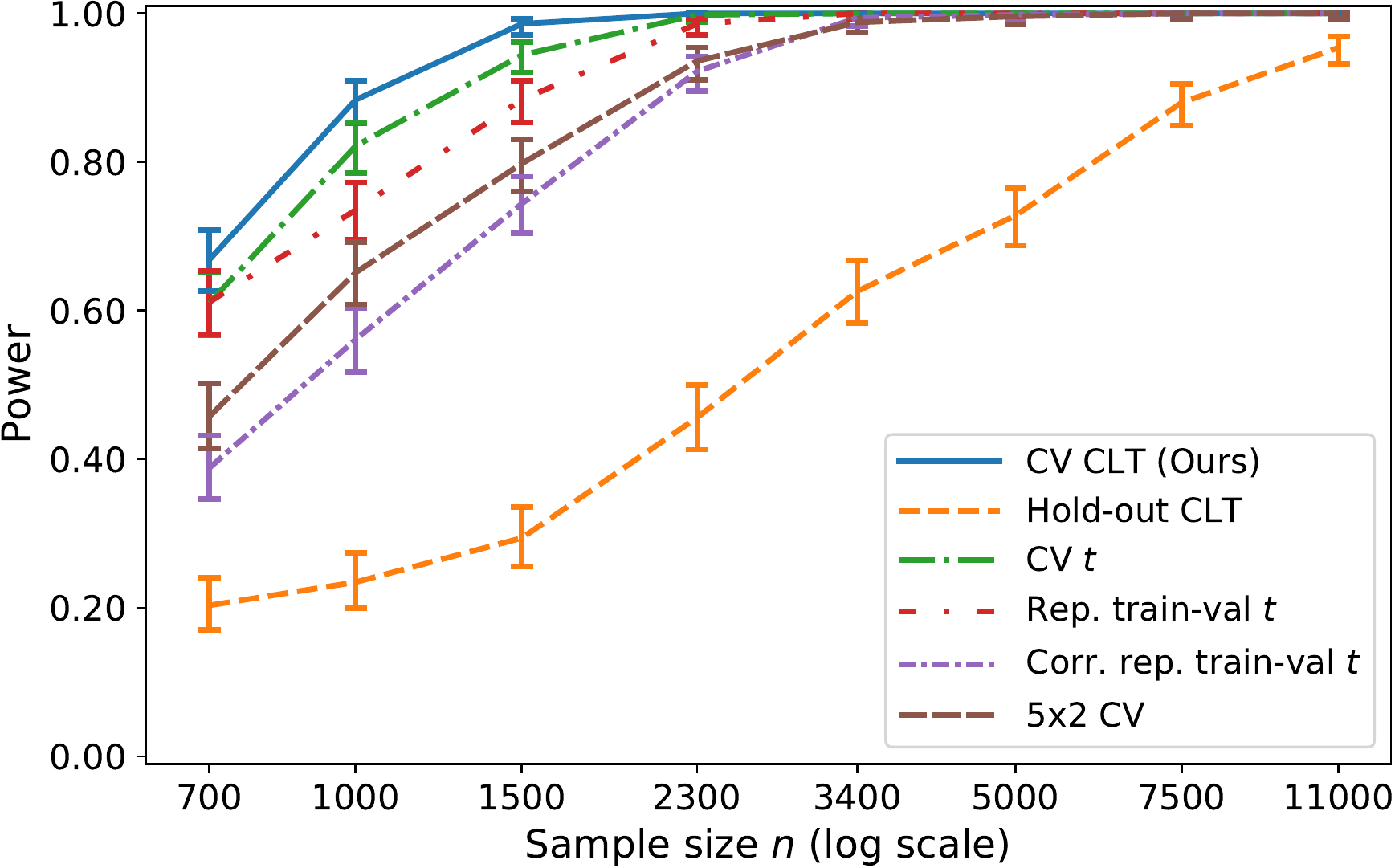}
    \end{subfigure}
    \caption{
    Size when testing $H_1: \err{1}{2}$ (top) and power when testing $H_1: \err{2}{1}$ (bottom) of level-$0.05$ tests for improved test error (see \cref{sec:sim:test}).
    \tbf{Left}: $\alg_1 =$ $\ell^2$-regularized logistic regression, $\alg_2=$ neural network classification. \tbf{Right}: $\alg_1 =$ random forest, $\alg_2=$ ridge regression.}
    \label{fig:test-error-improvement}
\end{figure}

\subsection{Testing for improved algorithm performance}\label{sec:sim:test}
Let us write $\err{1}{2}$ to signify that the test error of $\alg_1$ is smaller than that of $\alg_2$. 
In \cref{sec:additional-results-test}, for each testing procedure, dataset, and pair of algorithms $(\alg_{1},\alg_2)$, we display the size and power of level $\alpha=0.05$ one-sided tests \cref{eq:one-sided-test} of $H_1: \err{1}{2}$.  
In each case, we report size estimates for experiments with at least 25 replications under the null and power estimates for experiments with at least 25 replications under the alternative.
Here, for representative algorithm pairs, we identify the algorithm $\alg_1$ that more often has smaller test error across our simulations and display
both the power of the level $\alpha=0.05$ test of $H_1: \err{1}{2}$
and the size of the level $\alpha=0.05$ test of $H_1: \err{2}{1}$.
\cref{fig:test-error-improvement} displays these results for $(\alg_1,\alg_2)$ = ($\ell^2$-regularized logistic regression, neural network) classification on the left and $(\alg_1,\alg_2)$ = (random forest, ridge) regression on the right.
The sizes of all testing procedures are below the nominal level of $0.05$, and our test is consistently the most powerful for both classification and regression.
The hold-out test, while also valid, is significantly less powerful due to its reliance on a single train-validation split.
In \cref{sec:synthetic-test}, we observe analogous results when labels are synthetically generated.

\subsection{The importance of stability}
\label{sec:import-stab}
To illustrate the impact of algorithmic instability on testing procedures, we additionally compare a less stable neural network (with substantially reduced $\ell^2$ regularization strength) and a less stable random forest regressor (with larger-depth trees).
In \cref{fig:bad-plots} in \cref{sec:appendix-importance-stability}, we observe that the size of every test save the hold-out test
rises above the nominal level.
In the case of our test, the cause of this size violation is clear.
\cref{fig:comp-var-to-one} in \cref{sec:appendix-importance-stability} demonstrates that the variance of $\frac{\sqrt{n}}{\sigma_n}(\Rhat-\Rcondcv)$ in \cref{iid-cv-normal} is much larger than $1$ for this experiment, and \cref{asymp-from-lstability} implies this can only occur when the loss stability $\lstab(h_n)$ is large.
Meanwhile, the variance of the same quantity is close to $1$ for the original stable settings of the neural network and random forest regressors.
We suspect that instability is also the cause of the other tests' size violations; however, it is difficult to be certain, as these alternative tests have no correctness guarantees.
Interestingly, the same destabilized algorithms produce high-quality confidence intervals and relatively stable $h_n$ in the context of single algorithm assessment (see \cref{fig:unstable-single-algos-plots,fig:comp-var-to-one-single} in \cref{sec:appendix-importance-stability}), as the variance parameter $\sigma_n^2 = \Var(\bar{h}_n(Z_0))$ is significantly larger for single algorithms. 
This finding highlights an important feature of our results: it suffices for the loss stability to be negligible relative to the noise level $\sigma_n^2/n$.

\subsection{Leave-one-out cross-validation}\label{sec:loocv}
Leave-one-out cross-validation (LOOCV) is often viewed as prohibitive for large datasets, due to the expense of refitting a prediction rule $n$ times.
However, for ridge regression, a well-known shortcut based on 
the Sherman--Morrison--Woodbury formula allows one to carry out LOOCV exactly in the time required to fit a small number of base ridge regressions (see \cref{sec:appendix-LOOCV} for a derivation of this result). Moreover, recent work shows that, for many learning procedures, LOOCV estimates can be efficiently approximated with only $O(1/n^2)$ error \citep{beirami-etal:2017,giordano-etal:2019,weikoh-etal:2019,wilson-etal:2020} (see also \citep{rad-maleki:2020,stephenson-broderick:2020,ghosh-etal:2020} for related guarantees). %
The $O(1/n^2)$ precision of these inexpensive approximations coupled with the LOOCV consistency of $\sigout^2$ (see \cref{consistent-variance-est-out}) allows us to efficiently construct asymptotically-valid CIs and tests for LOOCV, even when $n$ is large.
As a simple demonstration, we construct $95\%$ CIs for ridge regression test error based on our LOOCV CLT and compare their coverage and width with those of the procedures described in \cref{sec:ci-experiment}. 
In \cref{fig:test-error-CI-RR-LOOCV} in \cref{sec:appendix-LOOCV}, we see that, like the 10-fold CV CLT intervals, the LOOCV intervals provide coverage near the nominal level and widths smaller than the popular alternatives from the literature; in fact, the 10-fold CV CLT curves are obscured by the nearly identical LOOCV CLT curves.
Complete experimental details can be found in \cref{sec:appendix-LOOCV}.

\section{Conclusion and Future Work}
Our central limit theorems and consistent variance estimators provide new, valid tools for testing algorithm improvement and generating test error intervals under algorithmic stability.
An important open question is whether practical valid tests and intervals are also available when our stability conditions are violated.
Another promising direction for future work is developing analogous tools for the \emph{expected} test error $\E[\Rcondcv]$ instead of the $k$-fold test error $\Rcondcv$; \citet{MA-WZ:2020} provide significant progress in this direction, but more work, particularly on variance estimation, is needed. %

\section*{Broader Impact}
This work will benefit both users and developers of machine learning methods who want to rigorously assess or compare learning algorithms. Failure of the methods we discuss (which can only happen when the assumptions we state are not satisfied) may lead to the over- or under-estimation of the performance of a learning algorithm on a particular dataset.

\begin{ack}
We would like to thank Jianqing Fan, Mykhaylo Shkolnikov, Miklos Racz, and Morgane Austern for helpful discussions.
\end{ack}

\bibliography{references}
\bibliographystyle{apalike}
\newpage
\appendix
\section{Proof of \cref{cv-asymp-linear}: Asymptotic linearity of $k$-fold CV}
\label{sec:proof-cv-asymp-linear-general}
We first prove a general asymptotic linearity result for repeated sample-splitting estimators.
Given a collection $A_n = \{(B_j, B'_j)\}_{j \in [J]}$ of index vector pairs such that for any pair $(B_j,B'_j)$ in $A_n$, $B_j$ and $B'_j$ are disjoint, and a scalar loss function $\rho_{n,j}(Z_{B_j'},Z_{B_j})$, define the 
\emph{cross-validation error} as
\begin{talign}
\Rhatn
    = 
    \frac{1}{J}
    \sum_{j=1}^J
    \rho_{n,j}(Z_{B_j'}, Z_{B_j})
\end{talign} 
and the \emph{multi-fold test error}
\begin{talign}
    \Rcondcvn
    = 
    \frac{1}{J}
    \sum_{j=1}^J
    \E[\rho_{n,j}(Z_{B_j'}, Z_{B_j})\mid Z_{B_j}].
\end{talign}
Note that similarly to the number of folds $k$ in cross-validation, $J$ can depend on the sample size $n$, but we write $J$ in place of $J_n$ to simplify our notation.
\begin{proposition}[Asymptotic linearity of CV]\label{cv-asymp-linear-general}
For any sequence of datapoints $(Z_i)_{i\geq 1}$,
\begin{talign}\label{eq:asylin-general}
\frac{\sqrt{n}}{\sigma_n}\left(\Rhatn-\Rcondcvn\right) 
- \frac{\sqrt{n}}{\sigma_n J}
\sum_{j=1}^J
    (\bar{\rho}_{n,j}(Z_{B_j'})
    - \E[\bar{\rho}_{n,j}(Z_{B_j'})])
    \toprob \Big(\text{resp.\,} \toL{q}\hspace{-0.08cm}\Big)\; 0
\end{talign}
for functions $\bar{\rho}_{n,1},\dots,\bar{\rho}_{n,J}$ with $\sigma_n^2\defeq \frac{1}{J}\Var(\sum_{j=1}^J \bar{\rho}_{n,j}(Z_{B_j'}))$ if and only if
\balignt\label{eq:asylincond-general}
\frac{\sqrt{n}}{\sigma_n J}
\sum_{j=1}^J
&\Big(
\rho_{n,j}(Z_{B_j'},Z_{B_j})
- \E\left[\rho_{n,j}(Z_{B_j'},Z_{B_j})\mid Z_{B_j}\right]\\
&-\left(\bar{\rho}_{n,j}(Z_{B_j'})
-\E\left[\bar{\rho}_{n,j}(Z_{B_j'})\right]\right)
\Big)
\toprob \Big(\text{resp.\,} \toL{q}\hspace{-0.08cm}\Big)\; 0
\ealignt
where the parenthetical convergence indicates that the same statement holds when both convergences in probability are replaced with convergences in $L^q$ for the same $q > 0$.
\end{proposition}
\begin{proof}
For each $(B_j,B'_j)\in A_n$, let 
\balignt
L_{j} = 
\rho_{n,j}(Z_{B_j'},Z_{B_j})
- \E[\rho_{n,j}(Z_{B_j'},Z_{B_j})\mid Z_{B_j}]
-\left(\bar{\rho}_{n,j}(Z_{B_j'})
-\E[\bar{\rho}_{n,j}(Z_{B_j'})]\right).
\ealignt
Then
\begin{talign}
\frac{\sqrt{n}}{\sigma_n}\left(\Rhatn-\Rcondcvn\right) 
&= 
\frac{\sqrt{n}}{\sigma_n J}
\sum_{j=1}^J
    (\rho_{n,j}(Z_{B_j'},Z_{B_j})
    - \E[\rho_{n,j}(Z_{B_j'},Z_{B_j})\mid Z_{B_j}]) \\
&= 
\frac{\sqrt{n}}{\sigma_n J}
\sum_{j=1}^J
    L_{j}
+
\frac{\sqrt{n}}{\sigma_n J}
\sum_{j=1}^J
    (\bar{\rho}_{n,j}(Z_{B_j'})
    - \E[\bar{\rho}_{n,j}(Z_{B_j'})]).
\end{talign}
The result now follows from the assumption that $\frac{\sqrt{n}}{\sigma_n J}
\sum_{j=1}^J
    L_{j}
    \toprob \Big(\text{resp.\,} \toL{q}\hspace{-0.08cm}\Big)\; 0$.
\end{proof}

\cref{cv-asymp-linear} now follows directly from \cref{cv-asymp-linear-general} with the choices:
\begin{itemize}
    \item $A_n=\{(B_{\ell},i):\ell\in [k],i\in B_{\ell}'\}$,
    \item for all $j \in [J]$,  $\rho_{n,j}(Z_{i}, Z_{B_{\ell}})=h_n(Z_i,Z_{B_{\ell}})$ and $\bar{\rho}_{n,j}(Z_{i})=\bar{h}_n(Z_i)$ for the associated $\ell\in [k]$ and $i\in B_{\ell}'$.
\end{itemize}
Note that for these choices, we have $J=|A_n|=\sum_{\ell=1}^k|B_{\ell}'|=n$.
\section{Proof of \cref{iid-cv-normal}: Asymptotic normality of $k$-fold CV with i.i.d.\ data}\label{sec:proof-iid-cv-normal}
\cref{iid-cv-normal} follows from the %
next more general result, which establishes the asymptotic normality of $k$-fold CV with independent (not necessarily identically distributed) data.
\begin{theorem}[Asymptotic normality of $k$-fold CV with independent data]\label{independent-cv-normal}

Under the notation of \cref{cv-asymp-linear}, suppose that the datapoints $(Z_i)_{i\geq 1}$ are independent.
If the triangular array $\left(\bar{h}_n(Z_i)-\E\left[\bar{h}_n(Z_i)\right]\right)_{n,i}$ satisfies \emph{Lindeberg's condition},
\begin{talign}\label{eq:Lindeberg}
\forall \varepsilon>0,
\frac{1}{n \sigma_n^2} \sum_{i=1}^n \E\left[\left(\bar{h}_n(Z_i)-\E\left[\bar{h}_n(Z_i)\right]\right)^2 \indic{|\bar{h}_n(Z_i)-\E\left[\bar{h}_n(Z_i)\right]|>\varepsilon\, \sigma_n \sqrt{n}}\right] \to 0,
\end{talign}
then
\begin{talign}
\frac{1}{\sigma_n \sqrt{n}} \sum_{i=1}^n \left(\bar{h}_n(Z_i)-\E\left[\bar{h}_n(Z_i)\right]\right)
\todist \N(0,1).
\end{talign}
Additionally, if \cref{eq:asylincond} holds in probability, then
\begin{talign}
\frac{\sqrt{n}}{\sigma_n}\left(\Rhat-\Rcondcv\right)
\todist \N(0,1).
\end{talign}

\end{theorem}

\begin{proof}
By independence of the datapoints $(Z_i)_{i \geq 1}$, $(\bar{h}_n(Z_i))_{n,i}$ are independent, and $n \sigma_n^2=\Var(\sum_{i=1}^n \bar{h}_n(Z_i))$. 
Under Lindeberg's condition, we get the first convergence result thanks to Lindeberg's Central Limit Theorem (see \citep[Thm.\ 27.2]{billingsley-book:1995}). Additionally, if assumption \cref{eq:asylincond} holds, we apply \cref{cv-asymp-linear} and Slutsky's theorem to get the second convergence result.
\end{proof}

If the $(Z_i)_{i \geq 1}$ are \iid, then $\sigma_n^2=\frac{1}{n}\Var(\sum_{i=1}^n \bar{h}_n(Z_i))=\Var(\bar{h}_n(Z_0))$, and Lindeberg's condition \cref{eq:Lindeberg} reduces to
\begin{talign}
\forall \varepsilon>0,
\frac{1}{\sigma_n^2} \E\left[\left(\bar{h}_n(Z_0)-\E\left[\bar{h}_n(Z_0)\right]\right)^2 \indic{|\bar{h}_n(Z_0)-\E\left[\bar{h}_n(Z_0)\right]|>\varepsilon \, \sigma_n \sqrt{n}}\right] \to 0.
\end{talign}
We will show that this follows from the assumed uniform integrability of the sequence $X_n = (\bar{h}_n(Z_0)-\E[\bar{h}_n(Z_0)])^2/\sigma_n^2$.  
Indeed, for any $\varepsilon>0$ and all $n$,
\begin{talign}
\E[X_n \indic{X_n>n \varepsilon^2}]
&\leq \sup_m \E[X_m \indic{X_m>n \varepsilon^2}] 
\to 0,
\end{talign}
as $n \to \infty$ by the uniform integrability of the sequence of $X_n$.
\cref{iid-cv-normal} therefore follows from \cref{independent-cv-normal}.

\section{Proof of \cref{asymp-from-lstability}: Approximate linearity from loss stability}\label{sec:proof-asymp-from-lstability}
\cref{asymp-from-lstability} will follow from the following more general result.

\begin{theorem}[Approximate linearity from loss stability]
\label{asymp-from-lstability-gen}
Under the notation of \cref{sec:proof-cv-asymp-linear-general}, with $\{(B_j, B'_j)\}_{j \in [J]}$ a collection of disjoint index vector pairs where $(B_j')_{j \in [J]}$ is a pairwise disjoint family, and $\rho_{n,j}(Z_{B_j'},Z_{B_j}) \defeq \frac{1}{|B_j'|} \sum_{i\in B_j'} h_{n,j}(Z_i,Z_{B_j})$, suppose that the datapoints $(Z_i)_{i\geq 1}$ are \iid copies of a random element $Z_0$. Define $\rho_{n,j}'(Z_{B_j'},Z_{B_j}) \defeq \rho_{n,j}(Z_i,Z_{B_j}) - \E[\rho_{n,j}(Z_i,Z_{B_j})\mid Z_{B_j}]$ and $\rho_{n,j}''(Z_{B_j'},Z_{B_j}) \defeq \rho_{n,j}'(Z_i,Z_{B_j}) - \E[\rho_{n,j}'(Z_i,Z_{B_j})\mid Z_i]$.
Then
\balignt
\E[
(
\frac{1}{J}\sum_{j=1}^J \rho_{n,j}''(Z_{B_j'},Z_{B_j})
)^2
]
\leq
\frac{1}{J^2}
&\Big(
\sum_{j\neq j'}
\sqrt{
\lstab(h_{n,j})
\lstab(h_{n,j'})
} \\
&+
\sum_{j=1}^J
\frac{1}{|B_j'|}
\half |B_{j}| \lstab(h_{n,j})
\Big).
\label{eq:boundforasylincondgen}
\ealignt
\end{theorem}
\begin{proof}

Define $h_{n,j}'$ and $h_{n,j}''$ as:

$h_{n,j}'(Z_i,Z_{B_j}) \defeq h_{n,j}(Z_i,Z_{B_j}) - \E[h_{n,j}(Z_i,Z_{B_j})\mid Z_{B_j}]$,

$h_{n,j}''(Z_i,Z_{B_j}) \defeq h_{n,j}'(Z_i,Z_{B_j}) - \E[h_{n,j}'(Z_i,Z_{B_j})\mid Z_i]$.

Therefore, we have $\rho_{n,j}'(Z_{B_j'},Z_{B_j}) = \frac{1}{|B_j'|} \sum_{i\in B_j'} h_{n,j}'(Z_i,Z_{B_j})$ and $\rho_{n,j}''(Z_{B_j'},Z_{B_j}) = \frac{1}{|B_j'|} \sum_{i\in B_j'} h_{n,j}''(Z_i,Z_{B_j})$.

Thus
\balignt
(
\frac{1}{J}\sum_{j=1}^J \rho_{n,j}''(Z_{B_j'},Z_{B_j})
)^2
&=
\frac{1}{J^2}\sum_{j,j'=1}^J \rho_{n,j}''(Z_{B_{j}'},Z_{B_{j}})
\rho_{n,j'}''(Z_{B_{j'}'},Z_{B_{j'}}) \\
&=
\frac{1}{J^2}\sum_{j,j'=1}^J
\frac{1}{|B_{j}'|}\frac{1}{|B_{j'}'|}
\sum_{i\in B_{j}'}\sum_{{i'}\in B_{j'}'}
h_{n,j}''(Z_{i},Z_{B_{j}})
h_{n,j'}''(Z_{i'},Z_{B_{j'}}).
\ealignt

In what follows, $Z_{B_{j'}}^{\backslash{i}}$ is $Z_{B_{j'}}$ with $Z_i$ replaced by $Z_0'$, an \iid copy of $Z_0$, independent of $(Z_i)_{i\geq 1}$. Note that if $i \notin B_{j'}$, $Z_{B_{j'}}^{\backslash{i}}$ is just $Z_{B_{j'}}$. We similarly define $Z_{B_{j}}^{\backslash{i'}}$.

If $j\neq j'$, we have $\E_{Z_i}[\sum_{i\in B_j'}\sum_{i' \in B_{j'}'} h_{n,j}''(Z_{i},Z_{B_{j}}) h_{n,j'}''(Z_{i'},Z_{B_{j'}}^{\backslash{i}})] = 0$, because (i) $h_{n,j}''(Z_{i},Z_{B_{j}})$ and $h_{n,j'}''(Z_{i'},Z_{B_{j'}}^{\backslash{i}})$ are conditionally independent given everything but $Z_i$, and (ii) $\E_{Z_i}[h_{n,j}''(Z_{i},Z_{B_{j}})] = 0$.

Similarly, if $j\neq j'$,
\balignt
\E_{Z_{i'}}[\sum_{i\in B_j'}\sum_{i' \in B_{j'}'} h_{n,j}''(Z_{i},Z_{B_{j}}^{\backslash{i'}}) h_{n,j'}''(Z_{i'},Z_{B_{j'}})]
= 0,
\ealignt
\balignt
\E_{Z_{i'}}[\sum_{i\in B_j'}\sum_{i' \in B_{j'}'} h_{n,j}''(Z_{i},Z_{B_{j}}^{\backslash{i'}}) h_{n,j'}''(Z_{i'},Z_{B_{j'}}^{\backslash{i}})]
= 0.
\ealignt

Therefore, if $j\neq j'$,
\balignt
&\E[
\frac{1}{|B_{j}'|}\frac{1}{|B_{j'}'|}
\sum_{i\in B_j'}\sum_{i' \in B_{j'}'} h_{n,j}''(Z_{i},Z_{B_{j}})
h_{n,j'}''(Z_{i'},Z_{B_{j'}})] \\
&=
\E[
\frac{1}{|B_{j}'|}\frac{1}{|B_{j'}'|}
\sum_{i\in B_j'}\sum_{i' \in B_{j'}'}
\big((
h_{n,j}''(Z_{i},Z_{B_{j}})
- h_{n,j}''(Z_{i},Z_{B_{j}}^{\backslash {i'}})
) \\
& \hspace{4.35cm} \times
(
h_{n,j'}''(Z_{i'},Z_{B_{j'}})
- h_{n,j'}''(Z_{i'},Z_{B_{j'}}^{\backslash{i}})
)\big)
] \\
&=
\frac{1}{|B_{j}'|}\frac{1}{|B_{j'}'|}
\sum_{i\in B_j'}\sum_{i' \in B_{j'}'}
\E\big[
\big((
h_{n,j}''(Z_{i},Z_{B_{j}})
- h_{n,j}''(Z_{i},Z_{B_{j}}^{\backslash {i'}})
) \\
& \hspace{4.35cm} \times
(
h_{n,j'}''(Z_{i'},Z_{B_{j'}})
- h_{n,j'}''(Z_{i'},Z_{B_{j'}}^{\backslash{i}})
)\big)
\big]\\
&\leq
\frac{1}{|B_{j}'|}\frac{1}{|B_{j'}'|}
\sum_{i\in B_j'}\sum_{i' \in B_{j'}'}
\sqrt{\E\left[
(
h_{n,j}''(Z_{i},Z_{B_{j}})
- h_{n,j}''(Z_{i},Z_{B_{j}}^{\backslash {i'}})
)^2\right]} \\
& \hspace{4.35cm} \times
\sqrt{\E\left[
(
h_{n,j'}''(Z_{i'},Z_{B_{j'}})
- h_{n,j'}''(Z_{i'},Z_{B_{j'}}^{\backslash {i}})
)^2\right]}\\
&\leq
\Big(
\frac{1}{|B_{j}'|}\frac{1}{|B_{j'}'|}
\sum_{i\in B_j'}\sum_{i' \in B_{j'}'}
\E\left[
(
h_{n,j}''(Z_{i},Z_{B_{j}})
- h_{n,j}''(Z_{i},Z_{B_{j}}^{\backslash {i'}})
)^2\right] \\
& \hspace{4.35cm} \times
\E\left[
(
h_{n,j'}''(Z_{i'},Z_{B_{j'}})
- h_{n,j'}''(Z_{i'},Z_{B_{j'}}^{\backslash {i}})
)^2\right]
\Big)^{1/2}\\
&=
\Big(
\frac{1}{|B_{j}'|}\frac{1}{|B_{j'}'|}
\sum_{i\in B_j'}\sum_{i' \in B_{j'}'}
\E\left[
(
h_{n,j}''(Z_{0},Z_{B_{j}})
- h_{n,j}''(Z_{0},Z_{B_{j}}^{\backslash {i'}})
)^2\right] \\
& \hspace{4.35cm} \times
\E\left[
(
h_{n,j'}''(Z_{0},Z_{B_{j'}})
- h_{n,j'}''(Z_{0},Z_{B_{j'}}^{\backslash {i}})
)^2\right]
\Big)^{1/2}\\
&=
\Big(
\frac{1}{|B_{j'}'|}
\sum_{i'\in B_{j'}'}
\E\left[
(
h_{n,j}''(Z_{0},Z_{B_{j}})
- h_{n,j}''(Z_{0},Z_{B_{j}}^{\backslash {i'}})
)^2\right]
\Big)^{1/2}\\
& \hspace{4.35cm} \times
\Big(
\frac{1}{|B_{j}'|}
\sum_{i\in B_{j}'}
\E\left[
(
h_{n,j'}''(Z_{0},Z_{B_{j'}})
- h_{n,j'}''(Z_{0},Z_{B_{j'}}^{\backslash {i}})
)^2\right]
\Big)^{1/2}\\
&= \sqrt{\msstab(h_{n,j}'')\msstab(h_{n,j'}'')}
= \sqrt{\msstab(h_{n,j}')\msstab(h_{n,j'}')}
=\sqrt{\lstab(h_{n,j})\lstab(h_{n,j'})},
\ealignt
where we have applied Cauchy--Schwarz inequality and Jensen's inequality, used that the datapoints are \iid copies of $Z_0$ and applied the definitions of mean-square stability and loss stability.
If $j=j'$ and $i\neq i'$, then $\E_{Z_i}[h_{n,j}''(Z_{i},Z_{B_{j}})
h_{n,j}''(Z_{i'},Z_{B_{j}})] = 0$.

If $j=j'$ and $i=i'$, then $\E[h_{n,j}''(Z_{i},Z_{B_{j}})^2] = \E[\Var(h_{n,j}'(Z_{i},Z_{B_{j}})\mid Z_{i})].$

We now state a conditional application of a version of the Efron--Stein inequality due to \citet{steele:1986}.
\begin{lemma}[Conditional Efron--Stein inequality]\label{efron-stein}
Suppose that, given $W$, the random vectors $X_{1:m}$ and $X'_{1:m}$ are conditionally independent and identically distributed and that the components of $X_{1:m}$ are conditionally independent given $W$.  Then, for any suitably
measurable function $f$
\balignt
\half \E[(f(X_{1:m},W) - f(X'_{1:m},W))^2\mid W]
&= \Var(f(X_{1:m},W) \mid W)\\
&\leq \half \sum_{i=1}^m \E[(f(X_{1:m},W) - f(X^{\backslash i}_{1:m},W))^2\mid W]
\ealignt
where, for each $i\in[m]$, $X^{\backslash i}_{1:m}$ represents $X_{1:m}$ with $X_i$ replaced with $X_i'$.
\end{lemma}

Using \cref{efron-stein}, we get $\E[\Var(h_{n,j}'(Z_{i},Z_{B_{j}})\mid Z_{i})] \leq \half |B_{j}| \msstab(h_{n,j}') = \half |B_{j}| \lstab(h_{n,j})$.

Combining everything, we get
\balignt
\E[
(
\frac{1}{J}\sum_{j=1}^J \rho_{n,j}''(Z_{B_j'},Z_{B_j})
)^2
]
\leq
\frac{1}{J^2}
&\Big(
\sum_{j\neq j'}
\sqrt{
\lstab(h_{n,j})
\lstab(h_{n,j'})
} \\
&+
\sum_{j=1}^J
\frac{1}{|B_j'|}
\half |B_{j}| \lstab(h_{n,j})
\Big).
\ealignt
\end{proof}

In the case of $k$-fold cross-validation with equal-sized folds and \iid data, the left-hand side of \cref{eq:boundforasylincondgen} becomes
\balignt
\Var\left(\frac{1}{n}\sum_{j=1}^k\sum_{i\in B_j'} (h'_n(Z_i,Z_{B_j})-\E\left[h'_n(Z_i,Z_{B_j})\mid Z_i\right])\right),
\ealignt
and its right-hand side simplifies to
\balignt
\frac{1}{k^2}
\Big(
k(k-1)
\sqrt{
\lstab(h_n)^2}
+
k
\frac{k}{n} \half n(1-\frac{1}{k}) \lstab(h_n)
\Big)
= \frac{3}{2}(1-\frac{1}{k}) \lstab(h_n).
\ealignt

Hence,
\begin{talign}
&\frac{1}{n}\Var\left(\frac{1}{\sqrt{n}}\sum_{j=1}^k\sum_{i\in B_j'} (h'_n(Z_i,Z_{B_j})-\E\left[h'_n(Z_i,Z_{B_j})\mid Z_i\right])\right) \le \frac{3}{2}\left(1-\frac{1}{k}\right)\lstab(h_n).
\end{talign}

We then note that the asymptotic linearity condition \cref{eq:asylincond} in $L^2$-norm with the choice $\bar{h}_n(z)=\E\left[h_n(z,Z_{1:n(1-1/k)})\right]$ can be written as 
\balignt \frac{1}{\sigma_n\sqrt{n}}\sum_{j=1}^k\sum_{i\in B_j'} (h'_n(Z_i,Z_{B_j})-\E\left[h'_n(Z_i,Z_{B_j})\mid Z_i\right]) \toL{2} 0,\ealignt
which is implied by \cref{eq:suffcondforasylincond} when $\lstab(h_n)=o(\sigma_n^2/n)$.
Therefore, \cref{asymp-from-lstability} follows from \cref{asymp-from-lstability-gen}.
\section{Proof of \cref{asymp-from-cond-var}: Asymptotic linearity from conditional variance convergence}\label{sec:proof-asymp-from-cond-var}
\cref{asymp-from-cond-var} will follow from the following more general statement. 
\begin{theorem}[Asymptotic linearity from conditional variance convergence]\label{asymp-from-cond-var-uneven}
Under the notation of \cref{cv-asymp-linear}, suppose that the datapoints $(Z_i)_{i\geq 1}$ are i.i.d.\ copies of a random element $Z_0$.
If a function $\bar{h}_n$ satisfies
\balignt\label{eq:variance_convergence_assump}
\max(k^{q-1},1)\sum_{j=1}^k
    \Earg{\Big(\frac{|B_j'|}{n \sigma_n^2}
    \Var_{Z_0}\left(h_n(Z_0,Z_{B_j})-\bar{h}_n(Z_0)\right)\Big)^{q/2}} \to 0
\ealignt
for some $q \in (0,2]$, then $\bar{h}_n$ satisfies the $L^q$ asymptotic linearity condition \cref{eq:asylincond}.
If a function $\bar{h}_n$ satisfies
\balignt\label{eq:variance_convergence_assump_prob}
\sum_{j=1}^k \Earg{\min\Bigg(1, 
    \frac{\sqrt{|B_j'|}}{\sigma_n\sqrt{n}}
    \sqrt{\Var_{Z_0}\left(h_n(Z_0,Z_{B_j})-\bar{h}_n(Z_0)\right)}
    \Bigg)} 
    \to 0,
\ealignt
then $\bar{h}_n$ satisfies the in-probability asymptotic linearity condition \cref{eq:asylincond}.

\end{theorem}

\begin{proof}
In the notation of \cref{cv-asymp-linear},
for each $j\in [k]$, let 
\balignt
L_{j} = 
\frac{1}{|B_j'|} 
\sum_{i\in B_j'} 
(h_n\left(Z_i,Z_{B_j}\right)
-\bar{h}_n(Z_i))
- \E_{Z_0}[h_n\left(Z_0,Z_{B_j}\right)
-\bar{h}_n(Z_0)].
\ealignt
We first note that for any non-decreasing concave $\psi$ satisfying the triangle inequality, we have
\balignt
\Earg{
\psi\Bigg(\Bigg|
\frac{1}{\sigma_n\sqrt{n}}
\sum_{j=1}^k
    |B_j'| L_{j}
    \Bigg|\Bigg)}
&\leq
\Earg{
\psi\Bigg(
\frac{1}{\sigma_n\sqrt{n}}
\sum_{j=1}^k
    |B_j'| |L_{j}|
\Bigg)} \\
&\leq
\sum_{j=1}^k 
\Earg{
\psi\Bigg(
\frac{1}{\sigma_n\sqrt{n}}
    |B_j'| |L_{j}|
\Bigg)} \\
&=
\sum_{j=1}^k 
\Earg{
\Esubarg{Z_{B_j'}}{
\psi\Bigg(
\frac{1}{\sigma_n\sqrt{n}}
    |B_j'| |L_{j}|
\Bigg)}} \\
&\leq
\sum_{j=1}^k 
\Earg{
\psi\Bigg(
\frac{1}{\sigma_n\sqrt{n}}
    |B_j'| \Esubarg{Z_{B_j'}}{|L_{j}|}
\Bigg)} \\
&\leq
\sum_{j=1}^k 
\Earg{
\psi\Bigg(
\frac{1}{\sigma_n\sqrt{n}}
    |B_j'| \sqrt{\Varsubarg{Z_{B_j'}}{L_{j}}}
\Bigg)} \\
&=
\sum_{j=1}^k 
\Earg{
\psi\Bigg(
\frac{\sqrt{|B_j'|}}{\sigma_n\sqrt{n}}
     \sqrt{\Var_{Z_0}\left(h_n(Z_0,Z_{B_j})-\bar{h}_n(Z_0)\right)}
\Bigg)},
\ealignt
where we have applied the triangle inequality twice,
the tower property once, and Jensen's inequality twice.
The advertised $L^q$ result for $q \in (0,1]$ now follows by taking $\psi(x) = x^q$, and the in-probability result follows by taking $\psi(x) = \min(1, x)$ and invoking the following lemma.

\begin{lemma}\label{dudley}
For any sequence of random variables $(X_n)_{n\geq 1}$, $X_n\toprob 0$ if and only if $\E[\psi(|X_n|)]\to 0$, where $\psi(x)=\min(1,x)$.
\end{lemma}
\begin{proof}
If $X_n\toprob 0$, then as $X_n\todist 0$ and $\psi$ is bounded and continuous for nonnegative $x$, $\E[\psi(|X_n|)]\to 0$.
Now suppose $\E[\psi(|X_n|)]\to 0$.
Since $\psi$ is nonnegative and non-decreasing for nonnegative $x$, we have $\P(|X_n| > \epsilon) \leq \E[\psi(|X_n|)]/\psi(\epsilon) \to 0$ for every $\epsilon > 0$ by Markov's inequality.  Hence, $X_n\toprob 0$.
\end{proof}

Now fix any $q \in (1,2]$, and note that as $x \mapsto x^q$ is non-decreasing and convex on the nonnegative reals, we have
\balignt
\Earg{
\Bigg|
\frac{1}{\sigma_n\sqrt{n}}
\sum_{j=1}^k
    |B_j'| L_{j}
    \Bigg|^q}
&\leq
\Earg{
\Bigg(
\frac{k}{k}
\sum_{j=1}^k
    \frac{1}{\sigma_n\sqrt{n}} |B_j'| |L_{j}|
\Bigg)^q} \\
&\leq
\frac{k^q}{k}
\sum_{j=1}^k 
\Earg{
\Bigg(
\frac{1}{\sigma_n\sqrt{n}}
    |B_j'| |L_{j}|
\Bigg)^q} \\
&=
k^{q-1}
\Earg{
\sum_{j=1}^k 
\Esubarg{Z_{B_j'}}{
\Bigg(\frac{1}{n \sigma_n^2}
    |B_j'|^2|L_{j}|^2\Bigg)^{q/2}}} \\
&\leq
k^{q-1}
\Earg{
\sum_{j=1}^k
\Bigg(\frac{|B_j'|^2}{n \sigma_n^2}
    \Varsubarg{Z_{B_j'}}{
L_{j}}\Bigg)^{q/2}} \\
&=
k^{q-1}
\Earg{
\sum_{j=1}^k
\Bigg(\frac{|B_j'|}{n \sigma_n^2}
    \Varsubarg{Z_0}{h_n(Z_0,Z_{B_j})-\bar{h}_n(Z_0)}\Bigg)^{q/2}
    },
\ealignt
where we have applied the triangle inequality, Jensen's inequality using the convexity of $x\mapsto x^q$, the tower property, and Jensen's inequality using the concavity of $x \mapsto x^{q/2}$.
Hence, the $L^q$ result for $q \in (1,2]$ follows from our convergence assumption.
\end{proof}

\cref{asymp-from-cond-var} then follows from \cref{asymp-from-cond-var-uneven} by replacing $|B_j'|$ with $\frac{n}{k}$ and $Z_{B_j}$ with $Z_{1:n(1-1/k)}$ since folds are equal-sized and the $Z_i$'s are i.i.d.

\section{Conditional Variance Convergence from Loss Stability
}\label{sec:proof-cond-var-conv-from-loss-stability}
We show that the quantity appearing in \cref{eq:variance_convergence_assump_k} is controlled by the loss stability, for any $q\in (0,2]$. Note however that \cref{eq:variance_convergence_assump_k} can be satisfied even in a case where the loss stability is infinite (see \cref{sec:proof-loss-vs-mean-square-stability}).
\begin{proposition}[Conditional variance convergence from loss stability]\label{cond-var-conv-from-loss-stability}
Suppose that $k$ divides $n$ evenly. Under the notation of \cref{asymp-from-cond-var} with $\bar{h}_n(z)=\E[h_n(z,Z_{1:n(1-1/k)})]$,
\balignt
\Earg{\left(\frac{1}{\sigma_n^2}\Var_{Z_0}\left(h_n(Z_0,Z_{1:n(1-1/k)})-\bar{h}_n(Z_0)\right)\right)^{q/2}}
\leq
\left(\frac{1}{\sigma_n^2} \frac{1}{2} n(1-1/k) \lstab(h_n)\right)^{q/2},
\ealignt
for any $q\in (0,2]$.
Consequently, the condition \cref{eq:variance_convergence_assump_k} is verified whenever $\lstab(h_n)=o\left(\frac{\sigma_n^2}{n(1-1/k)\max(k,k^{(2/q)-1})}\right)$.
\end{proposition}
\begin{remark}
If $k=O(1)$, this loss stability assumption simplifies to $\lstab(h_n)=o(\sigma_n^2/n)$ for any $q\in (0,2]$.
\end{remark}
\begin{proof}
Write $\ntrain=n(1-1/k)$. Then
\balignt
\Var_{Z_0}\left(h_n(Z_0,Z_{1:m})-\E\left[h_n(Z_0,Z_{1:m})\mid Z_0\right]\right) 
&= \Var_{Z_0}\left(h'_n(Z_0,Z_{1:m})-\E\left[h'_n(Z_0,Z_{1:m})\mid Z_0\right]\right),
\ealignt
since the difference $h_n(Z_0,Z_{1:m})-h'_n(Z_0,Z_{1:m})=\E[h_n(Z_0,Z_{1:m})\mid Z_{1:m}]$ is a $Z_{1:m}$-measurable function.
For $0<q\leq 2$, using Jensen's inequality,
\balignt
&\Earg{\left(\Var_{Z_0}\left(h'_n(Z_0,Z_{1:m})-
\E\left[h'_n(Z_0,Z_{1:m})\mid Z_0\right]\right)\right)^{q/2}} \\
&\leq \Earg{\Var_{Z_0}\left(h'_n(Z_0,Z_{1:m})-
\E\left[h'_n(Z_0,Z_{1:m})\mid Z_0\right]\right)}^{q/2}.
\ealignt
We can bound it using loss stability.
\balignt
& \Var_{Z_0}\left(h'_n(Z_0,Z_{1:m})-\E\left[h'_n(Z_0,Z_{1:m})\mid Z_0\right]\right) \\
&= \E_{Z_0}\big[
\big(
\left(h'_n(Z_0,Z_{1:m})-\E\left[h'_n(Z_0,Z_{1:m})\mid Z_0\right]\right) \\
& \hspace{1cm} - \left(\E[h'_n(Z_0,Z_{1:m})\mid Z_{1:m}]-\E\left[h'_n(Z_0,Z_{1:m})\right]\right)
\big)^2
\big] \\
&= \E[
(
h'_n(Z_0,Z_{1:m})-\E\left[h'_n(Z_0,Z_{1:m})\mid Z_0\right]
)^2\mid Z_{1:m}
],
\ealignt
so that
\balignt
\E\left[
\Var_{Z_0}\left(h'_n(Z_0,Z_{1:m})-\E\left[h'_n(Z_0,Z_{1:m})\mid Z_0\right]\right)
\right]
&= \E[
(
h'_n(Z_0,Z_{1:m})-\E\left[h'_n(Z_0,Z_{1:m})\mid Z_0]
)^2
\right] \\
&= \E[
\Var(h'_n(Z_0,Z_{1:m})\mid Z_0)
] \\
&\leq \frac{1}{2} m \lstab(h_n),
\ealignt
where the last inequality comes from \cref{efron-stein}.
Consequently,
\balignt
\Earg{\left(\frac{1}{\sigma_n^2} \Var_{Z_0}\left(h_n(Z_0,Z_{1:m})-
\E\left[h_n(Z_0,Z_{1:m})\mid Z_0\right]\right)\right)^{q/2}}
\leq
\left(\frac{1}{\sigma_n^2} \frac{1}{2} \ntrain \lstab(h_n)\right)^{q/2}.
\ealignt
\end{proof}

\section{Excess Loss of Sample Mean: $o(\frac{\sigma_n^2}{n})$ loss stability, constant $\sigma_n^2\in(0,\infty)$, infinite mean-square stability}\label{sec:proof-loss-vs-mean-square-stability-excess}
Here we present a very simple learning task in which (i) the CLT conditions of \cref{iid-cv-normal,asymp-from-lstability} hold and (ii) mean-square stability \cref{eq:ms_stab} is infinite. 
\begin{example}[Excess loss of sample mean: $o(\frac{\sigma_n^2}{n})$ loss stability, constant $\sigma_n^2\in(0,\infty)$, infinite mean-square stability]\label{loss-vs-mean-square-stability-excess}
Suppose $(Z_i)_{i\geq 1}$ are independent and identically distributed copies of a random element $Z_0$ with $\E[Z_0] = 0$ and $\E[Z_0^2] < \infty$. Consider $k$-fold cross-validation of the excess loss of the sample mean relative to a constant prediction rule:
\balignt
h_n(z, \D) = (z - \hat{f}(\D))^2 - (z - a)^2
\qtext{where}
\hat{f}(\D) \defeq \frac{1}{|\D|} \sum_{Z_0 \in \D} Z_0
\qtext{and} a\neq 0.
\ealignt
The variance parameter of \cref{iid-cv-normal} $\sigma_n^2 = \Var(\bar{h}_n(Z_0)) =  4a^2\Var(Z_0)$ when $\bar{h}_n(z) = \E[h_n(z, Z_{1:n(1-1/k)})]$, and the loss stability
$\lstab(h_n)=\frac{8\Var(Z_0)^2}{n^2(1-1/k)^2}
=o(\sigma_n^2/n)$. Consequently \cref{asymp-from-lstability} implies asymptotic linearity. The uniform integrability condition of \cref{iid-cv-normal} also holds. Together, these results imply that the CLT of \cref{iid-cv-normal} is applicable.
However, whenever $Z_0$ does not have a fourth moment, the mean-square stability \cref{eq:ms_stab} is infinite.
\end{example}
\begin{proof}
Introduce the shorthand $m = n(1-1/k)$,
fix any $\D$ with $|\D| = m$, and suppose $\D^{Z_0'}$ is formed by swapping $Z_0'$ for an independent point $Z_0''$ in $\D$.
For any $z$ we have
\balignt
(z - \hat{f}(\D))^2 - (z - \hat{f}(\D^{Z_0'}))^2
    &= (\hat{f}(\D) - \hat{f}(\D^{Z_0'})) (2z - \hat{f}(\D) -\hat{f}(\D^{Z_0'})) \\
    &= \frac{1}{m}(Z_0'' - Z_0') (2z - \hat{f}(\D) -\hat{f}(\D^{Z_0'})) \\
    &= \frac{1}{m}(Z_0'' - Z_0') (2z - \frac{1}{m}(Z_0''+Z_0') - 2 (\hat{f}(\D) - \frac{1}{m}Z_0'')) \\
    &= \frac{2}{m}(Z_0'' - Z_0') (z - (\hat{f}(\D) - \frac{1}{m}Z_0'')) - \frac{1}{m^2}(Z_0''^2 - Z_0'^2).
\ealignt
Hence, the mean-square stability equals
\balignt
&\E[((Z_0 - \hat{f}(\D))^2 - (Z_0 - \hat{f}(\D^{Z_0'}))^2)^2] \\
    &= \frac{1}{m^4}\E[(Z_0''^2 - Z_0'^2)^2]
    + \frac{4}{m^2}\E[(Z_0'' - Z_0')^2] \E[(Z_0 - (\hat{f}(\D) - \frac{1}{m}Z_0''))^2] \\
    &- \frac{4}{m^3} \E[(Z_0'' - Z_0')(Z_0''^2 - Z_0'^2)]\E[Z_0 - (\hat{f}(\D) - \frac{1}{m}Z_0'')] \\
    &= \frac{1}{m^4}\E[(Z_0''^2 - Z_0'^2)^2]
    + \frac{4}{m^2}\E[(Z_0'' - Z_0')^2] \E[(Z_0 - (\hat{f}(\D) - \frac{1}{m}Z_0''))^2] \\
    &\geq \frac{2}{m^4}\Var(Z_0^2)
\ealignt
since $\E[Z_0] = 0$, and $Z_0,Z_0',Z_0'', \hat{f}(\D) - \frac{1}{m}Z_0''$ are mutually independent.

Moreover, the loss stability equals
\balignt
&\E[((Z_0 - \hat{f}(\D))^2 - (Z_0 - \hat{f}(\D^{Z_0'}))^2 - \E_{Z_0}[(Z_0 - \hat{f}(\D))^2 - (Z_0 - \hat{f}(\D^{Z_0'}))^2])^2] \\
    &= \frac{4}{m^2}\E[(Z_0'' - Z_0')^2] \E[(Z_0 - \E[Z_0])^2] = \frac{8}{m^2}\Var(Z_0)^2.
\ealignt

Finally, for any $z$, $\E[(z - \hat{f}(\D))^2 - (z - a)^2]
= \frac{1}{m} \Var(Z_0) + 2za - a^2$. Consequently, for $\bar{h}_n(Z_0)=\E[h_n(Z_0, \D)\mid Z_0]$, we get the following equalities:
\begin{talign}
&\sigma_n^2=\Var(\bar{h}_n(Z_0))=4a^2\Var(Z_0), \qtext{and} \\
&(\bar{h}_n(Z_0)-\E[\bar{h}_n(Z_0)])^2/\sigma_n^2
= Z_0^2 / \Var(Z_0).
\end{talign}
The distribution of $Z_0^2/\Var(Z_0)$ does not depend on $n$ and is integrable, so the sequence of $(\bar{h}_n(Z_0)-\E[\bar{h}_n(Z_0)])^2/\sigma_n^2$ is uniformly integrable.
\end{proof}

\section{Loss of Surrogate Mean: constant $\sigma_n^2\in(0,\infty)$, infinite $\tilde{\sigma}_n^2$, vanishing conditional variance}\label{sec:proof-loss-vs-mean-square-stability}
The following example details a simple task in which (i) the CLT conditions of \cref{iid-cv-normal} and \cref{asymp-from-cond-var} hold and (ii) mean-square stability, $\tilde{\sigma}_n^2$, and loss stability are infinite.
\begin{example}[Loss of surrogate mean: constant $\sigma_n^2\in(0,\infty)$, infinite $\tilde{\sigma}_n^2$, vanishing conditional variance]\label{loss-vs-mean-square-stability}
Suppose $(Z_i)_{i\geq 1}$ are independent and identically distributed copies of a random element $Z_0=(X_0,Y_0)$ with $Z_i = (X_i, Y_i)$ and  $\E[X_0] = \E[Y_0]$. Consider $k$-fold  cross-validation of the following prediction rule under squared error loss:
\balignt
h_n((x,y), \D) = (y - \hat{f}(\D))^2
\qtext{where}
\hat{f}(\D) \defeq \frac{1}{|\D|} \sum_{(X_0,Y_0) \in \D} X_0.
\ealignt
The loss stability
$\lstab(h_n)=\frac{8\Var(X_0)\Var(Y_0)}{n^2(1-1/k)^2}$, and the variance parameter of \cref{iid-cv-normal}
\balignt\label{eq:finite-variance}
\sigma_n^2 = \Var(\bar{h}_n(Z_0)) =  \Var((Y_0-\E[Y_0])^2),
\ealignt
when $\bar{h}_n(z) = \E[h_n(z, Z_{1:n(1-1/k)})]$.
Hence, if $\E[X_0^2], \E[Y_0^4] < \infty$, then $\lstab(h_n)=o(\sigma_n^2/n)$  and \cref{asymp-from-lstability} implies asymptotic linearity. The uniform integrability condition of \cref{iid-cv-normal} also holds. Together, these results imply that the CLT of \cref{iid-cv-normal} is applicable.
 
If $X_0$ has no fourth moment, then the mean-square stability \cref{eq:ms_stab} is infinite.

If $X_0$ has no second moment, then the loss stability
and the \cite[Theorem 1]{MA-WZ:2020} variance parameter
\balignt
\tilde{\sigma}_n^2 = \E[\Var(h_n(Z_0, Z_{1:n(1-1/k)}) \mid Z_{1:n(1-1/k)})]  = \Var((Y_0-\E[Y_0])^2) +  \frac{8\Var(X_0)\Var(Y_0)}{n(1-1/k)}.
\ealignt
are infinite.
However,
\balignt
&\sqrt{k}
    \E\left[\sqrt{\frac{1}{\sigma_n^2}\Var_{Z_0}\left(h_n(Z_0,Z_{1:n(1-1/k)})-\bar{h}_n(Z_0)\right)}\right] \\
&=
    2\sqrt{k}\sqrt{\frac{\Var(Y_0)}{\Var((Y_0-\E[Y_0])^2)}}\E[|\hat{f}(Z_{1:n(1-1/k)})-\E[X_0]|].
\ealignt
Hence, if $\E[Y_0^4]<\infty$ and $k=O(1)$, $L^1$ asymptotic linearity   follows from \cref{asymp-from-cond-var}, the uniform integrability condition of \cref{iid-cv-normal} still holds, and the CLT of \cref{iid-cv-normal} holds with the finite variance parameter \cref{eq:finite-variance}.
\end{example}
\begin{proof}
Without loss of generality, we will assume $\E[X_0] = \E[Y_0] = 0$; the formulas in the general case are obtained by replacing $X_0$ with $X_0-\E[X_0]$ and similarly for $Y_0$.
Introduce the shorthand $m = n(1-1/k)$,
fix any $\D$ with $|\D| = m$, and suppose $\D^{Z_0'}$ is formed by swapping $Z_0'$ for an independent point $Z_0''$ in $\D$.
For any $z = (x,y)$ we have
\balignt
(y - \hat{f}(\D))^2 - (y - \hat{f}(\D^{Z_0'}))^2
    &= (\hat{f}(\D) - \hat{f}(\D^{Z_0'})) (2y - \hat{f}(\D) -\hat{f}(\D^{Z_0'})) \\
    &= \frac{1}{m}(X_0'' - X_0') (2y - \hat{f}(\D) -\hat{f}(\D^{Z_0'})) \\
    &= \frac{1}{m}(X_0'' - X_0') (2y - \frac{1}{m}(X_0''+X_0') - 2 (\hat{f}(\D) - \frac{1}{m}X_0'')) \\
    &= \frac{2}{m}(X_0'' - X_0') (y - (\hat{f}(\D) - \frac{1}{m}X_0'')) - \frac{1}{m^2}(X_0''^2 - X_0'^2).
\ealignt
Hence, the mean-square stability equals
\balignt
&\E[((Y_0 - \hat{f}(\D))^2 - (Y_0 - \hat{f}(\D^{Z_0'}))^2)^2] \\
    &= \frac{1}{m^4}\E[(X_0''^2 - X_0'^2)^2]
    + \frac{4}{m^2}\E[(X_0'' - X_0')^2] \E[(Y_0 - (\hat{f}(\D) - \frac{1}{m}X_0''))^2] \\
    &- \frac{4}{m^3} \E[(X_0'' - X_0')(X_0''^2 - X_0'^2)]\E[Y_0 - (\hat{f}(\D) - \frac{1}{m}X_0'')] \\
    &= \frac{1}{m^4}\E[(X_0''^2 - X_0'^2)^2]
    + \frac{4}{m^2}\E[(X_0'' - X_0')^2] \E[(Y_0 - (\hat{f}(\D) - \frac{1}{m}X_0''))^2] \\
    &\geq \frac{2}{m^4}\Var((X_0-\E[X_0])^2)
\ealignt
since $\E[Y_0] = 0$, and $Z_0,Z_0',Z_0'', \hat{f}(\D) - \frac{1}{m}X_0''$ are mutually independent.

Moreover, the loss stability equals
\balignt
&\E[((Y_0 - \hat{f}(\D))^2 - (Y_0 - \hat{f}(\D^{Z_0'}))^2 - \E_Y[(Y_0 - \hat{f}(\D))^2 - (Y_0 - \hat{f}(\D^{Z_0'}))^2])^2] \\
    &= \frac{4}{m^2}\E[(X_0'' - X_0')^2] \E[(Y_0 - \E[Y_0])^2] = \frac{8}{m^2}\Var(X_0)\Var(Y_0).
\ealignt

Next note that, for any $y,y'$,
\balignt
\E[(y - \hat{f}(\D))^2 - (y'- \hat{f}(\D))^2]
    &= (y-y') (y+y' - 2\E[\hat{f}(\D)]) \\
    &= (y^2-y'^2)-2(y-y')\E[\hat{f}(\D)]
    = y^2-y'^2
\ealignt
since $\E[X_0] = 0$.
Therefore, 
$\Var(\E[h_n(Z_0, \D)\mid Z_0]) = \half \E[(Y_0^2-Y_0'^2)^2] = \Var(Y_0^2)$.

For any $y$, $\E[(y-\hat{f}(\D))^2]=y^2+\E[\hat{f}(\D)^2]$ since $\E[X_0]=0$.
Consequently, for $\bar{h}_n(Z_0)=\E[h_n(Z_0, \D)\mid Z_0]$, we get the following equalities:
\begin{talign}
&\sigma_n^2=\Var(\bar{h}_n(Z_0))=\Var(Y_0^2), \qtext{and} \\
&(\bar{h}_n(Z_0)-\E[\bar{h}_n(Z_0)])^2/\sigma_n^2
= (Y_0^2 - \E[Y_0^2])^2/\Var(Y_0^2).
\end{talign}
The distribution of $(Y_0^2 - \E[Y_0^2])^2/\Var(Y_0^2)$ does not depend on $n$ and is integrable, so the sequence of $(\bar{h}_n(Z_0)-\E[\bar{h}_n(Z_0)])^2/\sigma_n^2$ is uniformly integrable.

Since
\balignt
(y - \hat{f}(\D))^2 - (y'- \hat{f}(\D))^2
    &= (y^2-y'^2)-2(y-y')\hat{f}(\D)
\ealignt
we can compute the variance parameter of \cite{MA-WZ:2020},
\begin{talign}
\tilde{\sigma}_n^2 
    &= \E[\Var(h_n(Z_0, \D) \mid \D)] 
    = \E[ (h_n(Z_0, \D) - \E[h_n(Z_0,\D)|\D])^2 ] \\
    &= \half \E[ (h_n(Z_0, \D) - h_n(Z_0',\D))^2 ] \\
    &= \half \E((Y_0^2-Y_0'^2)^2] + 4\E[(Y_0-Y_0')^2]\E[\hat{f}(\D)^2] - \E[(y^2-y'^2)(y-y')]\E[\hat{f}(\D)] \\
    &= \Var(Y_0^2) + 8 \Var(Y_0) \frac{1}{m}\Var(X_0),
\end{talign}
since $\E[\hat{f}(\D)]=0$ and  $\hat{f}(\D),Y_0,Y_0'$ are mutually independent.

Finally, let's compute $\sqrt{k} \,
    \E\left[\sqrt{\frac{1}{\sigma_n^2}\Var_{Z_0}\left(h_n(Z_0,\D)-\bar{h}_n(Z_0)\right)}\right]$.
    
For any $y,y'$,
\balignt
&((y - \hat{f}(\D))^2 - \E[(y - \hat{f}(\D))^2]) - ((y'- \hat{f}(\D))^2 - \E[(y' - \hat{f}(\D))^2]) \\
    &= (y^2-y'^2)-2(y-y')\hat{f}(\D) - (y^2 - y'^2) \\
    &= -2(y-y')\hat{f}(\D),
\ealignt
so that
$\Var_{Z_0}\left(h_n(Z_0,\D)-\bar{h}_n(Z_0)\right) = \frac{1}{2}\E[(-2(Y_0-Y_0')\hat{f}(\D))^2\mid \D] = 4\Var(Y_0)(\hat{f}(\D))^2.
$

Then
\balignt\label{eq:cond-conv-L1}
\sqrt{k}
    \E\left[\sqrt{\frac{1}{\sigma_n^2}\Var_{Z_0}\left(h_n(Z_0,\D)-\bar{h}_n(Z_0)\right)}\right]=
    2\sqrt{k}\sqrt{\frac{\Var(Y_0)}{\Var(Y_0^2)}}\E[|\hat{f}(\D)|].
\ealignt
If $\E[|X_0|] < \infty$, the family of empirical averages $\{\frac{1}{m}\sum_{i=1}^m X_i: m\geq 1\}$ is uniformly integrable and the weak law of large numbers implies that $\hat{f}(\D)$ converges to 0 in probability.
Hence, $\hat{f}(\D) \toL{1} 0$.
The quantity \cref{eq:cond-conv-L1} then goes to zero when $k=O(1)$.
\end{proof}

\section{Proof of \cref{var-comp}: Variance comparison}\label{sec:proof-var-comp}
\cref{var-comp} will follow from the following more general result.
\begin{proposition}\label{var-comp-general}
Fix any $j \in [k]$, and define $\sigma_{n,j}^2 \defeq \Var(\E[h_n(Z_0, Z_{B_j})\mid Z_0])$ and $\tilde{\sigma}_{n,j}^2 \defeq \E[\Var(h_n(Z_0, Z_{B_j}) \mid Z_{B_j})]$. Then
\balignt
\sigma_{n,j}^2\leq \tilde{\sigma}_{n,j}^2 \leq \sigma_{n,j}^2 + \frac{|B_j|}{2} \lstab(h_n),
\ealignt
where the first inequality is strict whenever $h_n'(Z_0, Z_{B_j}) = h_n(Z_0, Z_{B_j}) - \E[h_n(Z_0, Z_{B_j}) \mid Z_{B_j}]$ depends on $Z_{B_j}$.
\end{proposition}
\begin{proof} For all $j \in [k]$, we can rewrite both variance parameters.
\balignt
\tilde{\sigma}_{n,j}^2 
    &= \E[\Var(h_n(Z_0, Z_{B_j})\mid Z_{B_j})] \\
    &= \E[ (h_n(Z_0, Z_{B_j}) - \E[h_n(Z_0,Z_{B_j})\mid Z_{B_j}])^2 ] \\
    &= \E[h'_n(Z_0,Z_{B_j})^2]
    = \Var(h'_n(Z_0,Z_{B_j})).
\ealignt
\balignt
\sigma_{n,j}^2 
    &= \Var(\E[h_n(Z_0, Z_{B_j})\mid Z_0]) \\
    &= \E[ (\E[h_n(Z_0, Z_{B_j})\mid Z_0] - \E[h_n(Z_0,Z_{B_j})])^2 ] \\
    &= \E[ \E[h'_n(Z_0,Z_{B_j})\mid Z_0]^2]
    = \Var(\E[h'_n(Z_0,Z_{B_j})\mid Z_0]) \\
    &= \Var(h'_n(Z_0,Z_{B_j})) - \E[\Var(h'_n(Z_0,Z_{B_j})\mid Z_0)] \\
    &= \tilde{\sigma}_{n,j}^2 - \E[\Var(h'_n(Z_0,Z_{B_j})\mid Z_0)]
    \leq \tilde{\sigma}_{n,j}^2,
\ealignt
where the final inequality is strict whenever $\E[\Var(h'_n(Z_0,Z_{B_j})\mid Z_0)]$ is non-zero.

Since every non-constant variable has either infinite or strictly positive variance, $\E[\Var(h'_n(Z_0,Z_{B_j})\mid Z_0)] = 0 \iff h'_n(Z_0,Z_{B_j}) = \E[h'_n(Z_0,Z_{B_j})\mid Z_0]$, that is, if and only if $h'_n(Z_0,Z_{B_j}) = h_n(Z_0,Z_{B_j}) - \E[h_n(Z_0,Z_{B_j}) \mid Z_{B_j}]$ is independent of $Z_{B_j}$.

Finally, we know from \cref{efron-stein} that the difference  $\tilde{\sigma}_{n,j}^2- \sigma_{n,j}^2 = \E[\Var(h'_n(Z_0,Z_{B_j})\mid Z_0)] \leq \half |B_j|\lstab(h_n)$.
\end{proof}

\cref{var-comp} then follows from \cref{var-comp-general} since the $Z_i$'s are \iid\ and, when $k$ divides $n$, the only possible size for $B_j$ is $n(1-1/k)$.
\section{Proof of \cref{consistent-variance-est-in}: Consistent within-fold estimate of asymptotic variance}
\label{sec:proof-consistent-variance-est-in}
We will prove the following more detailed statement from which \cref{consistent-variance-est-in} will follow.

\begin{theorem}[Consistent within-fold estimate of asymptotic variance]
\label{consistent-variance-est-in-detailed}
Suppose that $k \leq n/2$ and that $k$ divides $n$ evenly.
Under the notation of \cref{iid-cv-normal} with $\ntrain = n(1-1/k)$, $\bar{h}_n(z) = \E[h_n(z,Z_{1:\ntrain})]$, $h_n'(Z_0,Z_{1:\ntrain}) = h_n(Z_0,Z_{1:\ntrain}) - \E[h_n(Z_0,Z_{1:\ntrain}) \mid Z_{1:\ntrain}]$ and $\bar{h}_n'(z) = \E[h_n'(z,Z_{1:\ntrain})]$, define the within-fold variance estimate
\balignt
\sigin^2 \defeq
\frac{1}{k}\sum_{j=1}^k
\frac{1}{(n/k)-1}\sum_{i\in B_j'}
\left(
h_n(Z_i,Z_{B_j}) - \frac{k}{n}\sum_{i'\in B_j'}h_n(Z_{i'},Z_{B_j})
\right)^2.
\ealignt
If $(Z_i)_{i\geq 1}$ are \iid copies of a random element $Z_0$, then %
\balignt
\E[|\sigin^2 - \sigma_n^2|]
\leq \frac{2 n^2}{n-k} \lstab(h_n) + 2\sqrt{\frac{2 n^2}{n-k} \lstab(h_n) \sigma_n^2} + \sqrt{\frac{1}{n}\E[\bar{h}_n'(Z_0)^4]
+ \frac{3k-n}{n(n-k)} \sigma_n^4}
\ealignt
and there exists an absolute constant $C$ specified in the proof such that
\balignt
\E[(\sigin^2 - \sigma_n^2)^2]
&\leq 4\frac{C n^4}{(n-k)^2} \gamma_{4}(h_n') + 8\sqrt{\frac{C n^4}{(n-k)^2} \gamma_{4}(h_n')(\frac{1}{n}\E[\bar{h}_n'(Z_0)^4] + \frac{3k-n}{n(n-k)} \sigma_n^4 + \sigma_n^4)} \\
&\quad + 2(\frac{1}{n}\E[\bar{h}_n'(Z_0)^4]
+ \frac{3k-n}{n(n-k)} \sigma_n^4)
\label{eq:sigin-bound-L2}
\ealignt
where $\gamma_{4}(h_n') \defeq \frac{1}{\ntrain}\sum_{i=1}^\ntrain\E[(h_n'(Z_0,Z_{1:\ntrain})-h_n'(Z_0,Z_{1:\ntrain}^{\backslash i}))^4]$.
Here, $Z^{\backslash i}_{1:\ntrain}$ denotes $Z_{1:\ntrain}$ with $Z_i$ replaced by an \iid copy independent of $Z_{0:\ntrain}$.

Moreover,
\balignt
\E[|\sigin^2 - \sigma_n^2|]
\leq \frac{2 n^2}{n-k} \lstab(h_n) + 2\sqrt{\frac{2 n^2}{n-k} \lstab(h_n) \sigma_n^2} + \sqrt{\frac{2}{n(n/k-1)}\sigma_n^4} + o(\sigma_n^2)
\label{eq:sigin-bound-L1}
\ealignt
whenever the sequence of
$(\bar{h}_n(Z_0)-\E[\bar{h}_n(Z_0)])^2/\sigma_n^2$
is uniformly integrable.
\end{theorem}

\begin{proof}
\paragraph{Eliminating training set randomness}
We begin by approximating our variance estimate
\begin{talign}
\sigin^2
&=
\frac{1}{k}\sum_{j=1}^k
\frac{1}{(n/k)-1}\sum_{i\in B_j'}
\left(
h_n(Z_i,Z_{B_j}) - \frac{k}{n}\sum_{i'\in B_j'}h_n(Z_{i'},Z_{B_j})
\right)^2\\
&=
\frac{1}{k}\sum_{j=1}^k
\frac{1}{(n/k)-1}\sum_{i\in B_j'}
\left(
h_n'(Z_i,Z_{B_j}) - \frac{k}{n}\sum_{i'\in B_j'}h_n'(Z_{i'},Z_{B_j})
\right)^2
\end{talign}
by a quantity eliminating training set randomness in each summand,
\balignt
\siginapprox^2 \defeq
\frac{1}{k}\sum_{j=1}^k
\frac{1}{(n/k)-1}\sum_{i\in B_j'}
\left(
\bar{h}_n'(Z_i) - \frac{k}{n}\sum_{i'\in B_j'}\bar{h}_n'(Z_{i'})
\right)^2,
\ealignt
where $\bar{h}_n'(z)=\E[h_n'(z,Z_{1:\ntrain})]$. Note that $\bar{h}_n'(Z_0)$ has expectation 0.

By Cauchy--Schwarz,  
we have 
\begin{talign}
|\sigin^2 - \siginapprox^2|
    &\leq \Delta + 2\sqrt{\Delta}\siginapprox
\end{talign}
for the error term
\balignt
\Delta 
&\defeq \frac{1}{k}\sum_{j=1}^k
\frac{1}{(n/k)-1}\sum_{i\in B_j'}
\left(
h_n'(Z_i,Z_{B_j}) - \bar{h}_n'(Z_i) + \frac{k}{n}\sum_{i'\in B_j'}\bar{h}_n'(Z_{i'}) - \frac{k}{n}\sum_{i'\in B_j'}h_n'(Z_{i'},Z_{B_j})
\right)^2\\
&\leq 
2\frac{1}{k}\sum_{j=1}^k
\frac{1}{(n/k)-1}\sum_{i\in B_j'}
(h_n'(Z_i,Z_{B_j}) - \bar{h}_n'(Z_i))^2\\
&\quad + 2\frac{1}{k}\sum_{j=1}^k
\frac{1}{(n/k)-1}\sum_{i\in B_j'}
\left(
\frac{k}{n}\sum_{i'\in B_j'}(
\bar{h}_n'(Z_{i'}) - h_n'(Z_{i'},Z_{B_j}))
\right)^2\\
&\leq 
\frac{2}{n-k}\sum_{j=1}^k\sum_{i\in B_j'}
(h_n'(Z_i,Z_{B_j}) - \bar{h}_n'(Z_i))^2\\
&\quad + 2\frac{1}{k}\sum_{j=1}^k
\frac{1}{(n/k)-1}\sum_{i\in B_j'}
\frac{k}{n}\sum_{i'\in B_j'}
(\bar{h}_n'(Z_{i'}) - h_n'(Z_{i'},Z_{B_j}))^2\\
&=\frac{4}{n-k}\sum_{j=1}^k\sum_{i\in B_j'}
(h_n'(Z_i,Z_{B_j}) - \bar{h}_n'(Z_i))^2
\label{eq:delta-in-bound}
\ealignt
where we have used Jensen's inequality twice.

Thus,
\balignt
\Delta^2
&\leq \frac{16n^2}{(n-k)^2}\frac{1}{n}\sum_{j=1}^k\sum_{i\in B_j'} (h_n'(Z_i,Z_{B_j})-\bar{h}_n'(Z_i))^4 \\
&= \frac{16n}{(n-k)^2}\sum_{j=1}^k\sum_{i\in B_j'} (h_n'(Z_i,Z_{B_j})-\bar{h}_n'(Z_i))^4
\label{eq:delta-squared-in-bound}
\ealignt
by Jensen's inequality.

\paragraph{Controlling the error $\Delta$}
We will first control the error term $\Delta$.
By the bound \cref{eq:delta-in-bound} and the conditional Efron--Stein inequality (\cref{efron-stein}), we have
\balignt
\E[\Delta]
&\leq \frac{4}{n-k}\sum_{j=1}^k\sum_{i\in B_j'}
\E[(h_n'(Z_i,Z_{B_j}) - \bar{h}_n'(Z_i))^2]\\
&= \frac{4}{n-k}\sum_{j=1}^k\sum_{i\in B_j'}
\E[(h_n'(Z_i,Z_{B_j}) - \E[h_n'(Z_i,Z_{B_j}) \mid Z_i])^2]\\
&= \frac{4}{n-k}\sum_{j=1}^k\sum_{i\in B_j'}
\E[\E[(h_n'(Z_i,Z_{B_j}) - \E[h_n'(Z_i,Z_{B_j}) \mid Z_i])^2 \mid Z_i ]]\\
&= \frac{4}{n-k}\sum_{j=1}^k\sum_{i\in B_j'}
\E[\Var(h_n'(Z_i,Z_{B_j}) \mid Z_i) ]\\
&\leq \frac{2}{n-k}\sum_{j=1}^k\sum_{i\in B_j'}
\ntrain \msstab(h'_n)\\
&\leq \frac{2 n^2}{n-k} \msstab(h'_n)
= \frac{2 n^2}{n-k} \lstab(h_n).
\label{eq:mean-delta-in-bound}
\ealignt

\paragraph{Controlling $\Delta^2$} In the following, for any $j\in [k]$ and $i' \in B_j$, $Z_{B_j}^{\backslash i'}$ is $Z_{B_j}$ with $Z_{i'}$ replaced by $Z_0$. By the bound \cref{eq:delta-squared-in-bound} and \citet[Thm.\ 2]{boucheron-etal:2005}, and by noting that $x^4 = x_{+}^4+x_{-}^4$ where $x_{+}=\max(x,0)$ and $x_{-}=\max(-x,0)$, we have
\balignt
\E[\Delta^2]
&\leq \frac{16n}{(n-k)^2}\sum_{j=1}^k\sum_{i\in B_j'} \E[(h_n'(Z_i,Z_{B_j})-\bar{h}_n'(Z_i))^4] \\
&= \frac{16n}{(n-k)^2}\sum_{j=1}^k\sum_{i\in B_j'} \E[\E[(h_n'(Z_i,Z_{B_j})-\E[h_n'(Z_i,Z_{B_j}) \mid Z_i])^4 \mid Z_i]] \\
&= \frac{16n}{(n-k)^2}
\sum_{j=1}^k\sum_{i\in B_j'} \Big(\E[\E[(h_n'(Z_i,Z_{B_j})-\E[h_n'(Z_i,Z_{B_j}) \mid Z_i])^4_{+} \mid Z_i]] \\
&\hspace{3.5cm} + \E[\E[(h_n'(Z_i,Z_{B_j})-\E[h_n'(Z_i,Z_{B_j}) \mid Z_i])^4_{-} \mid Z_i]]\Big) \\
&\leq \frac{16n}{(n-k)^2}(1-\frac{1}{4})^2 4 (\frac{8}{7})^2 16
\sum_{j=1}^k\sum_{i\in B_j'} \Big(\E[(\E[\sum_{i' \in B_j}(h_n'(Z_i,Z_{B_j})-h_n'(Z_i,Z_{B_j}^{\backslash i'}))^2_{+} \mid Z_{B_j}])^2] \\
&\hspace{3.5cm} + \E[(\E[\sum_{i' \in B_j}(h_n'(Z_i,Z_{B_j})-h_n'(Z_i,Z_{B_j}^{\backslash i'}))^2_{-} \mid Z_{B_j}])^2]\Big) \\
&\leq \frac{16n}{(n-k)^2}\frac{2304}{49}
\sum_{j=1}^k\sum_{i\in B_j'} \Big(\E[(\sum_{i' \in B_j}(h_n'(Z_i,Z_{B_j})-h_n'(Z_i,Z_{B_j}^{\backslash i'}))^2_{+})^2] \\
&\hspace{4.1cm} + \E[(\sum_{i' \in B_j}(h_n'(Z_i,Z_{B_j})-h_n'(Z_i,Z_{B_j}^{\backslash i'}))^2_{-})^2]\Big) \\
&\leq \frac{36864n}{49(n-k)^2}m
\sum_{j=1}^k\sum_{i\in B_j'} \Big(\sum_{i' \in B_j}\E[(h_n'(Z_i,Z_{B_j})-h_n'(Z_i,Z_{B_j}^{\backslash i'}))^4_{+}] \\
&\hspace{4.1cm} + \sum_{i' \in B_j}\E[(h_n'(Z_i,Z_{B_j})-h_n'(Z_i,Z_{B_j}^{\backslash i'}))^4_{-}]\Big) \\
&= \frac{36864n}{49(n-k)^2}m
\sum_{j=1}^k\sum_{i\in B_j'} \sum_{i' \in B_j}\E[(h_n'(Z_i,Z_{B_j})-h_n'(Z_i,Z_{B_j}^{\backslash i'}))^4] \\
&\leq \frac{36864n^4}{49(n-k)^2} \gamma_{4}(h_n') = \frac{C n^4}{(n-k)^2} \gamma_{4}(h_n')
\label{eq:mean-delta-squared-in-bound}
\ealignt
where $\gamma_{4}(h_n') = \frac{1}{\ntrain}\sum_{i=1}^\ntrain\E[(h_n'(Z_0,Z_{1:\ntrain})-h_n'(Z_0,Z_{1:\ntrain}^{\backslash i}))^4]$ and $C=36864/49$.

\paragraph{Controlling the error $\siginapprox^2 - \sigma_n^2$}
To control the error $\siginapprox^2 - \sigma_n^2$, we first rewrite $\siginapprox^2$ as
\balignt
\siginapprox^2
    &= \frac{1}{k}\sum_{j=1}^k
    \frac{1}{(n/k)-1}\sum_{i\in B_j'}
    \left(
    \bar{h}_n'(Z_i) - \frac{k}{n}\sum_{i'\in B_j'}\bar{h}_n'(Z_{i'})
    \right)^2\\
    &= \frac{1}{k}\sum_{j=1}^k
    \frac{n}{n-k}\frac{k}{n}\sum_{i\in B_j'}
    \left(
    \bar{h}_n'(Z_i) - \frac{k}{n}\sum_{i'\in B_j'}\bar{h}_n'(Z_{i'})
    \right)^2\\
    &= \frac{n}{n-k}\frac{1}{k}\sum_{j=1}^k
    \left(\frac{k}{n}\sum_{i\in B_j'} \bar{h}_n'(Z_i)^2
    -
    \left(\frac{k}{n}\sum_{i\in B_j'}
    \bar{h}_n'(Z_i)\right)^2\right)\\
    &= \frac{n}{n-k}\frac{1}{n}\sum_{j=1}^k\sum_{i\in B_j'} \bar{h}_n'(Z_i)^2
    -
    \frac{n}{n-k}\frac{1}{k}\sum_{j=1}^k
    \left(\frac{k}{n}\sum_{i\in B_j'}
    \bar{h}_n'(Z_i)\right)^2.
\ealignt

We rewrite it once again to find
\balignt
\siginapprox^2
    &= \frac{1}{n}\sum_{j=1}^k\sum_{i\in B_j'}\bar{h}_n'(Z_i)^2
    -\frac{1}{k}\sum_{j=1}^k W_{j} \\
    &= \frac{1}{n}\sum_{i=1}^n\bar{h}_n'(Z_i)^2 
    -\frac{1}{k}\sum_{j=1}^k W_{j}
    \label{eq:rewrite-siginapprox}
\ealignt
where
\balignt
W_{j} \defeq \frac{1}{\binom{n/k}{2}}
        \sum_{\substack{i,i'\in B_j' \\ i<i'}}
        \bar{h}_n'(Z_i)\bar{h}_n'(Z_{i'}).
\ealignt
Since $(W_{j})_{j\in [k]}$ are \iid with mean $0$ and for $i_1<i_1'$ and $i_2<i_2'$
\balignt
\E[\bar{h}_n'(Z_{i_1})\bar{h}_n'(Z_{i_1'})\bar{h}_n'(Z_{i_2})\bar{h}_n'(Z_{i_2'})] = 0
\ealignt
whenever $i_1\neq i_2$ or $i_1' \neq i_2'$, we have
\balignt
\E[(\frac{1}{k}\sum_{j=1}^k W_{j})^2]
    &= \frac{1}{k}\Var(W_{1})
    = \frac{1}{k}\frac{1}{\binom{n/k}{2}^2}
    \sum_{\substack{i,i'\in B_j' \\ i<i'}}
    \E[\bar{h}_n'(Z_i)^2 \bar{h}_n'(Z_{i'})^2] \\
    &= \frac{1}{k}\frac{1}{\binom{n/k}{2}^2}
    \sum_{\substack{i,i'\in B_j' \\ i<i'}}
    \E[\bar{h}_n'(Z_i)^2]\E[\bar{h}_n'(Z_{i'})^2] \\
    &= \frac{1}{k}\frac{1}{\binom{n/k}{2}^2}
    \sum_{\substack{i,i'\in B_j' \\ i<i'}} \sigma_n^4
    = \frac{1}{k}\frac{1}{\binom{n/k}{2}} \sigma_n^4
    = \frac{2}{n(n/k-1)} \sigma_n^4 \label{eq:var-W_{j}}
\ealignt
by noticing that $\E[\bar{h}_n'(Z_0)^2] = \Var(\bar{h}_n'(Z_0)) = \Var(\bar{h}_n(Z_0)) = \sigma_n^2$.

Moreover, by the independence of our datapoints, we have
\balignt
\E[\bar{h}_n'(Z_{i_1})^2\bar{h}_n'(Z_{i_2})\bar{h}_n'(Z_{i_2'})] = 0
\ealignt
for all $i_1,i_2,i_2' \in[n]$ such that $i_2<i_2'$, and thus
\balignt
\E[(\siginapprox^2 - \sigma_n^2)^2]
&= \Var(\siginapprox^2) \\
&= \Var(\frac{1}{n}\sum_{i=1}^n \bar{h}_n'(Z_i)^2)
 + \E[(\frac{1}{k}\sum_{j=1}^k W_{j})^2] \\
&= \frac{1}{n}\Var(\bar{h}_n'(Z_0)^2) 
+ \frac{2}{n(n/k-1)} \sigma_n^4 \\
&= \frac{1}{n}\E[\bar{h}_n'(Z_0)^4]
+ \frac{3k-n}{n(n-k)} \sigma_n^4.
\label{eq:final-approx-in}
\ealignt

\paragraph{Putting the pieces together}
We have
\balignt
\E[\sqrt{\Delta}\siginapprox]
\leq \sqrt{\E[\Delta]\E[\siginapprox^2]}
\leq \sqrt{\frac{2 n^2}{n-k} \lstab(h_n) \sigma_n^2}
\ealignt
by Cauchy--Schwarz and the bound \cref{eq:mean-delta-in-bound}.

We also have
\balignt
\E[\Delta\siginapprox^2]
&\leq \sqrt{\E[\Delta^2]\E[\siginapprox^4]} \\
&= \sqrt{\E[\Delta^2](\Var(\siginapprox^2)+\E[\siginapprox^2]^2)} \\
&\leq \sqrt{\frac{C n^4}{(n-k)^2} \gamma_{4}(h_n')(\frac{1}{n}\E[\bar{h}_n'(Z_0)^4] + \frac{3k-n}{n(n-k)} \sigma_n^4 + \sigma_n^4)}
\ealignt
by Cauchy-Schwarz, \cref{eq:mean-delta-squared-in-bound} and \cref{eq:final-approx-in}.

Assembling our results with the triangle inequality and Cauchy--Schwarz for the $L^1$ bound and with Jensen's inequality for the $L^2$ bound, we find that
\balignt
\E[|\sigin^2 - \sigma_n^2|]
&\leq 
\E[|\sigin^2 - \siginapprox^2|]
+ \E[|\siginapprox^2 - \sigma_n^2|] \\
&\leq \E[\Delta] + 2\E[\sqrt{\Delta}\siginapprox]
+ \sqrt{\E[(\siginapprox^2 - \sigma_n^2)^2]} \\
&\leq \frac{2 n^2}{n-k} \lstab(h_n) + 2\sqrt{\frac{2 n^2}{n-k} \lstab(h_n) \sigma_n^2} + \sqrt{\frac{1}{n}\E[\bar{h}_n'(Z_0)^4]
+ \frac{3k-n}{n(n-k)} \sigma_n^4}
\ealignt
and
\balignt
\E[(\sigin^2 - \sigma_n^2)^2]
&\leq 
2\E[(\sigin^2 - \siginapprox^2)^2]
+ 2\E[(\siginapprox^2 - \sigma_n^2)^2] \\
&\leq 
4\E[\Delta^2] + 8\E[\Delta\siginapprox^2]
+ 2\E[(\siginapprox^2 - \sigma_n^2)^2] \\
&\leq 
4\frac{C n^4}{(n-k)^2} \gamma_{4}(h_n') + 8\sqrt{\frac{C n^4}{(n-k)^2} \gamma_{4}(h_n')(\frac{1}{n}\E[\bar{h}_n'(Z_0)^4] + \frac{3k-n}{n(n-k)} \sigma_n^4 + \sigma_n^4)} \\
&\quad + 2(\frac{1}{n}\E[\bar{h}_n'(Z_0)^4]
+ \frac{3k-n}{n(n-k)} \sigma_n^4)
\ealignt
as advertised.

In order to get the bound
\balignt
\E[|\sigin^2 - \sigma_n^2|]
\leq \frac{2 n^2}{n-k} \lstab(h_n) + 2\sqrt{\frac{2 n^2}{n-k} \lstab(h_n) \sigma_n^2} + \sqrt{\frac{2}{n(n/k-1)}\sigma_n^4} + o(\sigma_n^2)
\ealignt
whenever the sequence of
$(\bar{h}_n(Z_0)-\E[\bar{h}_n(Z_0)])^2/\sigma_n^2$ is uniformly integrable, i.e., the sequence of $\bar{h}_n'(Z_0)^2/\sigma_n^2$ is uniformly integrable, we need to argue that $\frac{1}{n}\sum_{i=1}^n \bar{h}_n'(Z_i)^2 / \sigma_n^2 \toL{1} 1$. Indeed, thanks to \cref{eq:rewrite-siginapprox} and \cref{eq:var-W_{j}}, this will lead to $\E[|\siginapprox^2 - \sigma_n^2|] \leq \sqrt{\frac{2}{n(n/k-1)}\sigma_n^4} + o(\sigma_n^2)$.

To this end, we show that for any triangular \iid array $(X_{n,i})_{n,i}$ such that $(X_{n,1})_{n\geq 1}$ is uniformly integrable, then the two conditions in the weak law of large numbers for triangular arrays of \citep[Thm.\ 2.2.11]{durrett-book:2019} (stated below) are satisfied. We will also show that for such $(X_{n,i})_{n,i}$, $(S_n\defeq \frac{1}{n} \sum_{i=1}^n X_{n,i})_{n\geq 1}$ is uniformly integrable. Together, these results will imply $L^1$ convergence. We will then choose $X_{n,i} = \bar{h}_n'(Z_i)^2 / \sigma_n^2$ to get the desired result in our specific case.

\begin{theorem}[{Weak law for triangular arrays \citep[Thm.\ 2.2.11]{durrett-book:2019}}]
\label{thm-WLLN-durrett}
For each $n$, let $X_{n,i}$, $1\leq i\leq n$, be independent. Let $b_n>0$ with $b_n\to \infty$, and let $\bar{X}_{n,i} = X_{n,i} \indic{|X_{n,i}|\leq b_n}$. Suppose that as $n\to \infty$
    \balignt
    \sum_{i=1}^n \P(|X_{n,i}|>b_n)&\to 0,
    \label{eq:durrett-first-cond} \qtext{and}\\
    b_n^{-2} \sum_{i=1}^n \E[\bar{X}_{n,i}^2]&\to 0.
    \label{eq:durrett-second-cond}
    \ealignt
If we let $S_n = \sum_{i=1}^n X_{n,i}$ and $a_n = \sum_{i=1}^n \E[\bar{X}_{n,i}]$, then $(S_n - a_n)/b_n \toprob 0$.
\end{theorem}

To prove our result, we specify the case of interest $b_n = n$. First, $n \P(|X_{n,1}|>n) \leq \E[|X_{n,1}| \indic{|X_{n,1}|>n}] \leq \sup_{m \geq 1} \E[|X_{m,1}| \indic{|X_{m,1}|>n}] \to 0$ as $n \to \infty$, because $(X_{n,1})_{n \geq 1}$ is uniformly integrable. Thus the first condition \cref{eq:durrett-first-cond} holds.

Note that we then get $\E[X_{n,1} \indic{|X_{n,1}|\leq n}]\to 1$
as $n\to \infty$, for our choice $X_{n,i} = \bar{h}_n'(Z_i)^2 / \sigma_n^2$ which satisfies $\E[X_{n,i}] = 1$.

To verify the second condition \cref{eq:durrett-second-cond}, we will show that $n^{-1} \E[X_{n,1}^2 \indic{X_{n,1} \leq n}] \to 0$. To this end, we need the following lemma, which gives a useful formulation of uniform integrability.
\begin{lemma}[{De la Vall\'ee Poussin Theorem \citep[Thm.\ 22]{meyer-book:1966}}]
If $(X_{n})_{n\geq1}$ is uniformly integrable, then there exists a nonnegative increasing function $G$ such that $G(t)/t \to \infty$ as $t \to \infty$ and $\sup_n \E[G(X_{n})]<\infty$.
\end{lemma}
With such a function $G$, fix any $T$ such that $G(t)/t \geq 1$ for all $t \geq T$, so that $t/G(t)\leq 1$ for all $t \geq T$. Using \citep[Lem. 2.2.13]{durrett-book:2019} for the first equality, we can write
\balignt
\frac{1}{n} \E[X_{n,1}^2 \indic{X_{n,1} \leq n}]
&= \frac{2}{n} \int_0^{\infty} y \P(X_{n,1} \indic{X_{n,1} \leq n} > y) dy \\
&\leq \frac{2}{n} \int_0^n y \P(X_{n,1} > y) dy \\
&= \frac{2}{n} (\int_0^T y \P(X_{n,1} > y) dy  + \int_T^n y \P(X_{n,1} > y) dy) \\
&\leq \frac{T^2}{n} + \frac{2}{n} \int_T^n y \P(X_{n,1} > y) dy) \\
&\leq \frac{T^2}{n} + \E[G(X_{n,1})] \frac{2}{n} \int_T^n y/G(y) dy \\
&= \frac{T^2}{n} + o(1),
\ealignt
where the penultimate line follows from Markov's inequality and the last line comes from the following lemma since $\sup_{y \geq T} y/G(y) \leq 1$ and $y/G(y) \to 0$.

\begin{lemma}
If $f(y) \to 0$ as $y \to \infty$ and $\sup_{y\geq T} |f(y)| \leq M$, then $\frac{1}{n} \int_T^n f(y) dy \to 0$.
\end{lemma}
\begin{proof}
Let $f_n(z) = f(n z) \indic{z > T/n}$, and note that, for any $z\geq 0$, $f_n(z) \to 0$ as $n \to \infty$.
Then
\balignt
\frac{1}{n} \int_{T}^n f(y) dy
&= \int_0^1 \indic{z > T/n} f(n z) dz \\
&= \int_0^1 f_n(z) dz \\
&\to 0
\ealignt
by the bounded convergence theorem.
\end{proof}
Consequently, the second condition \cref{eq:durrett-second-cond} holds.

Moreover, $(S_n\defeq \frac{1}{n} \sum_{i=1}^n X_{n,i})_{n\geq 1}$ is uniformly integrable whenever $(X_{n,i})_{n,i}$ is a triangular \iid array such that $(X_{n,1})_{n\geq 1}$ is uniformly integrable for the following reasons:
\begin{enumerate}
\item $\sup_n \E[|S_n|] \leq \sup_{n} \E[|X_{n,1}|] < \infty$ by triangle inequality and because $(X_{n,1})_{n \geq 1}$ is uniformly integrable.
\item For any $\varepsilon>0$, let $\delta>0$ such that for any event $A$ satisfying $\P(A)\leq \delta$, $\sup_{n} \E[|X_{n,1}| \indic{A}]\leq \varepsilon$. Such $\delta$ exists because $(X_{n,1})_{n \geq 1}$ is uniformly integrable. Then $\sup_n \E[|S_n| \indic{A}]\leq \varepsilon$ by triangle inequality.
\end{enumerate}

The combination of convergence in probability and uniform integrability implies convergence in $L^1$. As a result, $\frac{1}{n}\sum_{i=1}^n \bar{h}_n'(Z_i)^2 / \sigma_n^2 \toL{1} 1$ as long as the sequence of
$(\bar{h}_n(Z_0)-\E[\bar{h}_n(Z_0)])^2/\sigma_n^2 = \bar{h}_n'(Z_0)^2/\sigma_n^2$
is uniformly integrable.

Therefore, $\E[|\siginapprox^2/\sigma_n^2 - 1|] \leq \sqrt{\frac{2}{n(n/k-1)}} + o(1)$, and
we get the result advertised.
\end{proof}

If $k\leq n/2$, which is the case here since $k < n$ and $k$ divides $n$, then $\frac{2}{n(n/k-1)} \to 0$ and $\frac{3k-n}{n(n-k)} \to 0$. Therefore, by \cref{eq:sigin-bound-L1}, we have $(\sigin^2 - \sigma_n^2)/\sigma_n^2 \toL{1} 0$, i.e. $\sigin^2/\sigma_n^2 \toL{1} 1$, whenever the sequence of
$(\bar{h}_n(Z_0)-\E[\bar{h}_n(Z_0)])^2/\sigma_n^2$
is uniformly integrable and $\lstab(h_n) = o(\frac{n-k}{n^2}\sigma_n^2)$, or equivalently $\lstab(h_n) = o(\sigma_n^2 /n)$ since $k\leq n/2$, and, by \cref{eq:sigin-bound-L2}, we have $(\sigin^2 - \sigma_n^2)/\sigma_n^2 \toL{2} 0$, i.e. $\sigin^2/\sigma_n^2 \toL{2} 1$, whenever $\E[\bar{h}_n'(Z_0)^4] = \E[(\bar{h}_n(Z_0)-\E[\bar{h}_n(Z_0)])^4] = o(n \sigma_n^4)$ and $\gamma_{4}(h_n') = o(\frac{(n-k)^2}{n^4}\sigma_n^4)$, or equivalently $\gamma_{4}(h_n') = o(\sigma_n^4 /n^2)$ since $k\leq n/2$.

\cref{consistent-variance-est-in} thus follows from \cref{consistent-variance-est-in-detailed}.

\paragraph{Strengthening of the consistency result of \citep[Prop.~1]{MA-WZ:2020}}
We provide more details about the comparison of our $L^2$-consistency result with \citep[Prop.~1]{MA-WZ:2020}. We have $\gamma_4(h_n') \leq 16 \gamma_4(h_n)$ and $\E[(\bar{h}_n(Z_0)-\E[\bar{h}_n(Z_0)])^4] \leq 16 \E[h_n(Z_0,Z_{1:\ntrain})^4]$ by Jensen's inequality. Moreover, if $\tilde{\sigma}_n^2$ converges to a non-zero constant, since $\lstab(h_n) \leq \msstab(h_n) \leq \sqrt{\gamma_{4}(h_n)}$, then $\lstab(h_n) = o(\sigma_n^2/n)$ whenever $\gamma_{4}(h_n) = o(\sigma_n^4/n^2)$ and thus $\sigma_n^2$ converges to the same non-zero constant as $\tilde{\sigma}_n^2$ does by \cref{var-comp}.
\section{Proof of \cref{consistent-variance-est-out}: Consistent all-pairs estimate of asymptotic variance}
\label{sec:proof-consistent-variance-est-out}

We will prove the following more detailed statement from which \cref{consistent-variance-est-out} will follow.

\begin{theorem}[Consistent all-pairs estimate of asymptotic variance]
\label{consistent-variance-est-out-detailed}
Suppose that $k$ divides $n$ evenly.
Under the notation of \cref{iid-cv-normal} with $\ntrain = n(1-1/k)$, $\bar{h}_n(z) = \E[h_n(z,Z_{1:\ntrain})]$, $h_n'(Z_0,Z_{1:\ntrain}) = h_n(Z_0,Z_{1:\ntrain}) - \E[h_n(Z_0,Z_{1:\ntrain}) \mid Z_{1:\ntrain}]$ and $\bar{h}_n'(z) = \E[h_n'(z,Z_{1:\ntrain})]$, define the all-pairs variance estimate
\balignt
\sigout^2 \defeq
\frac{1}{k}\sum_{j=1}^k
\frac{k}{n}\sum_{i\in B_j'}
\left(
h_n(Z_i,Z_{B_j}) - \Rhat
\right)^2.
\ealignt
If $(Z_i)_{i\geq 1}$ are \iid copies of a random element $Z_0$ and 
$\tilde{\sigma}_n^2 
= \E[h_n'(Z_0, Z_{1:m})^2]
$, then
\balignt
\E[|\sigout^2 - \sigma_n^2|]
&\leq (1+\frac{n}{k})\msstab(h_n) + 2\sqrt{2(1+\frac{n}{k})\msstab(h_n)\tilde{\sigma}_n^2}
+ m\lstab(h_n) \\
&\quad + 2\sqrt{m \lstab(h_n)(1-\frac{1}{n})\sigma_n^2} + \sqrt{\frac{1}{n}(\E[\bar{h}_n'(Z_0)^4]-\sigma_n^4)} + \frac{1}{n} \sigma_n^2.
\ealignt
Moreover,
\balignt
\E[|\sigout^2 - \sigma_n^2|]
&\leq (1+\frac{n}{k})\msstab(h_n) + 2\sqrt{2(1+\frac{n}{k})\msstab(h_n)\tilde{\sigma}_n^2}
+ m\lstab(h_n) \\
&\quad + 2\sqrt{m \lstab(h_n)(1-\frac{1}{n})\sigma_n^2} + \frac{1}{n} \sigma_n^2 + o(\sigma_n^2).
\label{sigout-consistency-bound}
\ealignt
whenever the sequence of
$(\bar{h}_n(Z_0)-\E[\bar{h}_n(Z_0)])^2/\sigma_n^2$
is uniformly integrable.
\end{theorem}

\begin{proof}
\paragraph{A common training set for each validation point pair}
We begin by approximating our variance estimate
\begin{talign}
\sigout^2
    &= \frac{1}{n}\sum_{j=1}^k\sum_{i\in B_j'}
    \left(
        h_n(Z_i,Z_{B_j}) - \Rhat
    \right)^2 \\
    &= \frac{1}{n^2}\sum_{j,j' = 1}^{k}
     \sum_{i\in B_j',i'\in B_{j'}'}
     \half (h_n(Z_i, Z_{B_j}) - h_n(Z_{i'}, Z_{B_{j'}}))^2
\end{talign}
by a quantity that employs the same training set for each pair of validation points $Z_{(i,i')}$,
\balignt
\sigoutapprox^2
&\defeq \frac{1}{n^2}\sum_{j,j' = 1}^{k}
\sum_{i\in B_j',i'\in B_{j'}'}
\half (h_n(Z_i, Z_{B_j}^{\backslash i'}) - h_n(Z_{i'}, Z_{B_j}^{\backslash i'}))^2 \\
&= \frac{1}{n^2}\sum_{j,j' = 1}^{k}
\sum_{i\in B_j',i'\in B_{j'}'}
\half (h_n'(Z_i, Z_{B_j}^{\backslash i'}) - h_n'(Z_{i'}, Z_{B_j}^{\backslash i'}))^2.
\ealignt
Here, for any $j\in [k]$ and $i' \in [n]$,  $Z_{B_j}^{\backslash i'}$ is $Z_{B_j}$ with $Z_{i'}$ replaced by $Z_0$.
By Cauchy--Schwarz,  
we have 
\begin{talign}
|\sigout^2 - \sigoutapprox^2|
    &\leq \Delta_1 + 2\sqrt{\Delta_1}\sigoutapprox
\end{talign}
for the error term
\balignt
\Delta_1 
&\defeq \frac{1}{n^2}\sum_{j,j' = 1}^{k}
 \sum_{i\in B_j',i'\in B_{j'}'}
 \half(h_n(Z_i, Z_{B_j}) - h_n(Z_i, Z_{B_j}^{\backslash i'}) + h_n(Z_{i'}, Z_{B_j}^{\backslash i'}) - h_n(Z_{i'}, Z_{B_{j'}}))^2\\
&\leq \frac{1}{n^2}\sum_{j,j' = 1}^{k}
 \sum_{i\in B_j',i'\in B_{j'}'}
 (h_n(Z_i, Z_{B_j}) - h_n(Z_i, Z_{B_j}^{\backslash i'}))^2 \\
&\quad + \frac{1}{n^2}\sum_{j,j' = 1}^{k}
 \sum_{i\in B_j',i'\in B_{j'}'} (h_n(Z_{i'}, Z_{B_j}^{\backslash i'}) - h_n(Z_{i'}, Z_{B_{j'}}))^2, 
 \label{eq:delta1-out-bound}
\ealignt
where we have used Jensen's inequality in the final display.

\paragraph{Controlling the error $\Delta_1$}
We will first control the error term $\Delta_1$.
Note that, for $B_{j'} \neq B_j$, $|B_{j'} \backslash (B_{j'} \cap B_j)| = \frac{n}{k}$.
Hence, by the bound \cref{eq:delta1-out-bound} and the conditional Efron-Stein inequality (\cref{efron-stein}), we have
\balignt
\E[\Delta_1]
    &\leq \msstab(h_n) + 
    \frac{1}{n^2}\sum_{j,j' = 1}^{k}
     \sum_{i\in B_j',i'\in B_{j'}'} \E[(h_n(Z_{i'}, Z_{B_j}^{\backslash i'}) - h_n(Z_{i'}, Z_{B_{j'}}))^2] \\
     &\leq \msstab(h_n) + 
    \frac{n}{k} \msstab(h_n) = (1+\frac{n}{k})\msstab(h_n).
    \label{eq:mean-delta1-out-bound}
\ealignt

\paragraph{Eliminating training set randomness}
We then approximate $\sigoutapprox^2$ by a quantity eliminating training set randomness in each summand,
\balignt
\sigoutapproxx^2 \defeq \frac{1}{n^2}\sum_{j,j' = 1}^{k} \sum_{i\in B_j',i'\in B_{j'}'}
\half (\bar{h}_n'(Z_i) - \bar{h}_n'(Z_{i'}))^2
\ealignt
where $\bar{h}_n'(z)=\E[h_n'(z,Z_{1:\ntrain})]$. Note that $\bar{h}_n'(Z_0)$ has expectation 0.

By Cauchy--Schwarz,  
we have 
\begin{talign}
|\sigoutapprox^2 - \sigoutapproxx^2|
    &\leq \Delta_2 + 2\sqrt{\Delta_2}\sigoutapproxx
\end{talign}
for the error term
\balignt
\Delta_2
&\defeq \frac{1}{n^2}\sum_{j,j' = 1}^{k}
 \sum_{i\in B_j',i'\in B_{j'}'}
 \half(h_n'(Z_i, Z_{B_j}^{\backslash i'}) - \bar{h}_n'(Z_i) + \bar{h}_n'(Z_{i'}) - h_n'(Z_{i'}, Z_{B_j}^{\backslash i'}))^2\\
&\leq \frac{1}{n^2}\sum_{j,j' = 1}^{k}
 \sum_{i\in B_j',i'\in B_{j'}'}
 (h_n'(Z_i, Z_{B_j}^{\backslash i'}) - \bar{h}_n'(Z_i))^2 \\
&\quad + \frac{1}{n^2}\sum_{j,j' = 1}^{k}
 \sum_{i\in B_j',i'\in B_{j'}'} (\bar{h}_n'(Z_{i'}) - h_n'(Z_{i'}, Z_{B_j}^{\backslash i'}))^2,
 \label{eq:delta2-out-bound}
\ealignt
where we have used Jensen's inequality in the final display.

\paragraph{Controlling the error $\Delta_2$}
We will control the error term $\Delta_2$.
By the bound \cref{eq:delta2-out-bound} and the conditional Efron-Stein inequality (\cref{efron-stein}), we have
\balignt
\E[\Delta_2] \leq 2 \frac{m}{2}\msstab(h'_n) = m \lstab(h_n).
\label{eq:mean-delta2-out-bound}
\ealignt

\paragraph{Controlling the error $\sigoutapproxx^2 - \sigma_n^2$}
To control the error $\sigoutapproxx^2 - \sigma_n^2$, we first rewrite $\sigoutapproxx^2$ as
\balignt
\sigoutapproxx^2
&= \frac{1}{n^2}\sum_{j,j' = 1}^{k} \sum_{i\in B_j',i'\in B_{j'}'}
\half (\bar{h}_n'(Z_i) - \bar{h}_n'(Z_{i'}))^2 \\
&= \frac{1}{n^2}\sum_{i,i'=1}^n
\half (\bar{h}_n'(Z_i) - \bar{h}_n'(Z_{i'}))^2 \\
&= \frac{1}{n}\sum_{i=1}^n \left(\bar{h}_n'(Z_i) - \frac{1}{n}\sum_{i'=1}^n\bar{h}_n'(Z_{i'})\right)^2 \\
&= \frac{1}{n}\sum_{i=1}^n \bar{h}_n'(Z_i)^2 - \left(\frac{1}{n}\sum_{i=1}^n\bar{h}_n'(Z_i)\right)^2.
\label{eq:rewrite-sigoutapproxx}
\ealignt

Since $\E[\bar{h}_n'(Z_i)\bar{h}_n'(Z_{i'})] = 0$ for all $i,i'\in [n]$ with $i\neq i'$ due to independence, we have
\balignt
\E[\left(\frac{1}{n}\sum_{i=1}^n\bar{h}_n'(Z_i)\right)^2]
= \frac{1}{n}\E[\bar{h}_n'(Z_0)^2]
= \frac{1}{n} \sigma_n^2.
\label{eq:var-sum-hbarprime}
\ealignt

Furthermore,
\balignt
\E[(\frac{1}{n}\sum_{i=1}^n \bar{h}_n'(Z_i)^2 - \sigma_n^2)^2]
= \Var(\frac{1}{n}\sum_{i=1}^n \bar{h}_n'(Z_i)^2)
= \frac{1}{n}\Var(\bar{h}_n'(Z_0)^2)
= \frac{1}{n}(\E[\bar{h}_n'(Z_0)^4]-\sigma_n^4)
\ealignt
by independence.
Hence, 
\balignt
\E[|\sigoutapproxx^2 - \sigma_n^2|]
    \leq \sqrt{\frac{1}{n}(\E[\bar{h}_n'(Z_0)^4]-\sigma_n^4)} + \frac{1}{n} \sigma_n^2.
\ealignt

\paragraph{Putting the pieces together}
Since each
\balignt
\half (h_n'(Z_i, Z_{B_j}^{\backslash i'}) - h_n'(Z_{i'}, Z_{B_j}^{\backslash i'}))^2 \leq h_n'(Z_i, Z_{B_j}^{\backslash i'})^2 + h_n'(Z_{i'}, Z_{B_j}^{\backslash i'})^2,
\ealignt
we have
\balignt
\E[\sigoutapprox^2] 
    \leq 2\E[h_n'(Z_0, Z_{1:m})^2]
    = 2\tilde{\sigma}_n^2
\ealignt
and hence
\balignt
\E[\sqrt{\Delta_1}\sigoutapprox]
    \leq \sqrt{\E[\Delta_1]\E[\sigoutapprox^2]}
    \leq \sqrt{2(1+\frac{n}{k})\msstab(h_n)\tilde{\sigma}_n^2}
\ealignt
by Cauchy--Schwarz and the bound \cref{eq:mean-delta1-out-bound}.

Moreover, $\E[\sigoutapproxx^2]=(1-\frac{1}{n})\sigma_n^2$, hence
\balignt
\E[\sqrt{\Delta_2}\sigoutapproxx]
\leq \sqrt{\E[\Delta_2]\E[\sigoutapproxx^2]}
\leq \sqrt{m\lstab(h_n)(1-\frac{1}{n})\sigma_n^2}
\ealignt
by Cauchy--Schwarz and the bound \cref{eq:mean-delta2-out-bound}.

Assembling our results with the triangle inequality, we find that
\balignt
\E[|\sigout^2 - \sigma_n^2|]
&\leq 
\E[|\sigout^2 - \sigoutapprox^2|]
+ \E[|\sigoutapprox^2 - \sigoutapproxx^2|] \\
&\quad + \E[|\sigoutapproxx^2 - \sigma_n^2|] \\
&\leq \E[\Delta_1] + 2\E[\sqrt{\Delta_1}\sigoutapprox] \\
&\quad + \E[\Delta_2] + 2\E[\sqrt{\Delta_2}\sigoutapproxx] \\
&\quad + \sqrt{\frac{1}{n}(\E[\bar{h}_n'(Z_0)^4]-\sigma_n^4)} + \frac{1}{n} \sigma_n^2 \\
&\leq (1+\frac{n}{k})\msstab(h_n) + 2\sqrt{2(1+\frac{n}{k})\msstab(h_n)\tilde{\sigma}_n^2} \\
&\quad + m\lstab(h_n) + 2\sqrt{m \lstab(h_n)(1-\frac{1}{n})\sigma_n^2} \\
&\quad + \sqrt{\frac{1}{n}(\E[\bar{h}_n'(Z_0)^4]-\sigma_n^4)} + \frac{1}{n} \sigma_n^2
\ealignt
as advertised.

We showed in the proof of \cref{consistent-variance-est-in-detailed} that $\frac{1}{n}\sum_{i=1}^n \bar{h}_n'(Z_i)^2 / \sigma_n^2 \toL{1} 1$ whenever the sequence of $(\bar{h}_n(Z_0)-\E[\bar{h}_n(Z_0)])^2/\sigma_n^2 = \bar{h}_n'(Z_0)^2/\sigma_n^2$ is uniformly integrable. Thus, with \cref{eq:rewrite-sigoutapproxx} and \cref{eq:var-sum-hbarprime}, we get $\E[|\sigoutapproxx^2/\sigma_n^2 - 1|] \leq 1/n + o(1)$, and the final bound advertised
\balignt
\E[|\sigout^2 - \sigma_n^2|]
&\leq (1+\frac{n}{k})\msstab(h_n) + 2\sqrt{2(1+\frac{n}{k})\msstab(h_n)\tilde{\sigma}_n^2} \\
&\quad + m\lstab(h_n) + 2\sqrt{m \lstab(h_n)(1-\frac{1}{n})\sigma_n^2} \\
&\quad + \frac{1}{n} \sigma_n^2 + o(\sigma_n^2).
\ealignt
\end{proof}

By the bound \cref{sigout-consistency-bound}, $(\sigout^2 - \sigma_n^2)/\sigma_n^2 \toL{1} 0$, i.e. $\sigout^2/\sigma_n^2 \toL{1} 1$, if the sequence of
$(\bar{h}_n(Z_0)-\E[\bar{h}_n(Z_0)])^2/\sigma_n^2$
is uniformly integrable, $\lstab(h_n) = o(\sigma_n^2/n)$ and $\msstab(h_n) = o(\min(\frac{k\sigma_n^2}{n}, \frac{k\,\sigma_n^4}{n\,\tilde{\sigma}_n^2}))$.
By noticing that $\tilde{\sigma}_n^2/\sigma_n^2 \to 1$ when $\lstab(h_n) = o(\sigma_n^2/n)$ thanks to \cref{var-comp}, the last condition becomes $\msstab(h_n) = o(k\sigma_n^2/n)$.
Therefore, \cref{consistent-variance-est-out} follows from \cref{consistent-variance-est-out-detailed}.

\section{Experimental Setup Details}\label{sec:num-exp-appendix}

Here, we provide more details about the experimental setup of \cref{sec:experiments}.

\subsection{General experimental setup details}
\label{sec:gen-exper-setup}

\paragraph{Learning algorithms and hyperparameters}
To illustrate the performance of our confidence intervals and tests in practice, we carry out our experiments with a diverse collection of popular learning algorithms.
For classification, we use the \texttt{xgboost} \texttt{XGBRFClassifier} with \texttt{n\_estimators=100}, \texttt{subsample=0.5} and \texttt{max\_depth=1}, the \texttt{scikit-learn} \texttt{MLPClassifier} neural network with \texttt{hidden\_layer\_sizes=(8,4,)} defining the architecture and \texttt{alpha=1e2}, and the \texttt{scikit-learn} $\ell^2$-penalized \texttt{LogisticRegression} with \texttt{solver=\textquotesingle lbfgs\textquotesingle} and \texttt{C=1e-3}. For regression, we use the \texttt{xgboost} \texttt{XGBRFRegressor} with \texttt{n\_estimators=100}, \texttt{subsample=0.5} and \texttt{max\_depth=1}, the \texttt{scikit-learn} \texttt{MLPRegressor} neural network with \texttt{hidden\_layer\_sizes=(8,4,)} defining the architecture and \texttt{alpha=1e2}, and the \texttt{scikit-learn} \texttt{Ridge} regressor with \texttt{alpha=1e6}.
The random forest \texttt{max\_depth} hyperparameter and neural network, logistic, and ridge $\ell^2$ regularization strengths were selected to ensure the stability of each algorithm.
All remaining hyperparameters are set to their defaults, and we set random seeds for all algorithms' random states for reproducibility.
We use \texttt{scikit-learn} \citep{scikit-learn} version 0.22.1 and \texttt{xgboost} \citep{xgboost} version 1.0.2.

\paragraph{Training set sample sizes $n$} For both datasets, we work with the following training set sample sizes $n$: 700, 1,000, 1,500, 2,300, 3,400, 5,000, 7,500, 11,000. Up to some rounding, this corresponds to a geometric sequence with growth rate 50$\%$.

\paragraph{Details on the \texttt{Higgs} dataset}
The target variable has value either 0 or 1 and there are 28 features. We initially shuffle the rows of the dataset uniformly at random and then, starting at the 5,000,001-th instance, we take 500 consecutive chunks of the largest sample size, that is 11,000. For each $n$, we take the first $n$ instances of these 500 chunks to play the role of our 500 independent replications of size $n$.
The features are standardized during training in the following way: for each iteration of $k$-fold CV ($k=10$ here), we rescale the validation fold and the remaining folds, used as training, with the mean and standard deviation of the training data. The features for the training folds then have mean 0 and variance 1.

\paragraph{Details on the \texttt{FlightsDelay} dataset}
To avoid the temporal dependence issues inherent to time series datasets, we treat the complete \texttt{FlightsDelay} dataset as the population and thus process it differently from the \texttt{Higgs} dataset.
For this dataset, we predict the signed log transform ($y \mapsto \sign(y) \log(1+|y|)$; this addressed the very heavy tails of $y$ on its original scale) of the delay at arrival using 4 features: the scheduled time of the journey from the origin airport to the destination airport (taxi included), the distance between the two airports, the scheduled time of departure in minutes (converted from a time to a number between 0 and 1,439) and the airline operating the plane (that we one-hot encode). We drop the instances that have missing values for at least one of these variables.
Then, we perform 500 times the sampling with replacement of 11,000 points, that is the largest sample size. For each $n$, we take the first $n$ instances of these 500 chunks to play the role of our 500 independent replications of size $n$.
The features are standardized during training in the same way we do for the \texttt{Higgs} dataset.

\paragraph{Computing target test errors}
For the \texttt{FlightDelays} experiments, since training datapoints are sampled with replacement, the population distribution is the entirety of the \texttt{FlightDelays} dataset, and we use this exact population distribution to compute all test errors.
For the \texttt{Higgs} experiments, we form a surrogate ground-truth estimate of the target test errors using the first 5,000,000 datapoints of the shuffled \texttt{Higgs} dataset. As an illustration, for our method where the target test error is the $k$-fold test error $\Rcondcv = \frac{1}{n}\sum_{j=1}^{10}\sum_{i\in B_j'}\E[h_n(Z_i, Z_{B_j})\mid Z_{B_j}] = \frac{1}{k}\sum_{j=1}^{10}\E[h_n(Z_0, Z_{B_j})\mid Z_{B_j}]$, we use these instances to compute the $k$ conditional expectations by a Monte Carlo approximation. Practically, for each training set $Z_{B}$, we compute the average loss on these instances of the fitted prediction rule learned on $Z_{B}$.
Then, we evaluate the CIs and tests constructed from the 500 training sets of varying sizes $n$ sampled from the datasets.

\paragraph{Random seeds}
Seeds are set in the code to ensure reproducibility. They are used for the initial random shuffling of the datasets, the sampling with replacement for the regression dataset, the random partitioning of samples in each replication, and the randomized algorithms.

\subsection{List of procedures}
\label{sec:proc-list}

In our numerical experiments, we compare our procedures with the most popular alternatives from the literature. For each procedure, we give its target test error $\Rcondcv$, the estimator $\Rhat$ of this target, the variance estimator $\hat{\sigma}_n^2$, the two-sided CI used in \cref{sec:ci-experiment}, and the one-sided test used in \cref{sec:sim:test}.

In the following, $q_{\alpha}$ is the $\alpha$-quantile of a standard normal distribution and $t_{\nu,\alpha}$ is the $\alpha$-quantile of a $t$ distribution with $\nu$ degrees of freedom.

\begin{enumerate}[leftmargin=.5cm]

    \item Our 10-fold CV CLT-based test, with $\hat{\sigma}_n$ being either $\sigin$ (\cref{consistent-variance-est-in}) or $\sigout$ (\cref{consistent-variance-est-out}). The curve with $\sigin$ is not displayed in our plots since the results are almost identical to those for $\sigout$ and the curves are overlapping.
    
    \begin{itemize}
    \item Target test error: $\Rcondcv
    = 
    \frac{1}{10}
    \sum_{j=1}^{10}
    \E[h_n(Z_0, Z_{B_j})\mid Z_{B_j}]$.
    
    \item Estimator: $\Rhat
    = 
    \frac{1}{n}
    \sum_{j=1}^{10}
    \sum_{i\in B_j'}
    h_n(Z_i, Z_{B_j})$.
    
    \item Variance estimator:
    $\hat{\sigma}_n^2$, either $\sigin^2$ or $\sigout^2$.
    
    \item Two-sided $(1-\alpha)$-CI:
    $\Rhat \pm q_{1-\alpha/2}\hat{\sigma}_n/\sqrt{n}$.
    
    \item One-sided test: $\textsc{reject } H_0 \iff
    \Rhat < q_{\alpha} \hat{\sigma}_n/\sqrt{n}$.
    \end{itemize}
    
    \item Hold-out test described, for instance, in \citet[Eq. (17)]{MA-WZ:2020}.
    
    \begin{itemize}
    \item Target test error: $\Rcondcv = \E[h_n(Z_0, Z_{S})\mid Z_{S}]$, where $S$ is a subset of size $\lfloor n(1-1/10)\rfloor $ of $[n]$. Since we already have a partition for our 10-fold CV, we can use the first fold $B_1$ for $S$.
    
    \item Estimator: $\Rhat = \frac{1}{|S^c|}\sum_{i\in S^c} h_n(Z_i,Z_{S})$.
    
    \item Variance estimator:
    $\hat{\sigma}_n^2 = \frac{1}{|S^c|}\sum_{i\in S^c}
    (
    h_n(Z_i,Z_{S}) - \Rhat
    )^2$.
    
    \item Two-sided $(1-\alpha)$-CI:
    $\Rhat \pm q_{1-\alpha/2}\hat{\sigma}_n\sqrt{10}/\sqrt{n}$.
    
    \item One-sided test: $\textsc{reject } H_0 \iff
    \Rhat < q_{\alpha} \hat{\sigma}_n \sqrt{10}/\sqrt{n}$.
    \end{itemize}
    
    \item Cross-validated $t$-test of \citet{dietterich:1998}, 10 folds.
    
    \begin{itemize}
    \item Target test error: $\Rcondcv
    = 
    \frac{1}{10}
    \sum_{j=1}^{10}
    \E[h_n(Z_0, Z_{B_j})\mid Z_{B_j}]$.
    
    \item Estimator: $\Rhat
    = 
    \frac{1}{n}
    \sum_{j=1}^{10}
    \sum_{i\in B_j'}
    h_n(Z_i, Z_{B_j})$.
    
    \item Variance estimator:
    $\hat{\sigma}_n^2 =
    \frac{1}{10-1}\sum_{j=1}^{10} (p_{j} - \Rhat)^2$, where $p_{j}\defeq \frac{1}{|B_j'|}\sum_{i\in B_j'} h_n(Z_i,Z_{B_j})$.
    
    \item Two-sided $(1-\alpha)$-CI:
    $\Rhat \pm t_{10-1,1-\alpha/2}\hat{\sigma}_n/\sqrt{10}$.
    
    \item One-sided test: $\textsc{reject } H_0 \iff
    \Rhat < t_{10-1,\alpha} \hat{\sigma}_n /\sqrt{10}$.
    \end{itemize}
    
    \item Repeated train-validation $t$-test of \citet{nadeau-bengio:2003}, 10 repetitions of 90-10 train-validation splits.
    
    \begin{itemize}
    \item Target test error: $\Rcondcv = \frac{1}{10}\sum_{j=1}^{10} \E[h_n(Z_0,Z_{S_j})\mid Z_{S_j}]$, where for any $j\in [10]$, $S_j$ is a subset of size $\lfloor n(1-1/10) \rfloor$ of $[n]$, and these 10 subsets are chosen independently.
    
    \item Estimator: $\Rhat =
    \frac{1}{10}\sum_{j=1}^{10} p_j$, where $p_j\defeq \frac{1}{|S_j^c|} \sum_{i\in S_j^c} h_n(Z_i,Z_{S_j})$.
    
    \item Variance estimator:
    $\hat{\sigma}_n^2 =
    \frac{1}{10-1}\sum_{j=1}^{10} (p_j - \Rhat)^2$.
    
    \item Two-sided $(1-\alpha)$-CI:
    $\Rhat \pm t_{10-1,1-\alpha/2}\hat{\sigma}_n/\sqrt{10}$.
    
    \item One-sided test: $\textsc{reject } H_0 \iff
    \Rhat < t_{10-1,\alpha} \hat{\sigma}_n /\sqrt{10}$.
    \end{itemize}
    
    \item Corrected repeated train-validation $t$-test of \citet{nadeau-bengio:2003},
    10 repetitions of 90-10 train-validation splits.

    \begin{itemize}
    \item Target test error: $\Rcondcv = \frac{1}{10}\sum_{j=1}^{10} \E[h_n(Z_0,Z_{S_j})\mid Z_{S_j}]$,
    where for any $j\in [10]$, $S_j$ is the same as in the previous procedure.
    
    \item Estimator: $\Rhat =
    \frac{1}{10}\sum_{j=1}^{10} p_j$, where $p_j$ is the same as in the previous procedure.
    
    \item Variance estimator:
    $\hat{\sigma}_n^2 =
    (\frac{1}{10}+\frac{0.1}{1-0.1})\frac{10}{10-1}\sum_{j=1}^{10} (p_j - \Rhat)^2$.
    
    \item Two-sided $(1-\alpha)$-CI:
    $\Rhat \pm t_{10-1,1-\alpha/2}\hat{\sigma}_n/\sqrt{10}$.
    
    \item One-sided test: $\textsc{reject } H_0 \iff
    \Rhat < t_{10-1,\alpha} \hat{\sigma}_n /\sqrt{10}$.
    \end{itemize}
    
    \item $5\times 2$-fold CV test of \citet{dietterich:1998}.
    
    \begin{itemize}
    \item Target test error:
    $\Rcondcv = \frac{1}{5}\sum_{j=1}^5 \frac{1}{2} (\E[h_n(Z_0, Z_{B_{1,j}})\mid Z_{B_{1,j}}] + \E[h_n(Z_0, Z_{B_{2,j}})\mid Z_{B_{2,j}}])$, where for any $j\in [5]$, $\{B_{1,j}^c\,,\,B_{2,j}^c\}$ is a partition of $[n]$ into 2 folds of size $n/2$, and these 5 partitions are chosen independently.
    
    \item Estimator: $\Rhat =
    \frac{1}{|B_{1,1}^c|} \sum_{i\in B_{1,1}^c} h_n(Z_i,Z_{B_{1,1}})$.
    
    \item Variance estimator:
    $\hat{\sigma}_n^2 =
    \frac{1}{5}\sum_{j=1}^5 s_j^2$,
    where $s_j^2\defeq (p_{1,j}-\bar{p}_j)^2 + (p_{2,j}-\bar{p}_j)^2$ with $\bar{p}_j\defeq (p_{1,j}+p_{2,j})/2$ and
    $p_{k,j}\defeq
    \frac{1}{|B_{k,j}^c|}
    \sum_{i\in B_{k,j}^c} h_n(Z_i,Z_{B_{k,j}})$ for $k\in [2], j\in [5]$.
    
    \item Two-sided $(1-\alpha)$-CI:
    $\Rhat \pm t_{5,1-\alpha/2}\hat{\sigma}_n$.
    
    \item One-sided test: $\textsc{reject } H_0 \iff
    \Rhat < t_{5,\alpha} \hat{\sigma}_n$.
    \end{itemize}
    
\end{enumerate}

\subsection{Concentration-based confidence intervals}\label{sec:concentration-inequalities}
For comparison in \cref{sec:intro}, we also implemented the ridge regression CI from \citep[Thm. 3]{celisse-guedj:2016} for the \texttt{FlightDelays} experiment (an implementable CI is not provided for any other learning algorithm in \citep{celisse-guedj:2016}). This CI takes as input a uniform bound $B_Y$ on the absolute value of the target variable $Y$ and a uniform bound $B_X$ on the $\ell^2$ norm of the feature vector $X$.
After mean-centering, we find the maximum absolute value of $Y$ across the \texttt{FlightDelays} dataset to be $B_Y = 8.03$.
After mean-centering, we find the maximum $\ell^2$ norm of a feature vector $X$ across the \texttt{FlightDelays} to be $B_X = 13.17$ if each feature is normalized to have standard deviation $1$ or $B_X = 4200$ if the features are left unnormalized. 
When normalizing features as in \cref{fig:test-error-CI-reg}, the smallest width produced by \citep[Thm. 3]{celisse-guedj:2016} for any value of $n$ is 90.2; that is 91 times larger than the largest width of our CLT intervals (equal to 0.99). When not normalizing as in \cref{fig:test-error-CI-RR-LOOCV}, our maximum width is 0.98, but the minimum \citep[Thm. 3]{celisse-guedj:2016} width is $5\times 10^{14}$.

\subsection{Leave-one-out cross-validation}\label{sec:appendix-LOOCV}
To evaluate the LOOCV CLT-based CIs discussed in \cref{sec:loocv} we follow the ridge regression experimental setup of \cref{sec:gen-exper-setup} except that we regress onto the raw feature values instead of the standardized features values described in \cref{sec:gen-exper-setup}.
For our LOOCV CLT-based CIs, the quantities of interest are the following.

\begin{itemize}
    \item Target test error: $\Rcondcv
    = 
    \frac{1}{n}
    \sum_{i=1}^n
    \E[h_n(Z_0, Z_{\{i\}^c})\mid Z_{\{i\}^c}]$.
    \item Estimator: $\Rhat
    =
    \frac{1}{n}
    \sum_{i=1}^n
    h_n(Z_i, Z_{\{i\}^c})$ computed efficiently using the Sherman--Morrison--Woodbury derivation below. %
    \item Variance estimator:
    $\sigout^2$ with $k=n$ folds.
    \item Two-sided $(1-\alpha)$-CI:
    $\Rhat \pm q_{1-\alpha/2}\sigout/\sqrt{n}$.
    \end{itemize}

\paragraph{Results} We construct $95\%$ CIs for ridge regression test error based on our LOOCV CLT and compare their coverage and width with those of the procedures described in \cref{sec:ci-experiment}. 
We see that, like the 10-fold CV CLT intervals, the LOOCV intervals provide coverage near the nominal level and widths smaller than the popular alternatives from the literature; in fact, the 10-fold CV CLT curves are obscured by the nearly identical LOOCV CLT curves.
\begin{figure}[h!]
\centering
    \begin{subfigure}{\subfigfracin\linewidth}
        \includegraphics[width=\imgfrac\linewidth]{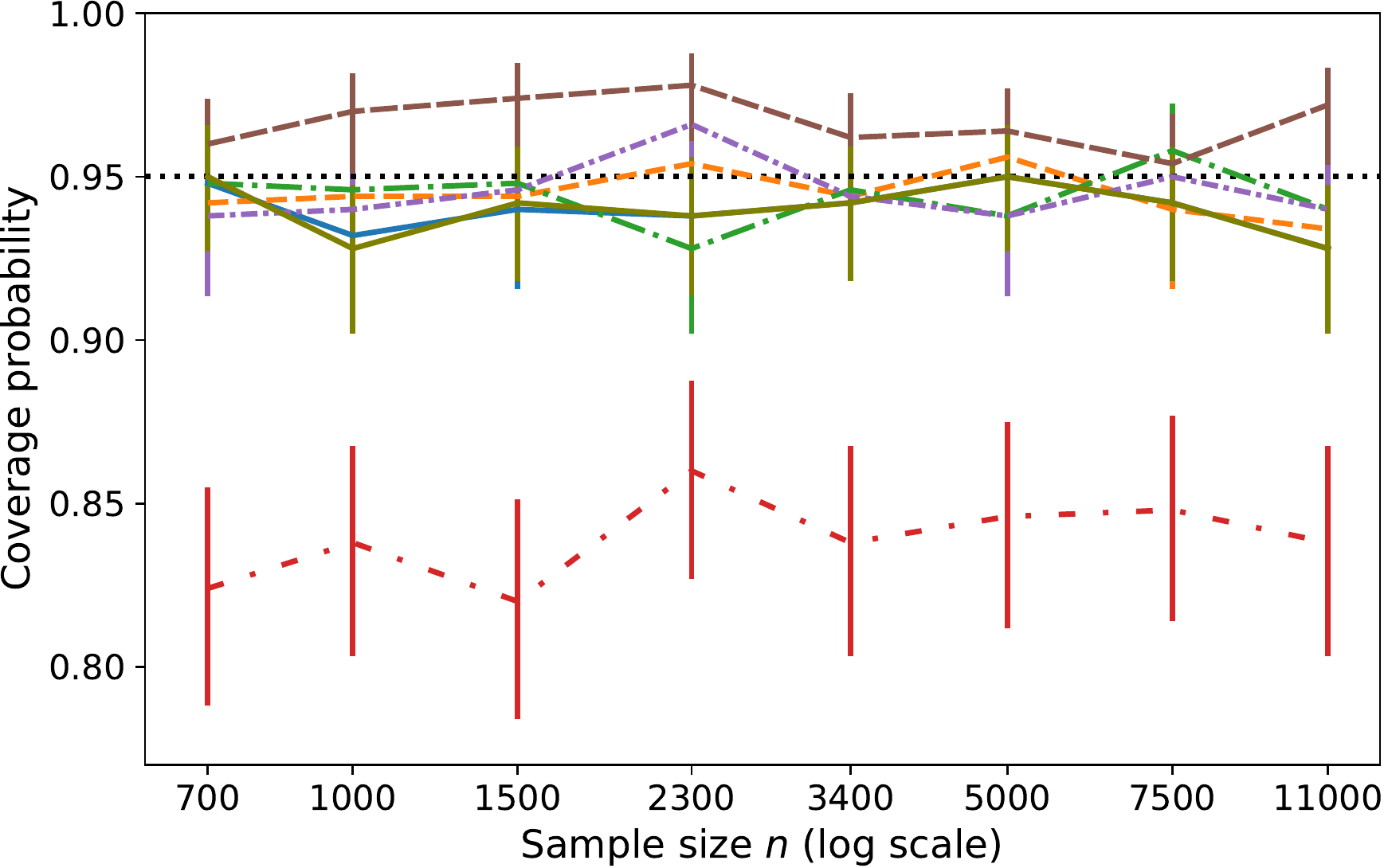}
    \end{subfigure}\hspace{\imgspace\linewidth}%
     \begin{subfigure}{\subfigfracin\linewidth}
             \includegraphics[width=\imgfrac\linewidth]{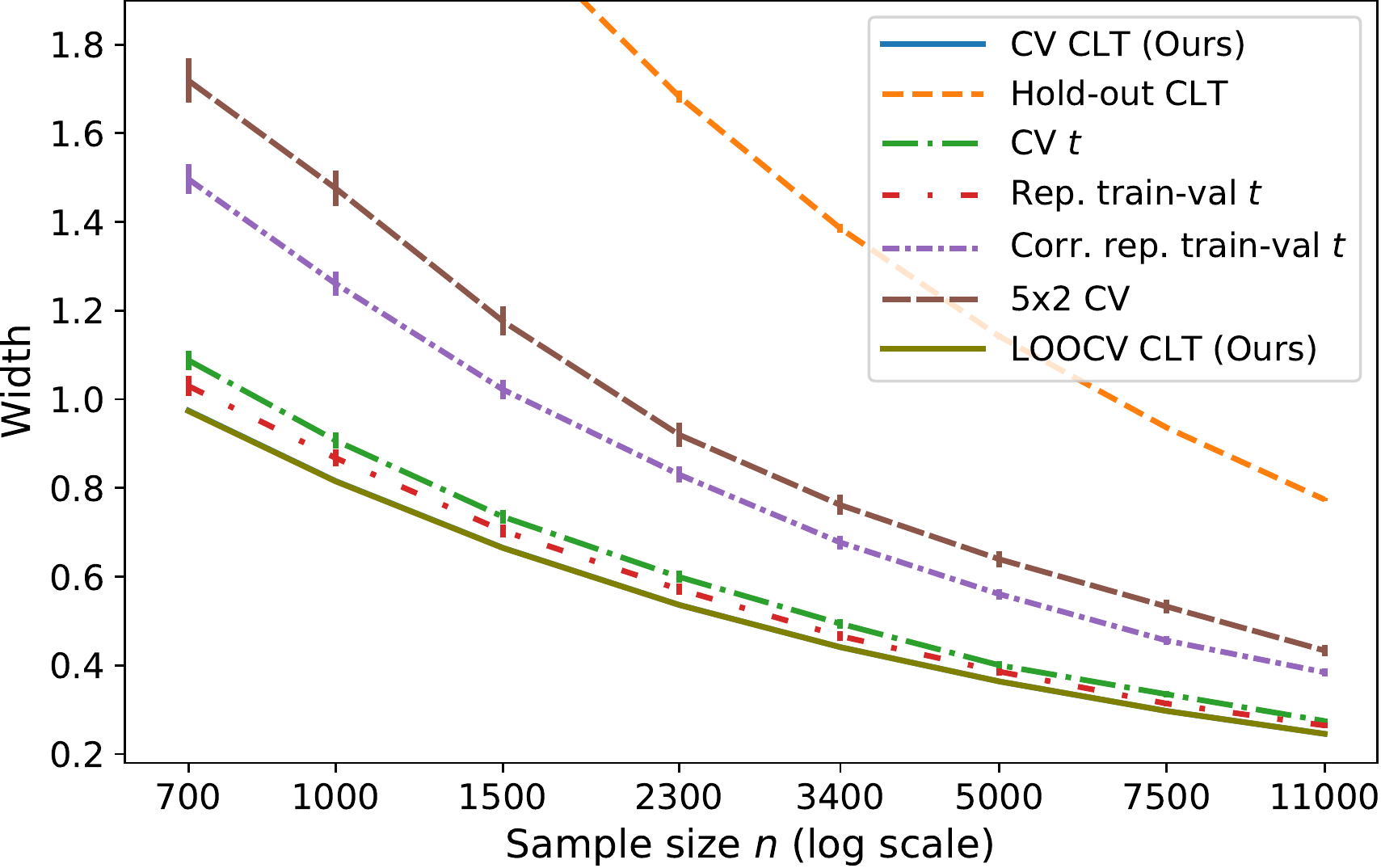}
    \end{subfigure}

    \caption{Test error coverage (left) and width (right) of $95\%$ confidence intervals for ridge regression, including leave-one-out CV intervals (see \cref{sec:loocv}).
    The CV CLT curves are obscured by the nearly identical LOOCV CLT curves.}
    \label{fig:test-error-CI-RR-LOOCV}
\end{figure}

\paragraph{Efficient computation} We explain here how the Sherman--Morrison--Woodbury formula can be used to efficiently compute the individual losses $h_n(Z_i, Z_{\{i\}^c})$, and therefore $\Rhat$ as well as $\sigout$, and the loss on the instances used to form a surrogate ground-truth estimate of the target error $\Rcondcv$.
Let $X\in \R^{n\times p}$ be the matrix of predictors, whose $i$-th row is $x_i^\top$, and $Y\in \R^n$ be the target variable. The weight vector estimate $\hat{w}$ minimizes
$\min_{w\in \R^p} \| Y-Xw \|_2^2 + \lambda \| w \|_2^2$, and is given by the closed-form formula
\begin{talign}
\hat{w} = (X^\top X + \lambda I_p)^{-1} X^\top Y.
\end{talign}
We precompute $M\defeq (X^\top X + \lambda I_p)^{-1}$ and $v\defeq X^\top Y$, that satisfy $\hat{w} = Mv$. Suppose that we have an additional set with covariate matrix $\tilde X$ and target variable $\tilde Y$, representing the instances used to form a surrogate ground-truth estimate of $\Rcondcv$. We also precompute $q\defeq \tilde X \hat{w}$ and $A\defeq \tilde X M$.

For the datapoint $i$, let $X^{(-i)}$ denote $X$ without its $i$-th row and $Y^{(-i)}$ denote $Y$ without its $i$-th element.
Let $M_i\defeq ({X^{(-i)}}^\top X^{(-i)} + \lambda I_p)^{-1}$, $v_i\defeq {X^{(-i)}}^\top Y^{(-i)}$ and $w_i\defeq M_i v_i$.
We can efficiently compute $M_i$ from $M$ based on the Sherman--Morrison--Woodbury formula.
\begin{talign}
M_i
&= ({X^{(-i)}}^\top X^{(-i)} + \lambda I_p)^{-1} \\
&= (X^\top X - x_i x_i^\top + \lambda I_p)^{-1} \\
&= M - M x_i x_i^\top M / (-1 + h_i),
\end{talign}
where $h_i\defeq x_i^\top M x_i$.

We can compute $v_i$ from $v$, with $v_i
= {X^{(-i)}}^\top Y^{(-i)}
= v - x_i y_i$.

Therefore, $w_i
= M_i v_i
= (M - M x_i x_i^\top M (-1 + h_i)^{-1}) (v - x_i y_i)
= \hat{w} + M x_i (\langle \hat{w},x_i \rangle - y_i) / (1-h_i)$ can be computed without fitting any additional prediction rule. Then $h_n(Z_i, Z_{\{i\}^c}) = (y_i - \langle w_i,x_i \rangle)^2$, and we use them to compute $\Rhat$ and $\sigout$.
To make predictions for the covariate matrix $\tilde X$, we efficiently compute $\tilde X w_i$ as
\begin{talign}
\tilde X w_i
&= \tilde X \hat{w} + \tilde X M x_i (\langle \hat{w},x_i \rangle - y_i) / (1-h_i) \\
&= q + A x_i (\langle \hat{w},x_i \rangle - y_i) / (1-h_i),
\end{talign}
and $\frac{1}{N} \|\tilde Y - \tilde X w_i\|_2^2$ is an estimate of $\E[h_n(Z_0, Z_{\{i\}^c})\mid Z_{\{i\}^c}]$, where $N$ is the size of the whole dataset.
An estimate of $\Rcondcv$ is then $\frac{1}{n}\sum_{i=1}^n \frac{1}{N} \|\tilde Y - \tilde X w_i\|_2^2$.

\section{Additional Experimental Results}\label{sec:additional-results}

This section reports the additional results of the experiments described in \cref{sec:experiments}.

\subsection{Additional results from  \cref{sec:ci-experiment}: Confidence intervals for test error}\label{sec:additional-results-CI}

The remaining results of the experiments described in \cref{sec:ci-experiment} are provided in \cref{fig:test-error-CI-clf,fig:test-error-CI-reg}.
We remind that each mean width estimate is displayed with a $\pm$ 2 standard error confidence band, while the confidence band surrounding each coverage estimate is a 95\% Wilson interval.
For all $6$ learning tasks, all procedures except the repeated train-validation $t$ interval provide near-nominal coverage, and our CV CLT intervals provide the smallest widths.

\begin{figure}[h!]
\centering
    \begin{subfigure}{\subfigfracin\linewidth}
        \includegraphics[width=\imgfrac\linewidth]{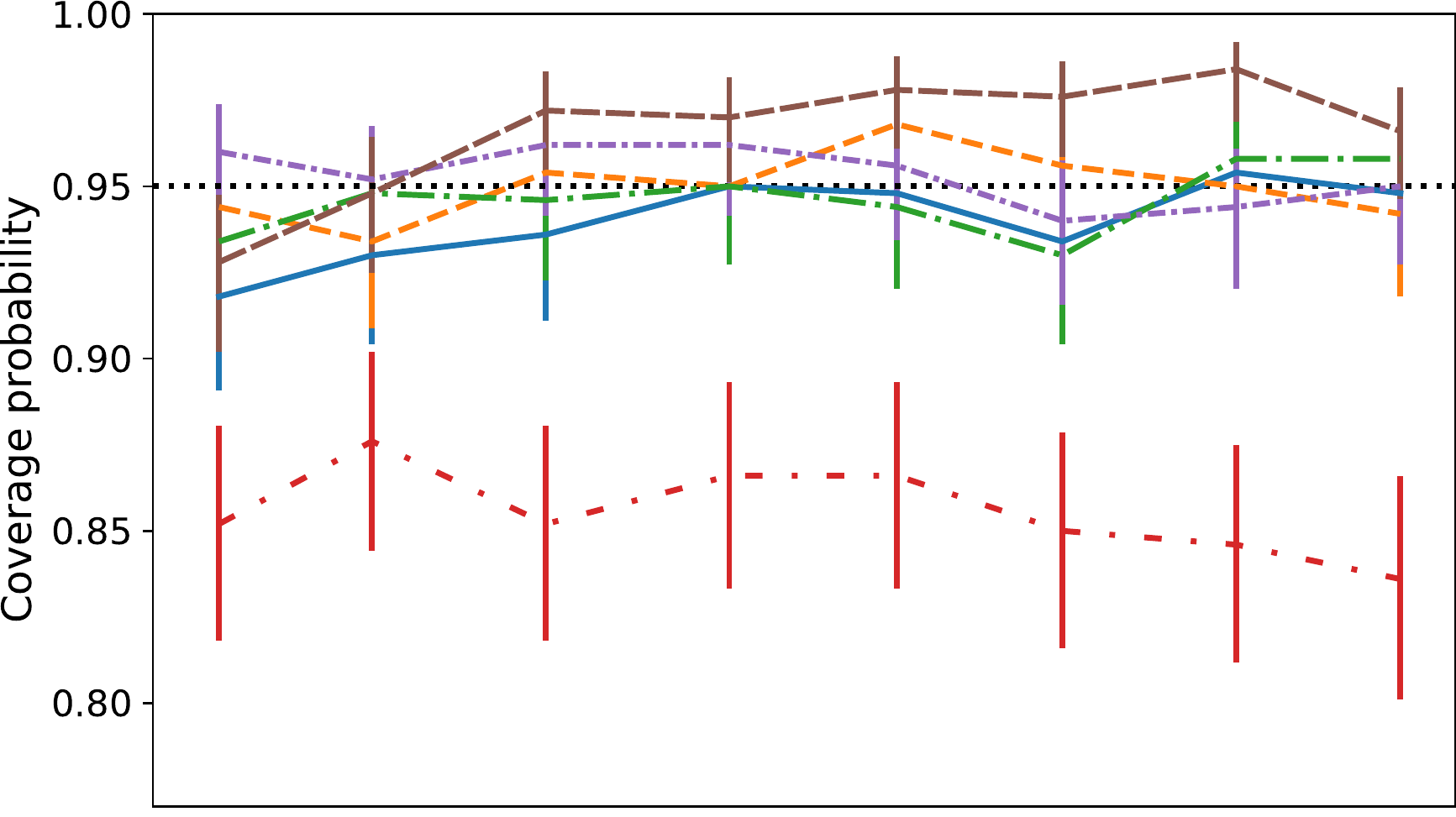}
    \end{subfigure}\hspace{\imgspace\linewidth}%
     \begin{subfigure}{\subfigfracin\linewidth}
             \includegraphics[width=\imgfrac\linewidth]{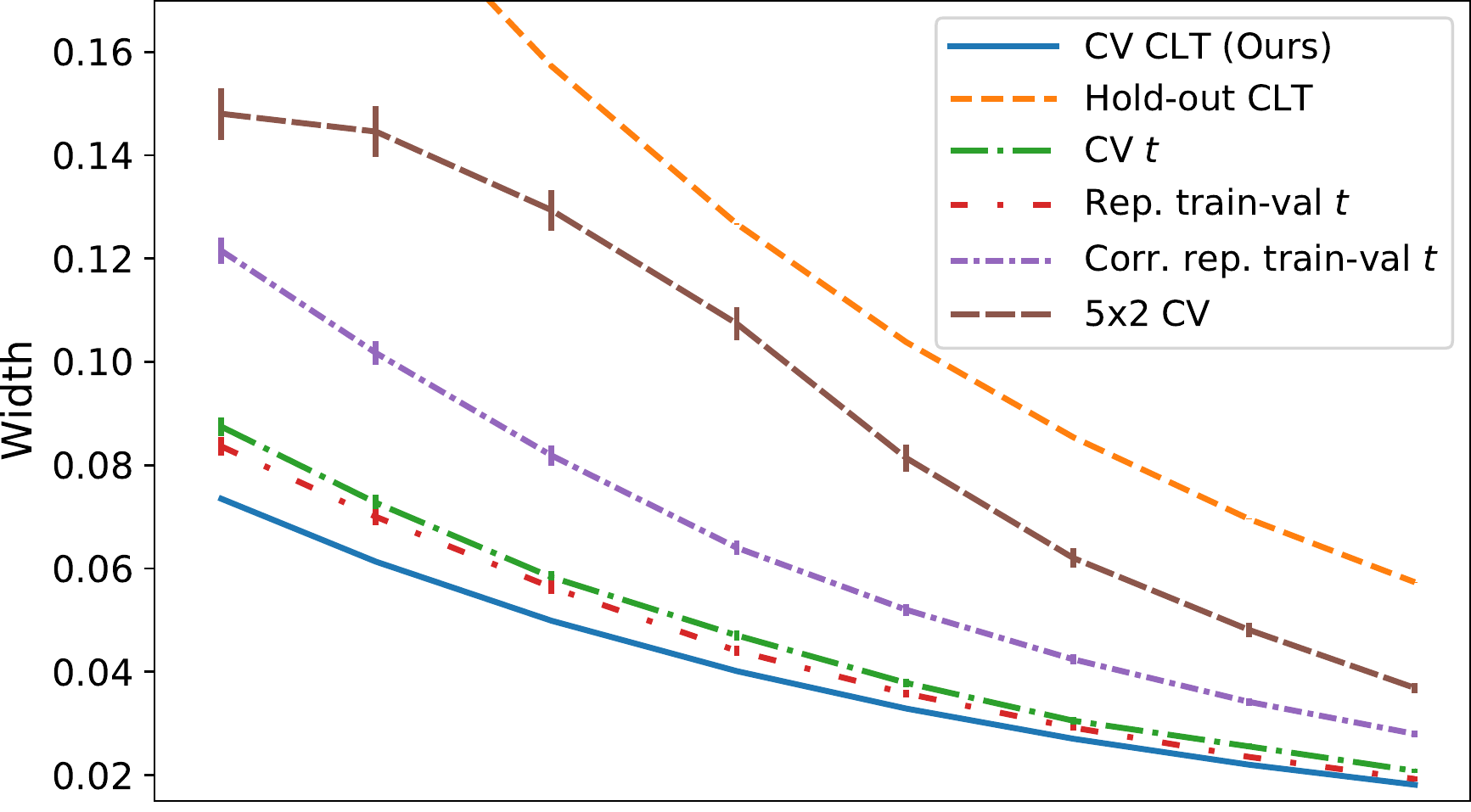}
    \end{subfigure}

    \vspace{\vgap\linewidth}
    
    \begin{subfigure}{\subfigfracin\linewidth}
        \includegraphics[width=\imgfrac\linewidth]{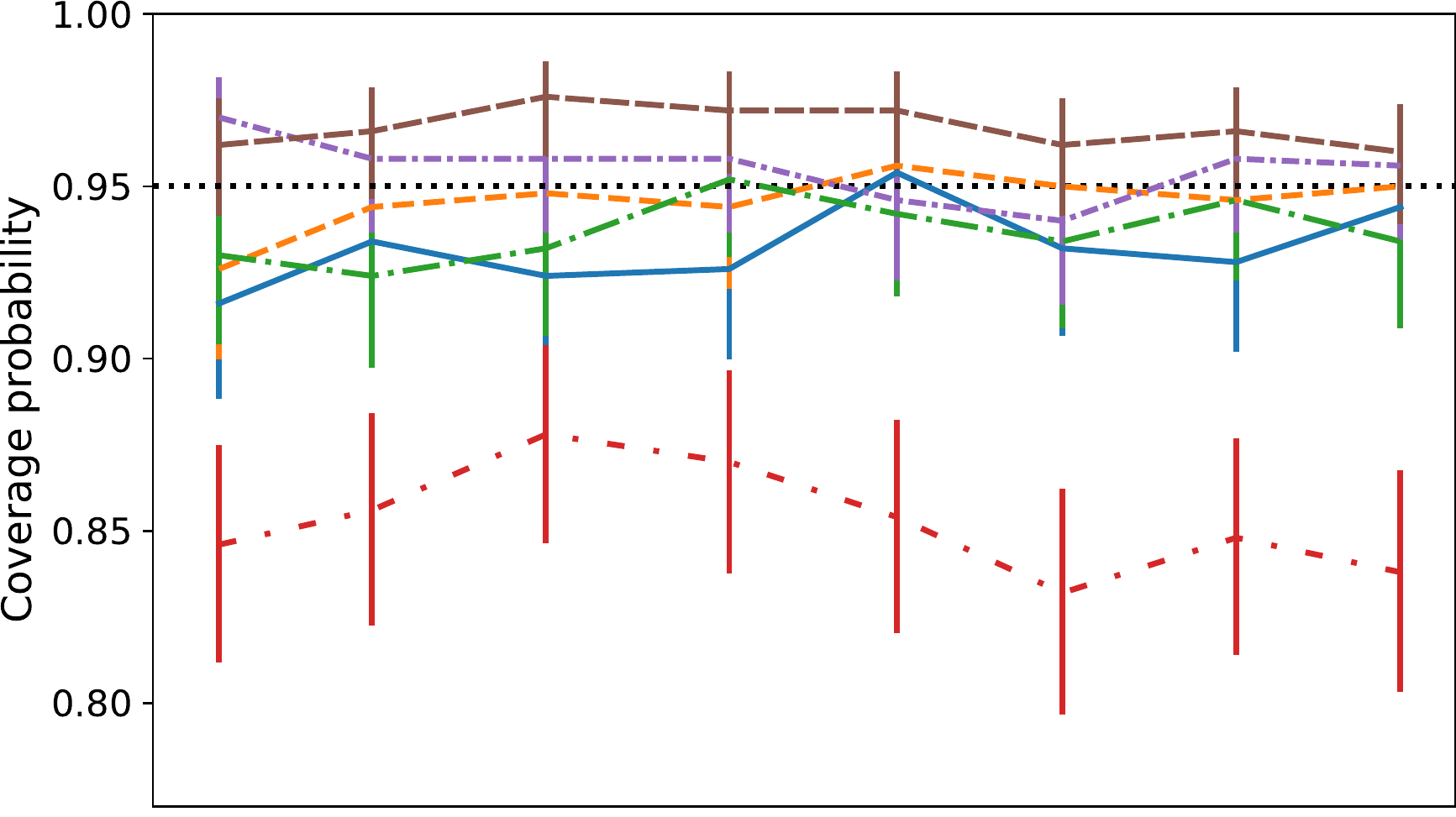}
    \end{subfigure}\hspace{\imgspace\linewidth}%
    \begin{subfigure}{\subfigfracin\linewidth}
        \includegraphics[width=\imgfrac\linewidth]{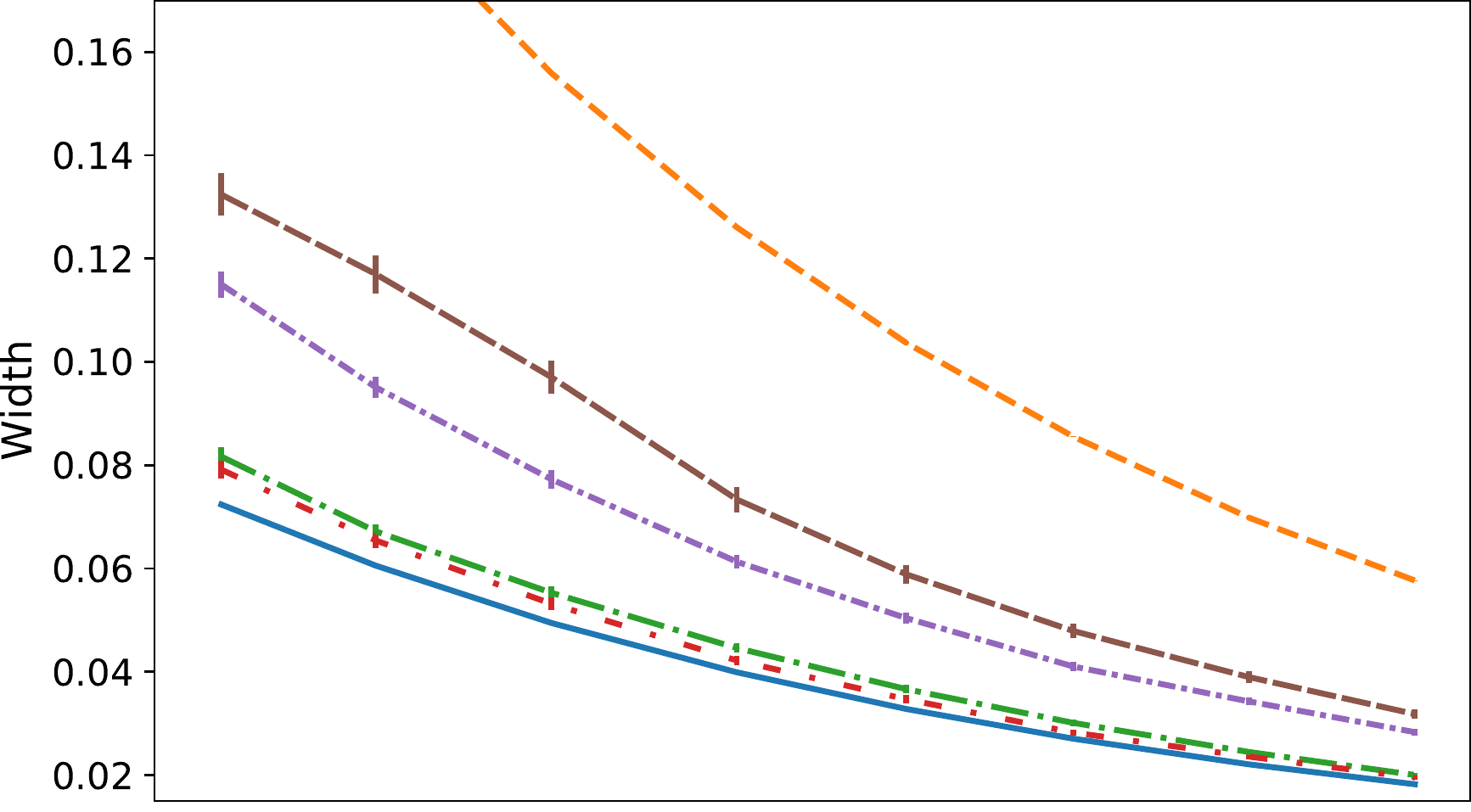}
    \end{subfigure}

    \vspace{\vgap\linewidth}
    
    \begin{subfigure}{\subfigfracin\linewidth}
        \includegraphics[width=\imgfrac\linewidth]{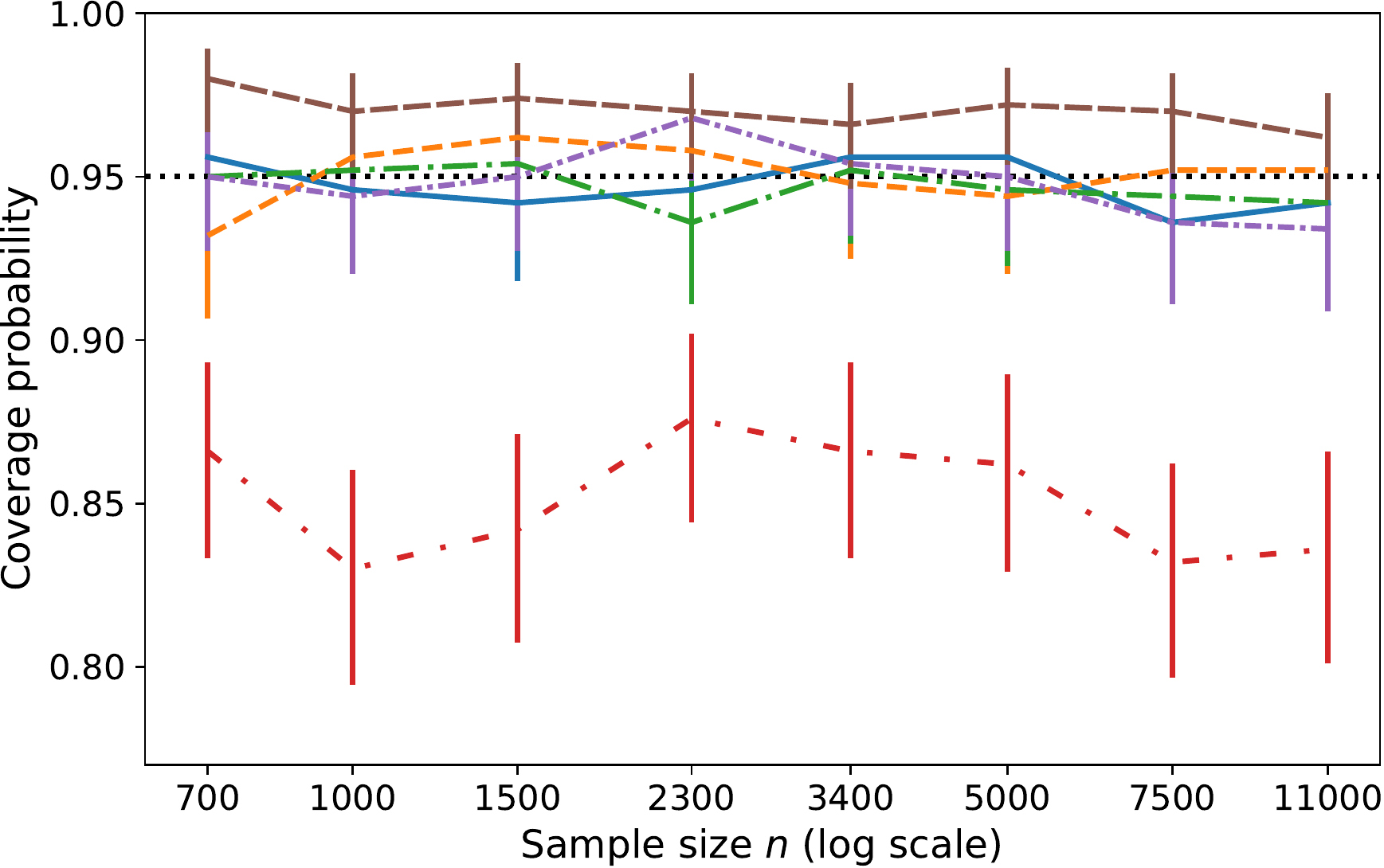}
    \end{subfigure}\hspace{\imgspace\linewidth}%
    \begin{subfigure}{\subfigfracin\linewidth}
        \includegraphics[width=\imgfrac\linewidth]{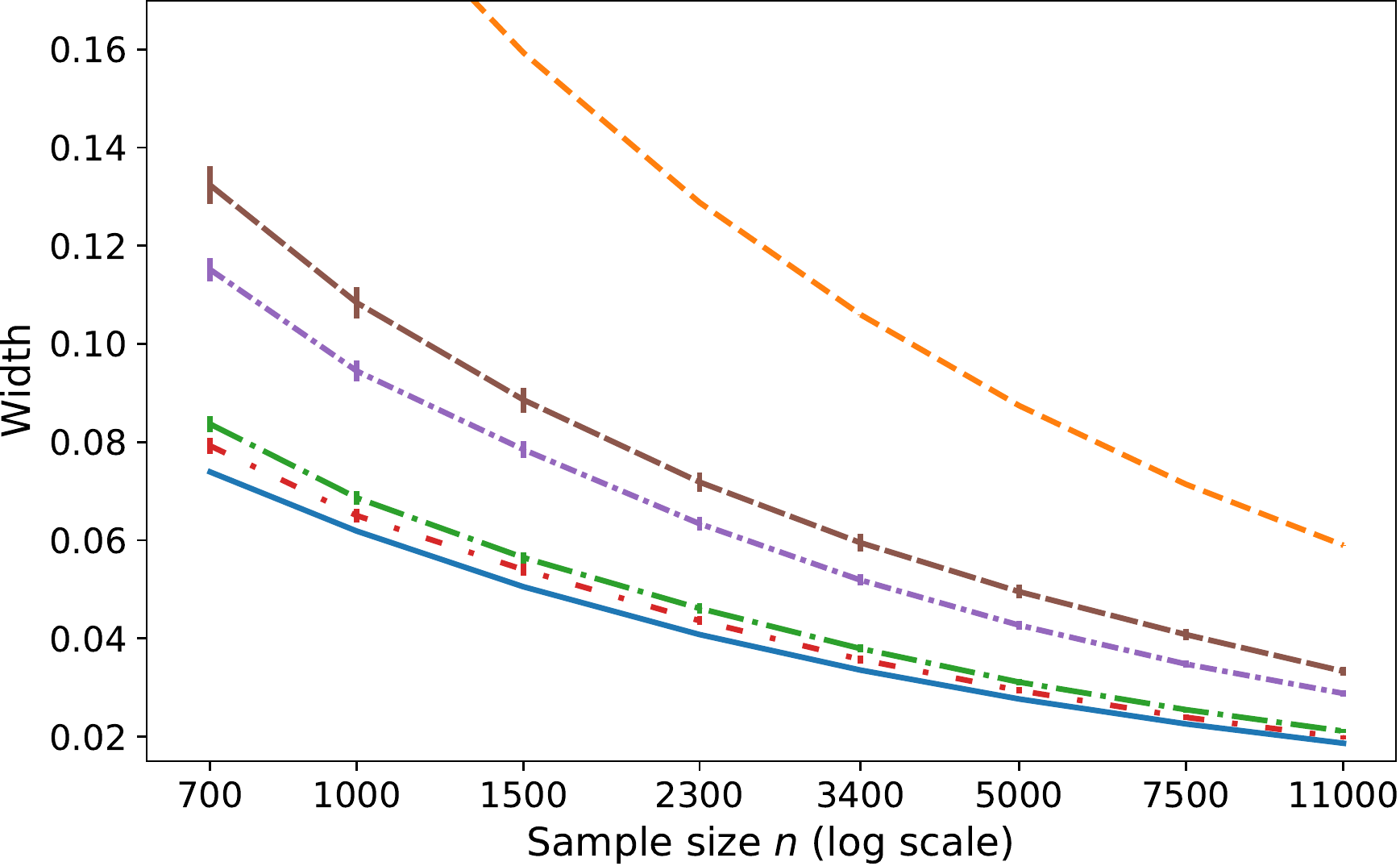}
    \end{subfigure}
    \caption{Test error coverage (left) and width (right) of $95\%$ confidence intervals (see \cref{sec:ci-experiment}). \tbf{Top:} $\ell^2$-regularized logistic regression classifier. \tbf{Middle:} Random forest classifier. \tbf{Bottom:} Neural network classifier.}
    \label{fig:test-error-CI-clf}
\end{figure}

\begin{figure}[h!]
\centering
    \begin{subfigure}{\subfigfracin\linewidth}
        \includegraphics[width=\imgfrac\linewidth]{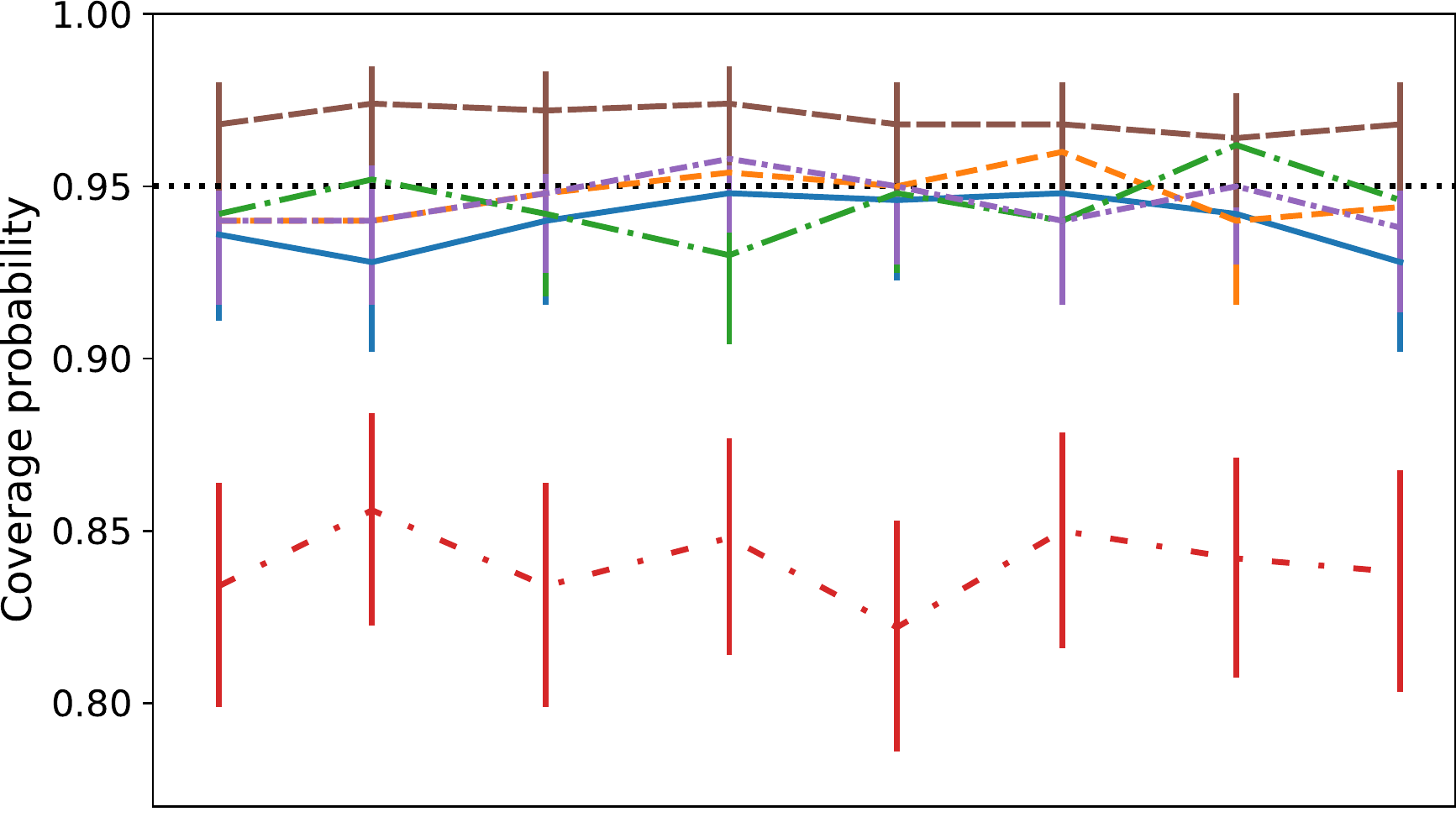}
    \end{subfigure}\hspace{\imgspace\linewidth}%
     \begin{subfigure}{\subfigfracin\linewidth}
             \includegraphics[width=\imgfrac\linewidth]{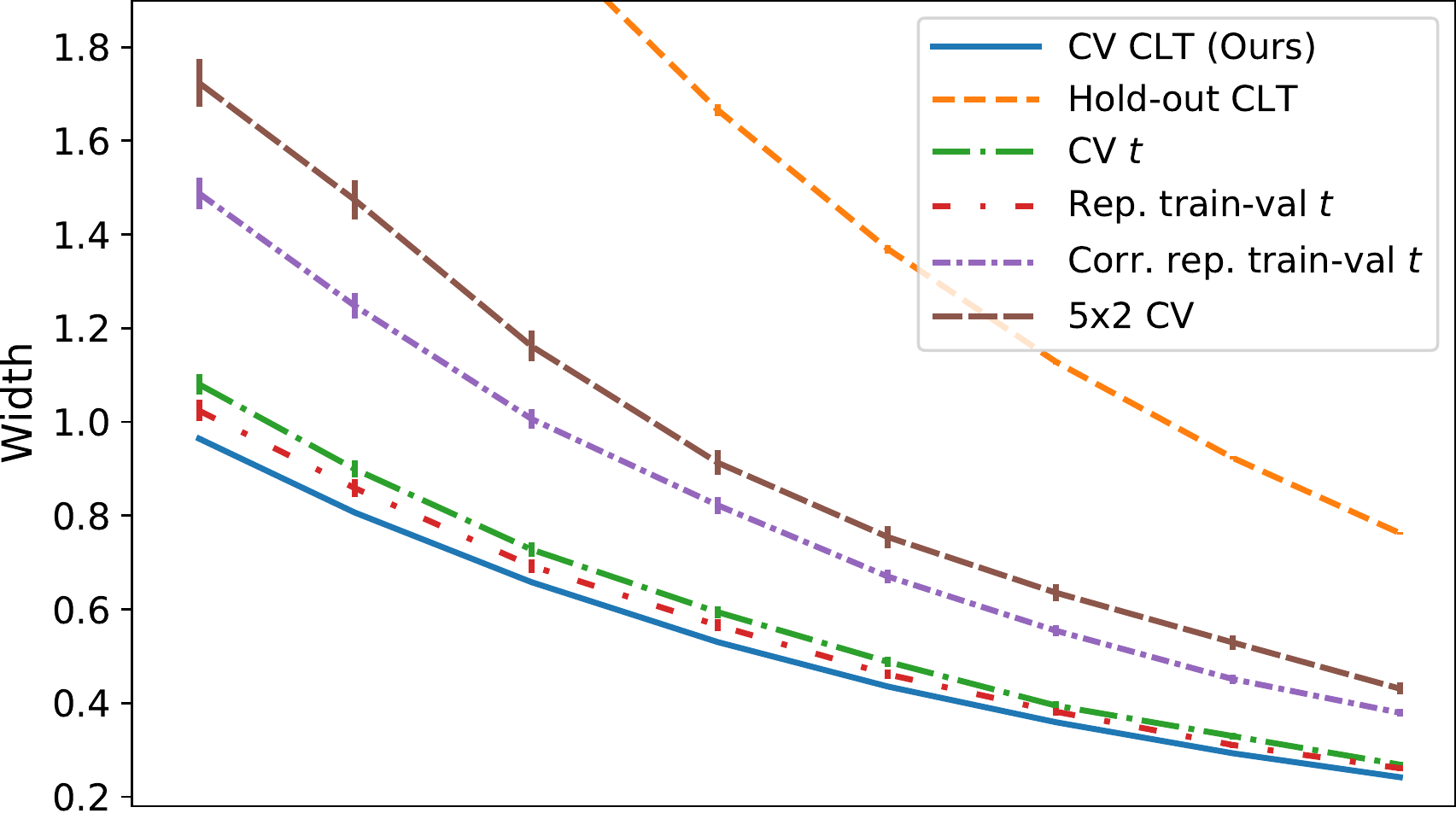}
    \end{subfigure}

    \vspace{\vgap\linewidth}
    
    \begin{subfigure}{\subfigfracin\linewidth}
        \includegraphics[width=\imgfrac\linewidth]{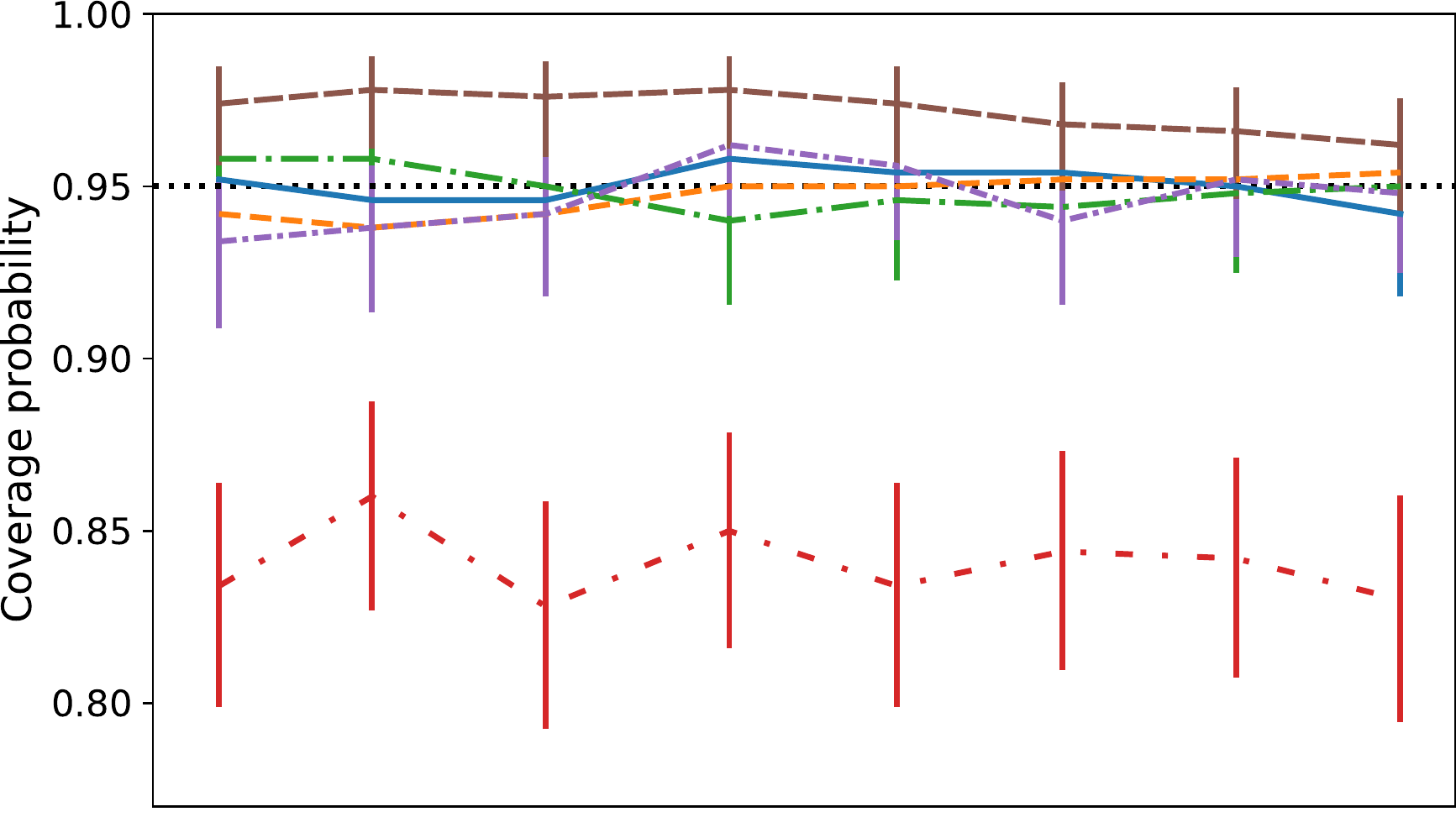}
    \end{subfigure}\hspace{\imgspace\linewidth}%
    \begin{subfigure}{\subfigfracin\linewidth}
        \includegraphics[width=\imgfrac\linewidth]{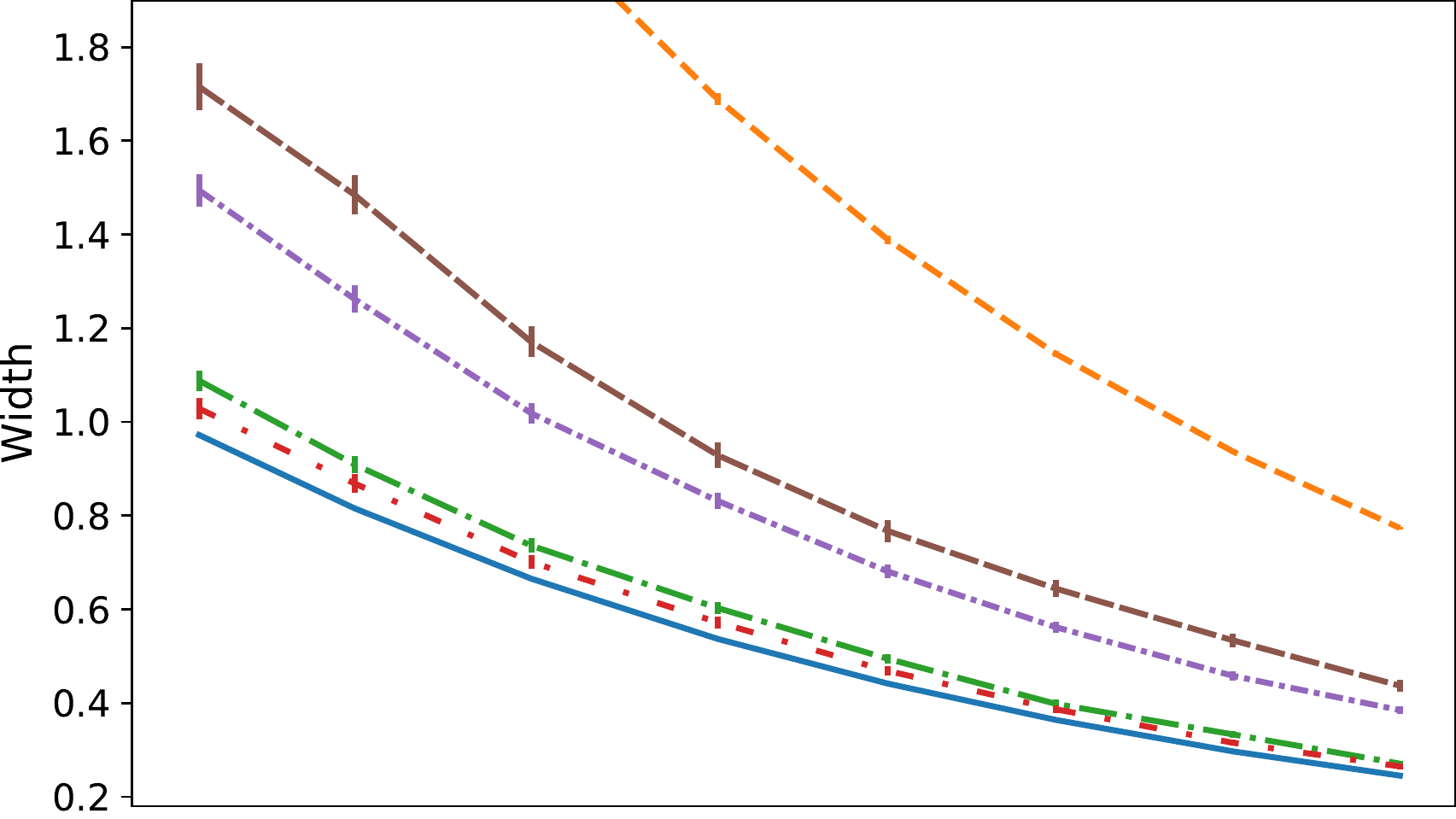}
    \end{subfigure}

    \vspace{\vgap\linewidth}
    
    \begin{subfigure}{\subfigfracin\linewidth}
        \includegraphics[width=\imgfrac\linewidth]{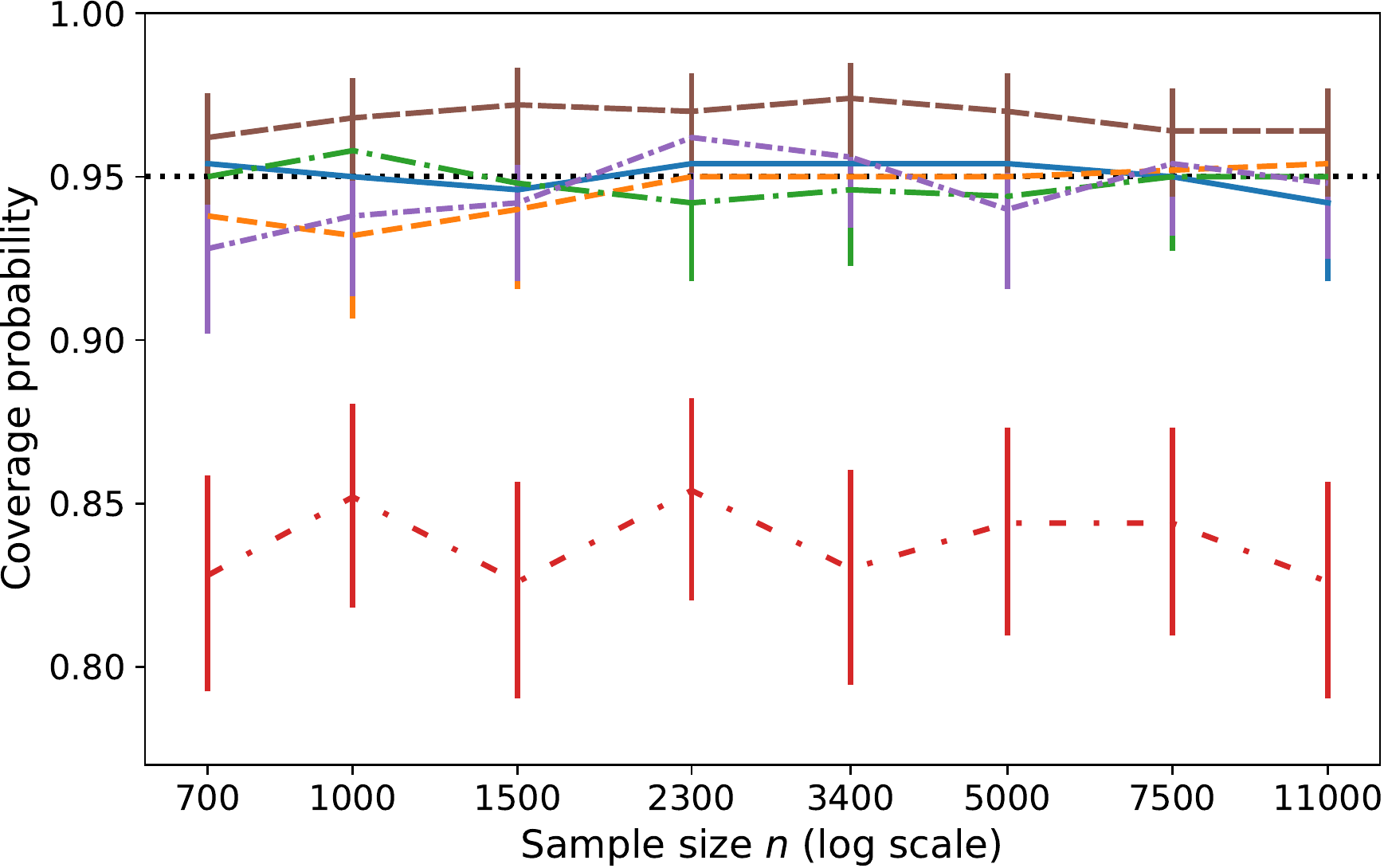}
    \end{subfigure}\hspace{\imgspace\linewidth}%
    \begin{subfigure}{\subfigfracin\linewidth}
        \includegraphics[width=\imgfrac\linewidth]{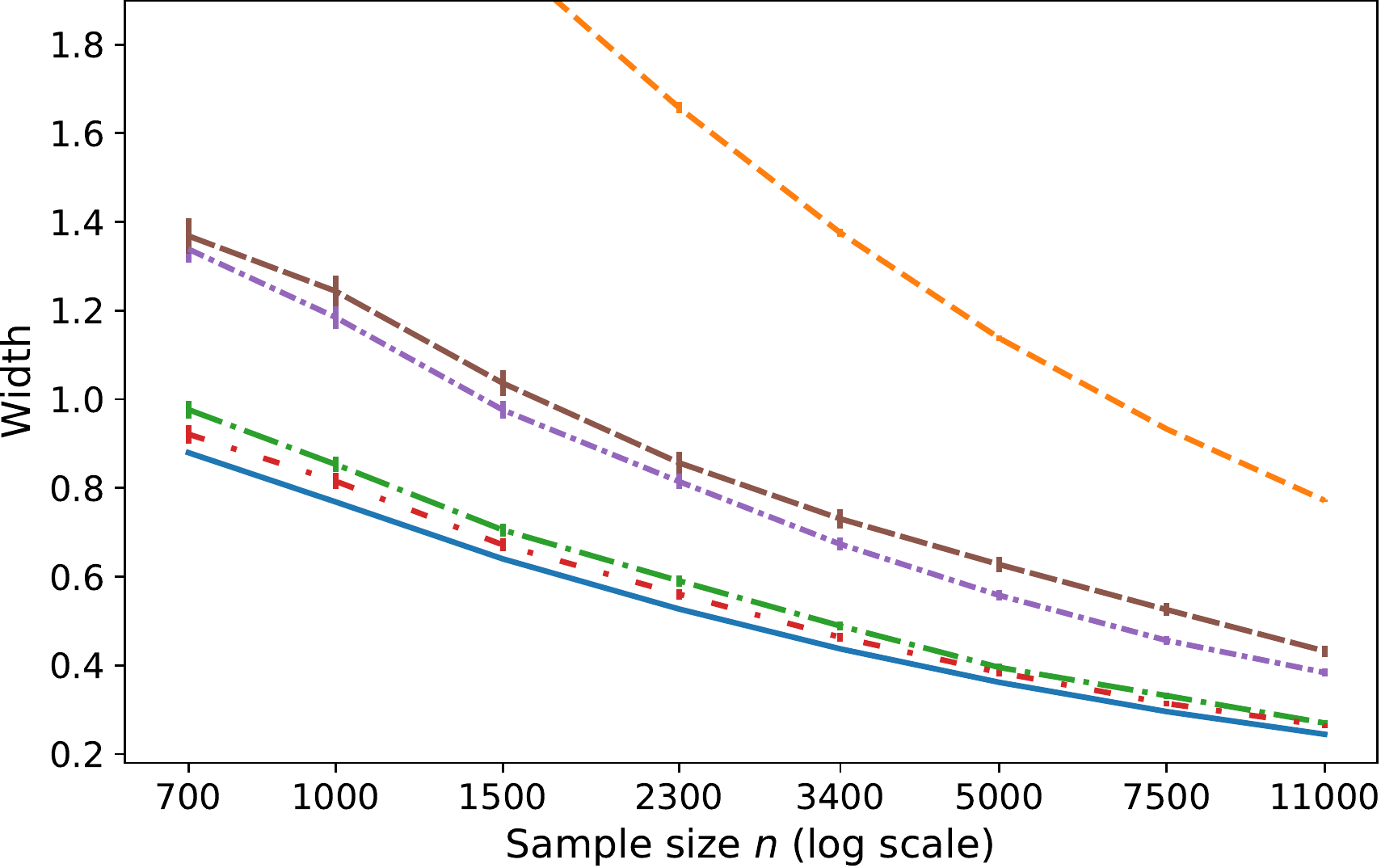}
    \end{subfigure}
    \caption{Test error coverage (left) and width (right) of $95\%$ confidence intervals (see \cref{sec:ci-experiment}). \tbf{Top:} Random forest regression. \tbf{Middle:} Ridge regression. \tbf{Bottom:} Neural network regression.}
    \label{fig:test-error-CI-reg}
\end{figure}

\subsection{Additional results from  \cref{sec:sim:test}: Testing for improved algorithm performance}\label{sec:additional-results-test}

In this section, we provide additional experimental details and results for the testing for improved algorithm performance experiments of \cref{sec:sim:test}.
We highlight that the aim of this assessment is not to establish power convergence or to assess power in an absolute sense but rather to verify whether, for a diversity of settings encountered in real learning problems, our proposed tests provide power comparable to or better than the most popular heuristics from the literature. 
For all testing experiments, we estimate size as $\frac{\text{\# of rejections in } H_0 \text{  replications}}{\text{\# } H_0 \text{  replications}}$ and power as $\frac{\text{\# of rejections in } H_1 \text{  replications}}{\text{\# } H_1 \text{  replications}}$, where each simulation is classified as $H_0$ or $H_1$ depending on which algorithm has smaller test error.
Moreover, a size point is only displayed if at least 25 replications were classified as $H_0$, and a power point is only displayed if at least 25 replications were classified as $H_1$.

The remaining results of the testing experiments described in \cref{sec:sim:test} are provided in \cref{fig:test-error-improvement-reg-RF-NN,fig:test-error-improvement-clf-LR-RF,fig:test-error-improvement-clf-LR-NN,fig:test-error-improvement-clf-NN-RF,fig:test-error-improvement-reg-RR-NN,fig:test-error-improvement-reg-RR-RF}.
In contrast to \cref{fig:test-error-improvement},\footnote{Recall that in \cref{fig:test-error-improvement} we identified the algorithm $\alg_1$ that more often had smaller test error across our simulations and displayed the power of $H_1: \err{1}{2}$ and the size of the level $\alpha=0.05$ test of $H_1: \err{2}{1}$.} in \cref{fig:test-error-improvement-reg-RF-NN,fig:test-error-improvement-clf-LR-RF,fig:test-error-improvement-clf-LR-NN,fig:test-error-improvement-clf-NN-RF,fig:test-error-improvement-reg-RR-NN,fig:test-error-improvement-reg-RR-RF} we plot the size and power of the level $\alpha=0.05$ test of $H_1: \err{1}{2}$ in the left column of each figure and the size and power of the level $\alpha$ test of $H_1: \err{2}{1}$ in the right column.
Notably, we only observe size estimates exceeding the level when the number of $H_0$ replications is very small (that is, when one algorithm improves upon the other so infrequently that the Monte Carlo error in the size estimate is large).

\subsection{Testing with synthetically generated labels}\label{sec:synthetic-test}
We complement the real-data hypothesis testing experiments of \cref{sec:sim:test} with a controlled experiment in which class labels are synthetically generated from a known logistic regression distribution.
Specifically, we replicate the exact classification experimental setup of \cref{sec:sim:test} to compare logistic regression and random forest classification and use the same \texttt{Higgs} dataset covariates, but we replace each datapoint label $Y_i$ with an independent draw from the logistic regression distribution
$Y_i \sim \Ber(\frac{1}{1+\exp({-\inner{X_i}{\beta}})})$ for $\beta$ a 28-dimensional vector with odd entries equal to $1$ and even entries equal to $-1$.
This experiment enables us to evaluate our hypothesis tests in a \emph{realizable} setting in which the true label generating distribution belongs to the logistic regression model family.
In \cref{fig:test-error-improvement-clf-RF-LR-simul}, we plot in the left column the size and power of the level $\alpha=0.05$ test of $H_1$: random forest improves upon $\ell^2$-regularized logistic regression classifier, and in the right column the size and power of the level $\alpha$ test of $H_1$: $\ell^2$-regularized logistic regression classifier improves upon random forest. 
As expected, almost all replications satisfy that $\ell^2$-regularized logistic regression improves upon random forest and we observe that in this setting as well, our method consistently outperforms other alternatives.

\begin{figure}[h!]
\centering
    \begin{subfigure}{\subfigfracin\linewidth}
        \includegraphics[width=\imgfrac\linewidth]{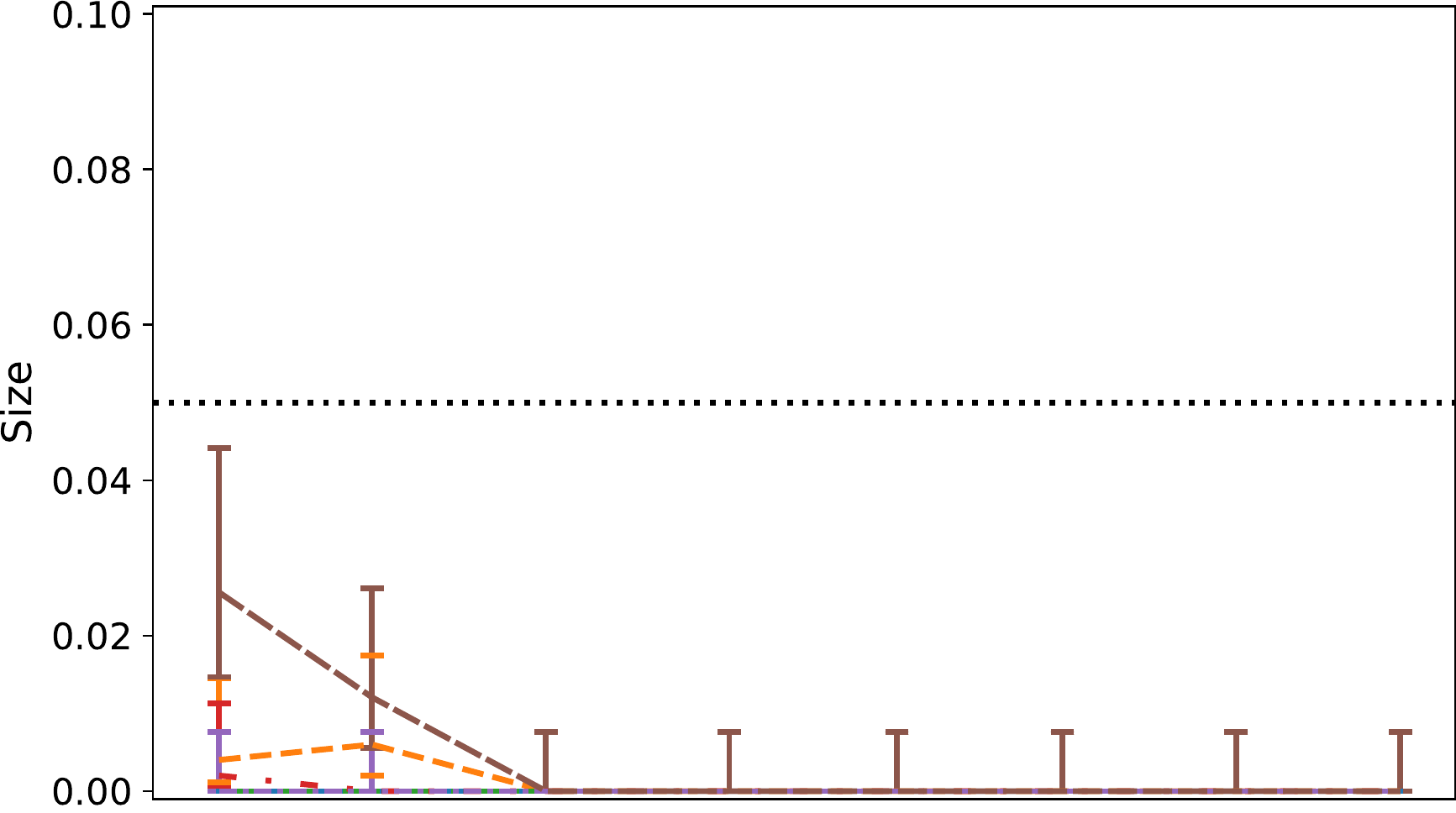}
    \end{subfigure}\hspace{\imgspace\linewidth}%
     \begin{subfigure}{\subfigfracin\linewidth}
             \includegraphics[width=\imgfrac\linewidth]{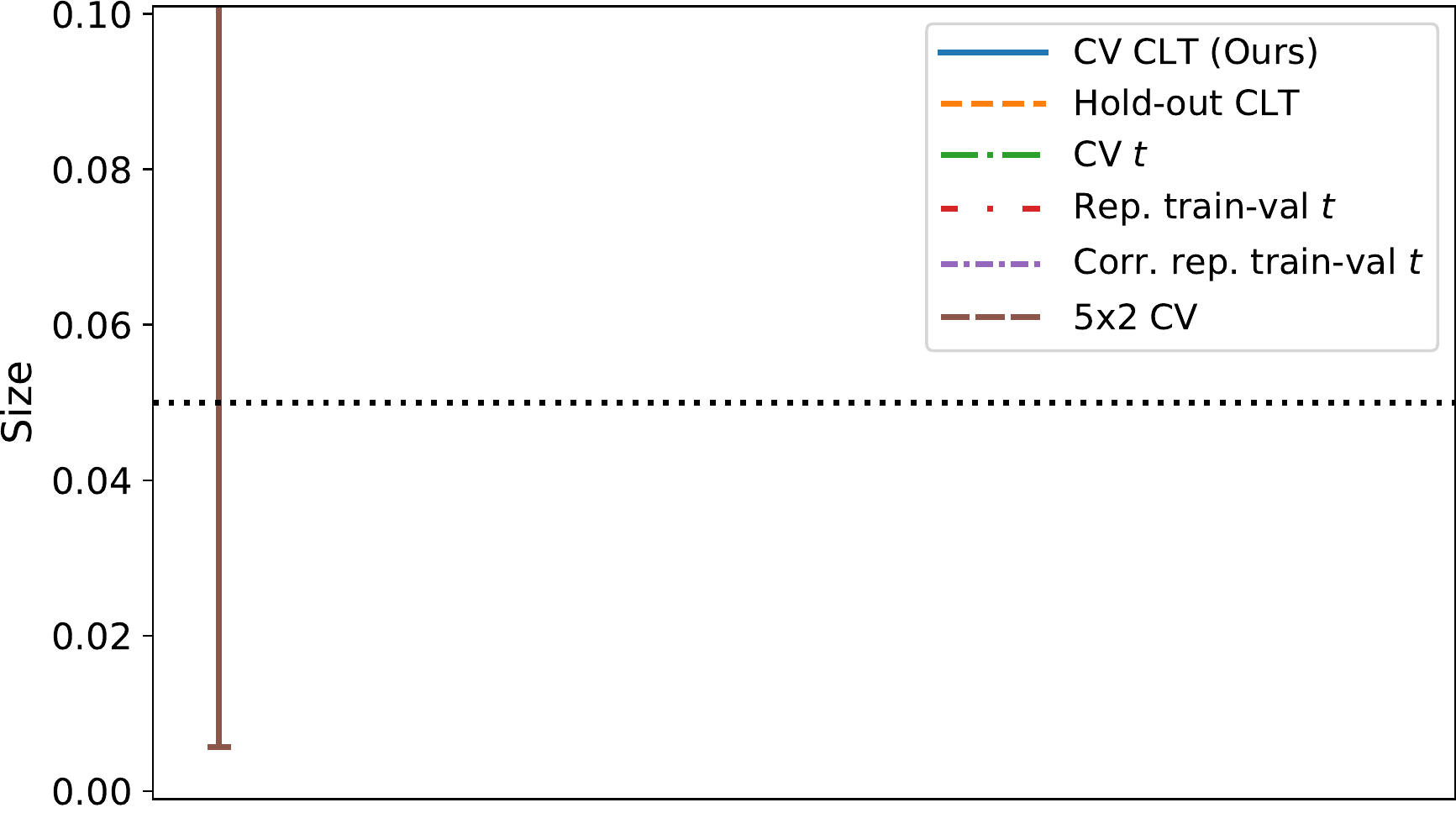}
    \end{subfigure}

    \vspace{\vgap\linewidth}
    
    \begin{subfigure}{\subfigfracin\linewidth}
        \includegraphics[width=\imgfrac\linewidth]{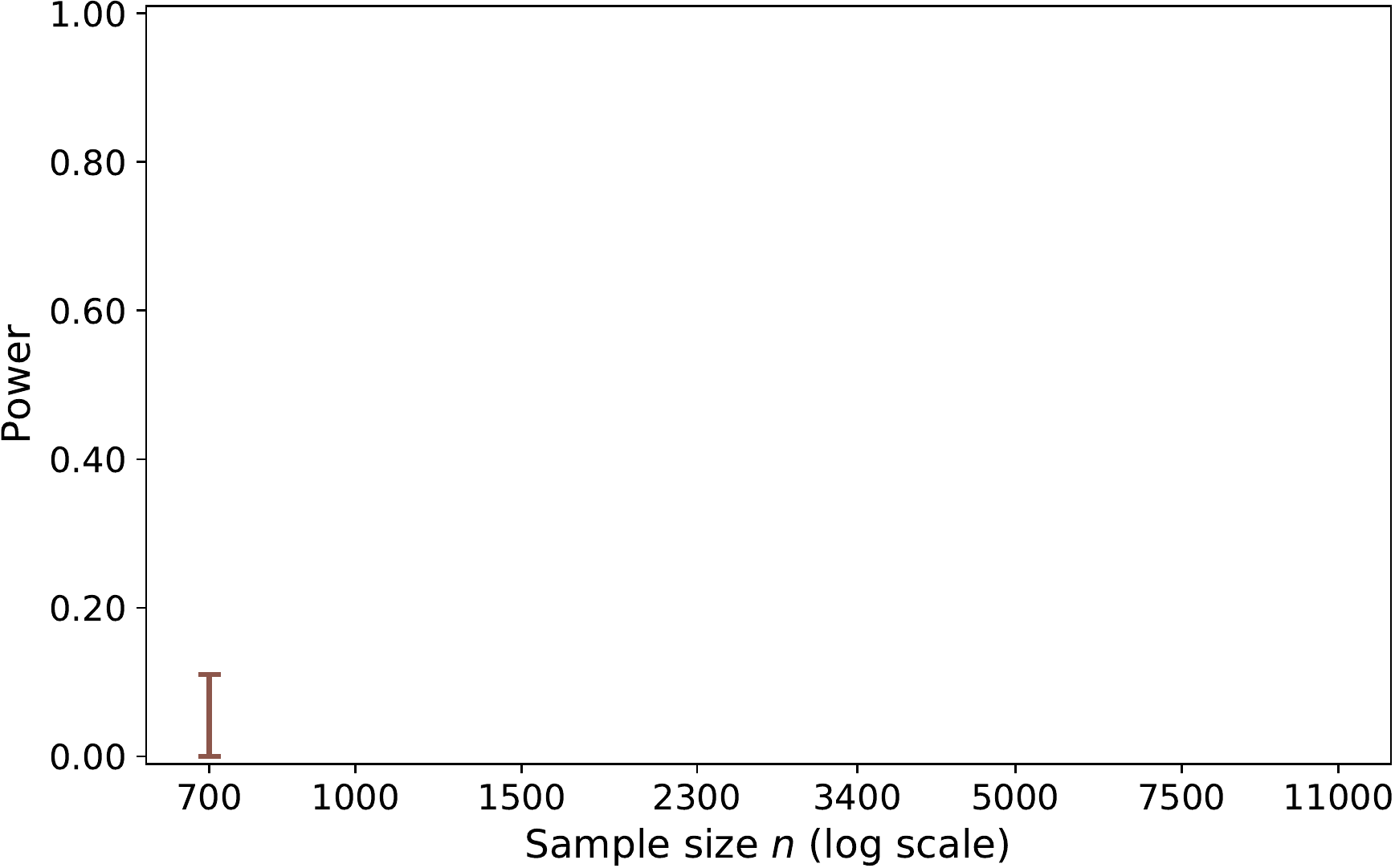}
    \end{subfigure}\hspace{\imgspace\linewidth}%
    \begin{subfigure}{\subfigfracin\linewidth}
        \includegraphics[width=\imgfrac\linewidth]{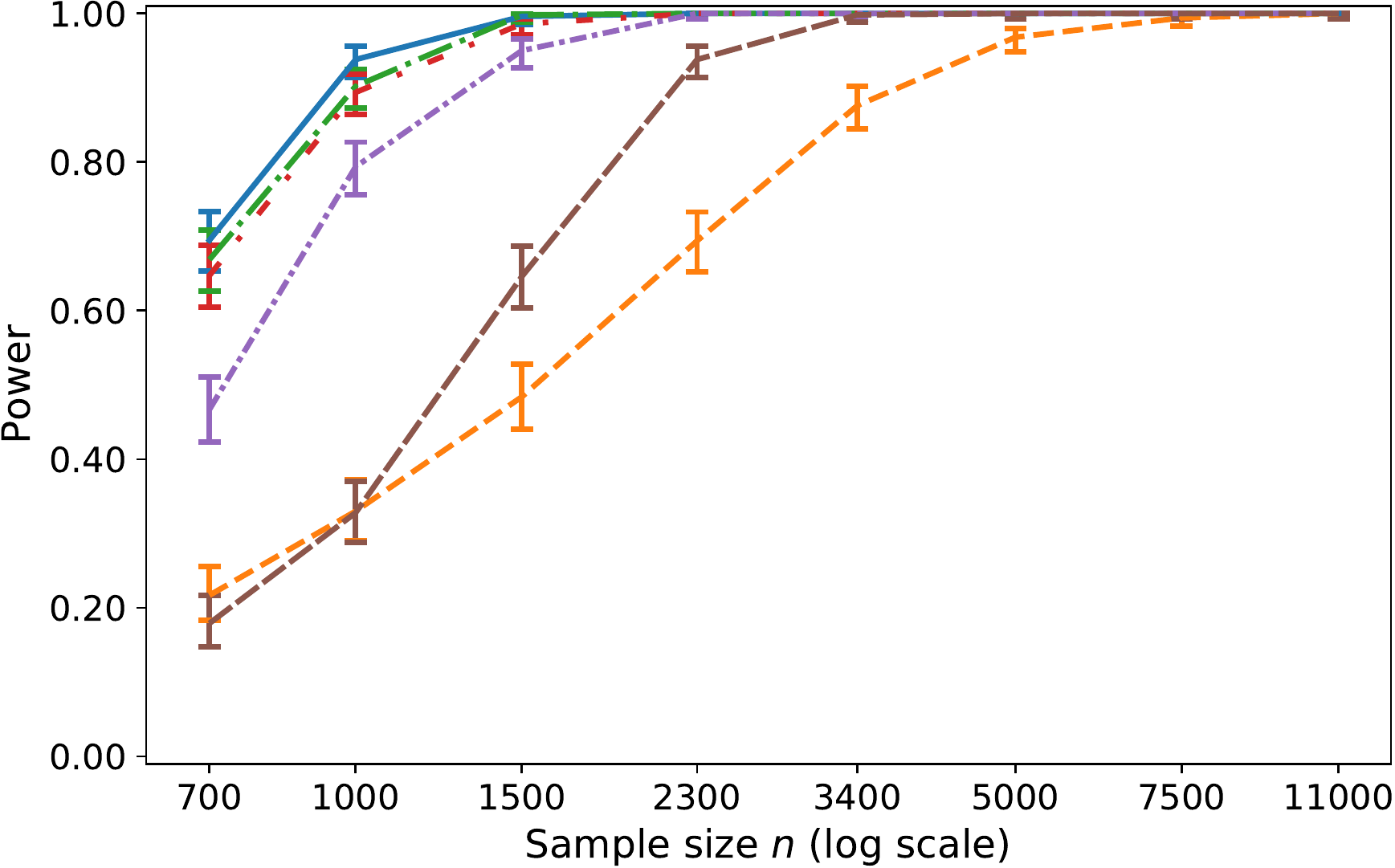}
    \end{subfigure}

    \caption{Size (top) and power (bottom) of level-$0.05$ tests for improved test error (see \cref{sec:sim:test}). \tbf{Left}: Testing $H_1$: neural network improves upon $\ell^2$-regularized logistic regression classifier. \tbf{Right}: Testing $H_1$: $\ell^2$-regularized logistic regression classifier improves upon neural network.}
    \label{fig:test-error-improvement-clf-LR-NN}
\end{figure}

\begin{figure}[h!]
\centering
    \begin{subfigure}{\subfigfracin\linewidth}
        \includegraphics[width=\imgfrac\linewidth]{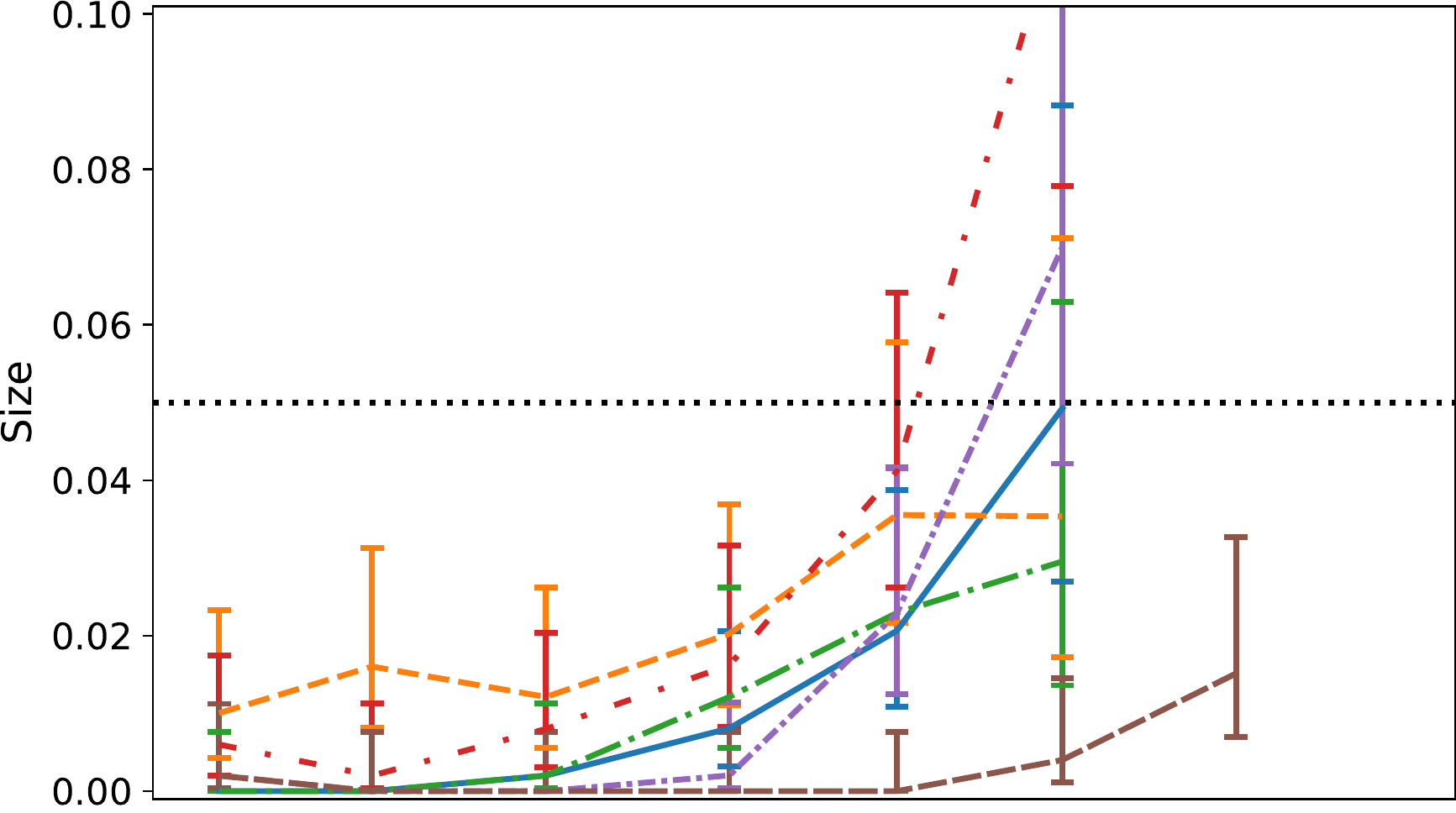}
    \end{subfigure}\hspace{\imgspace\linewidth}%
     \begin{subfigure}{\subfigfracin\linewidth}
             \includegraphics[width=\imgfrac\linewidth]{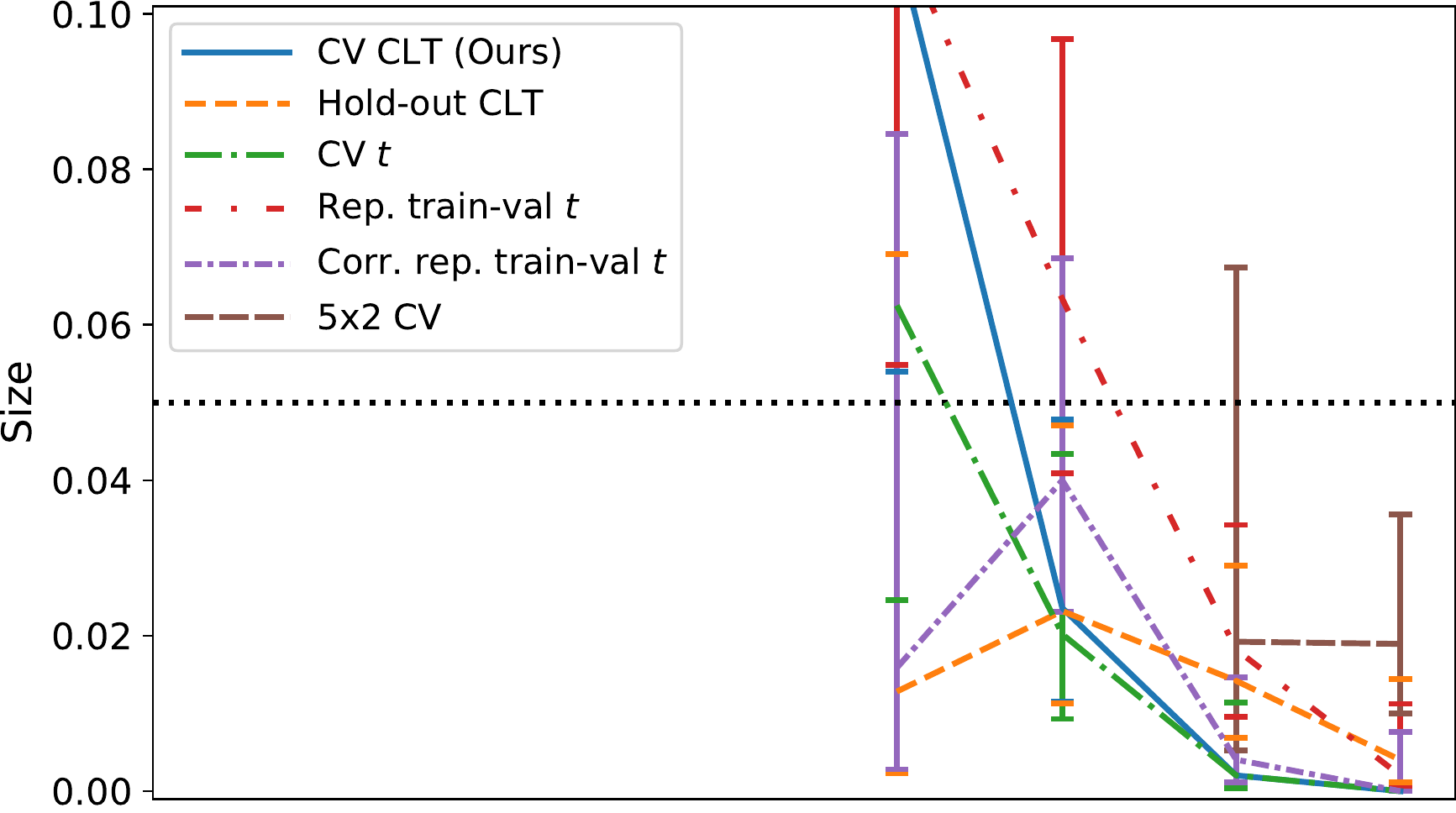}
    \end{subfigure}

    \vspace{\vgap\linewidth}
    
    \begin{subfigure}{\subfigfracin\linewidth}
        \includegraphics[width=\imgfrac\linewidth]{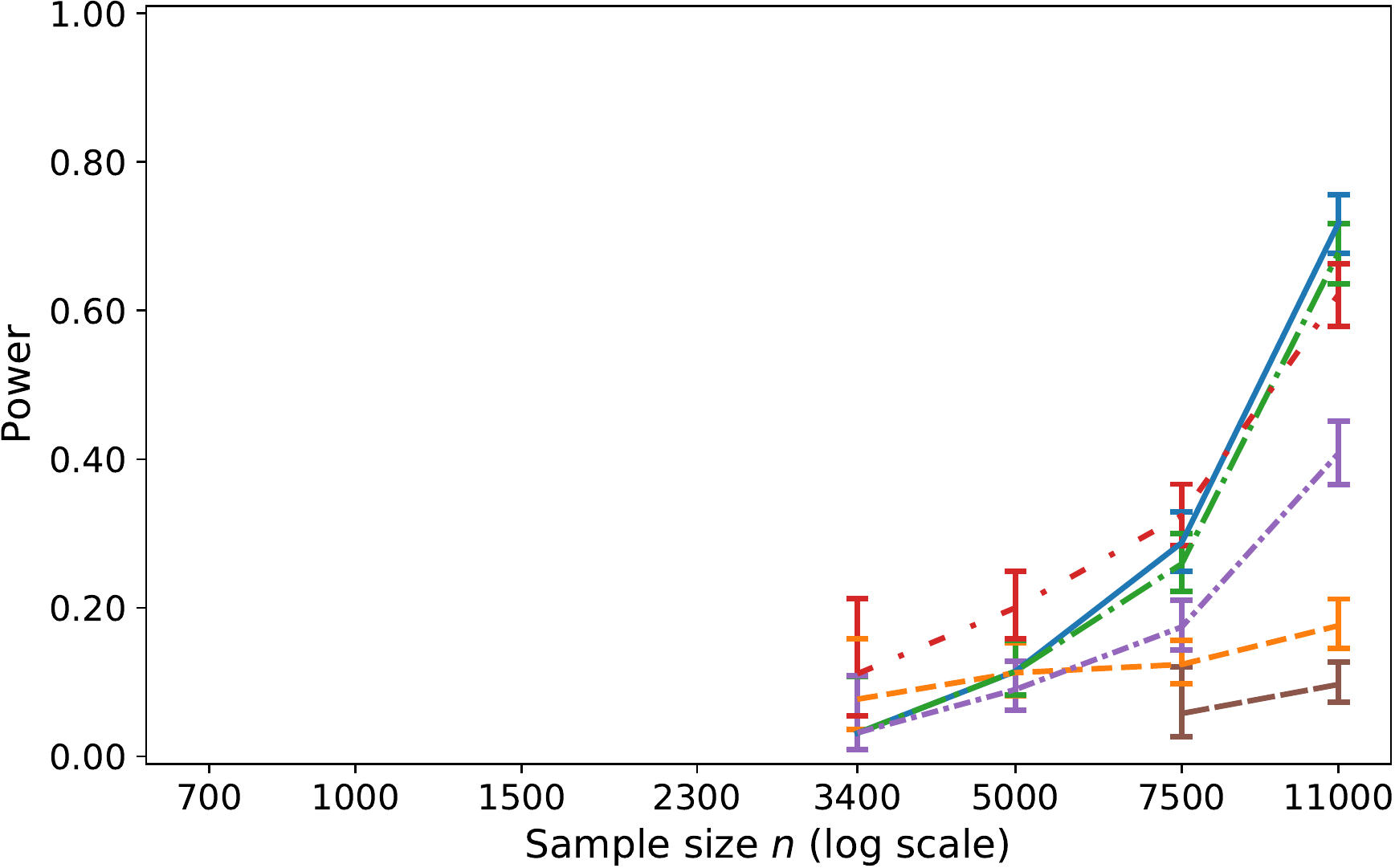}
    \end{subfigure}\hspace{\imgspace\linewidth}%
    \begin{subfigure}{\subfigfracin\linewidth}
        \includegraphics[width=\imgfrac\linewidth]{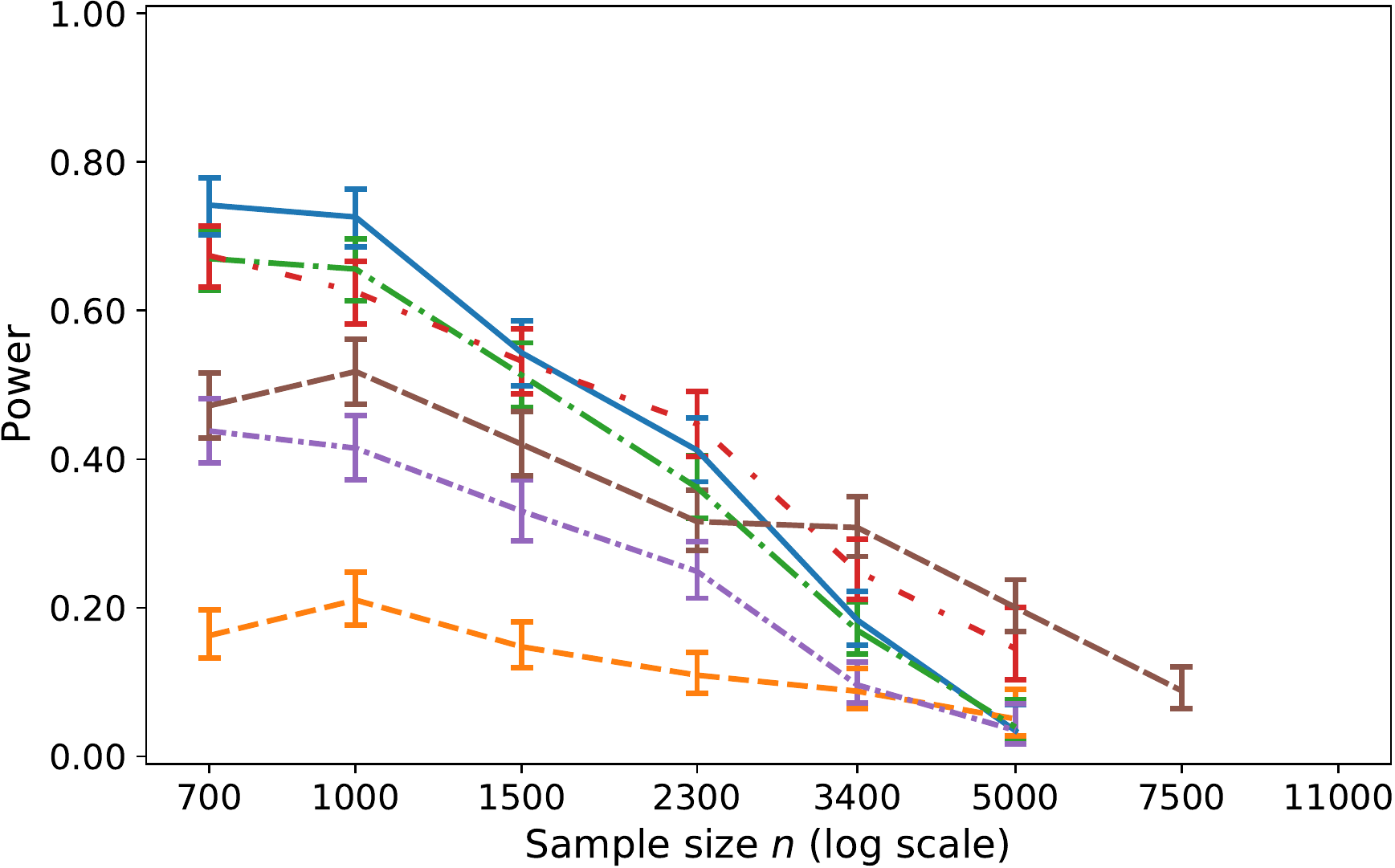}
    \end{subfigure}

    \caption{Size (top) and power (bottom) of level-$0.05$ tests for improved test error (see \cref{sec:sim:test}). \tbf{Left}: Testing $H_1$: $\ell^2$-regularized logistic regression classifier improves upon random forest. \tbf{Right}: Testing $H_1$: random forest improves upon $\ell^2$-regularized logistic regression classifier.}
    \label{fig:test-error-improvement-clf-LR-RF}
\end{figure}

\begin{figure}[h!]
\centering
    \begin{subfigure}{\subfigfracin\linewidth}
        \includegraphics[width=\imgfrac\linewidth]{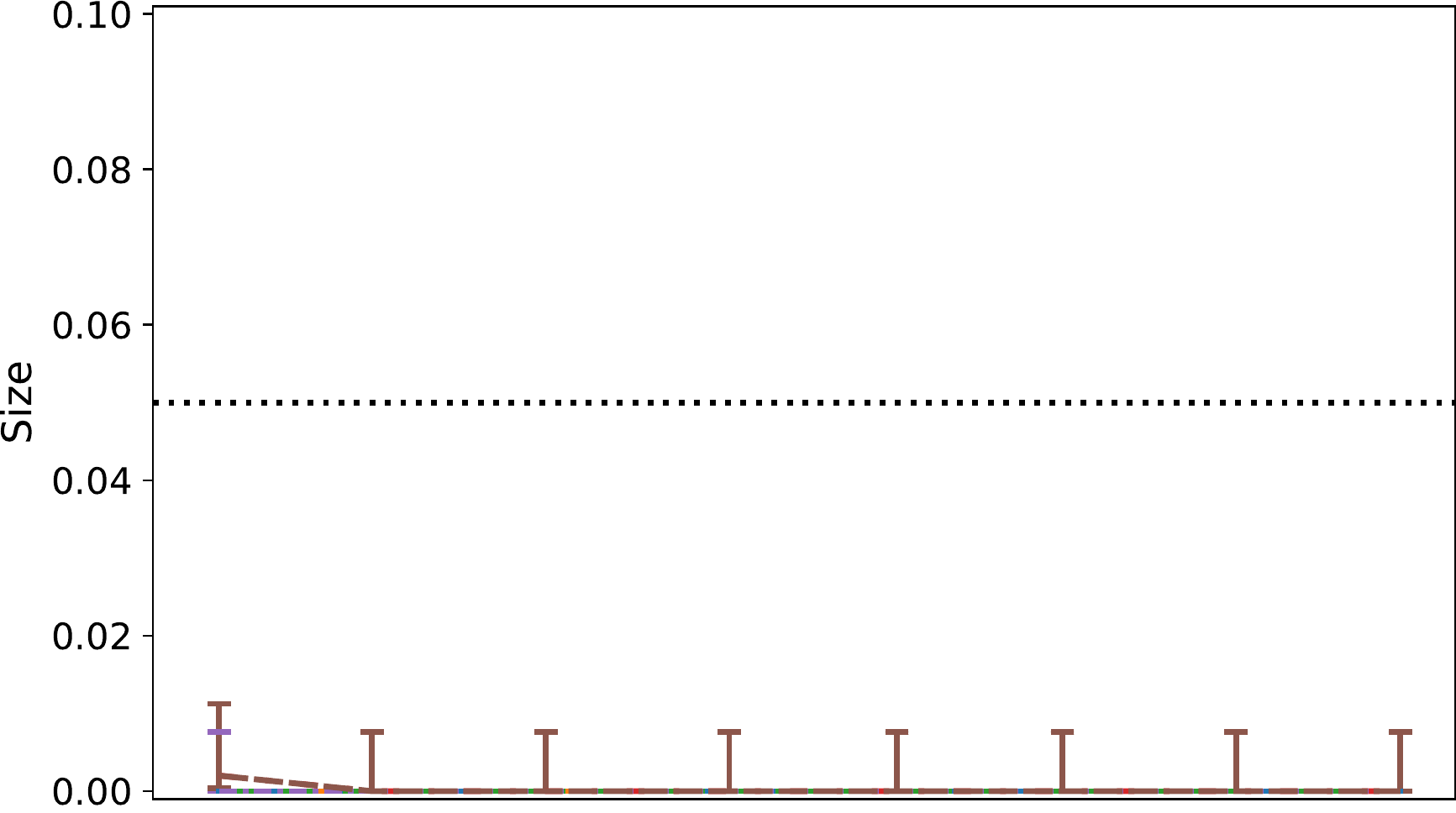}
    \end{subfigure}\hspace{\imgspace\linewidth}%
     \begin{subfigure}{\subfigfracin\linewidth}
             \includegraphics[width=\imgfrac\linewidth]{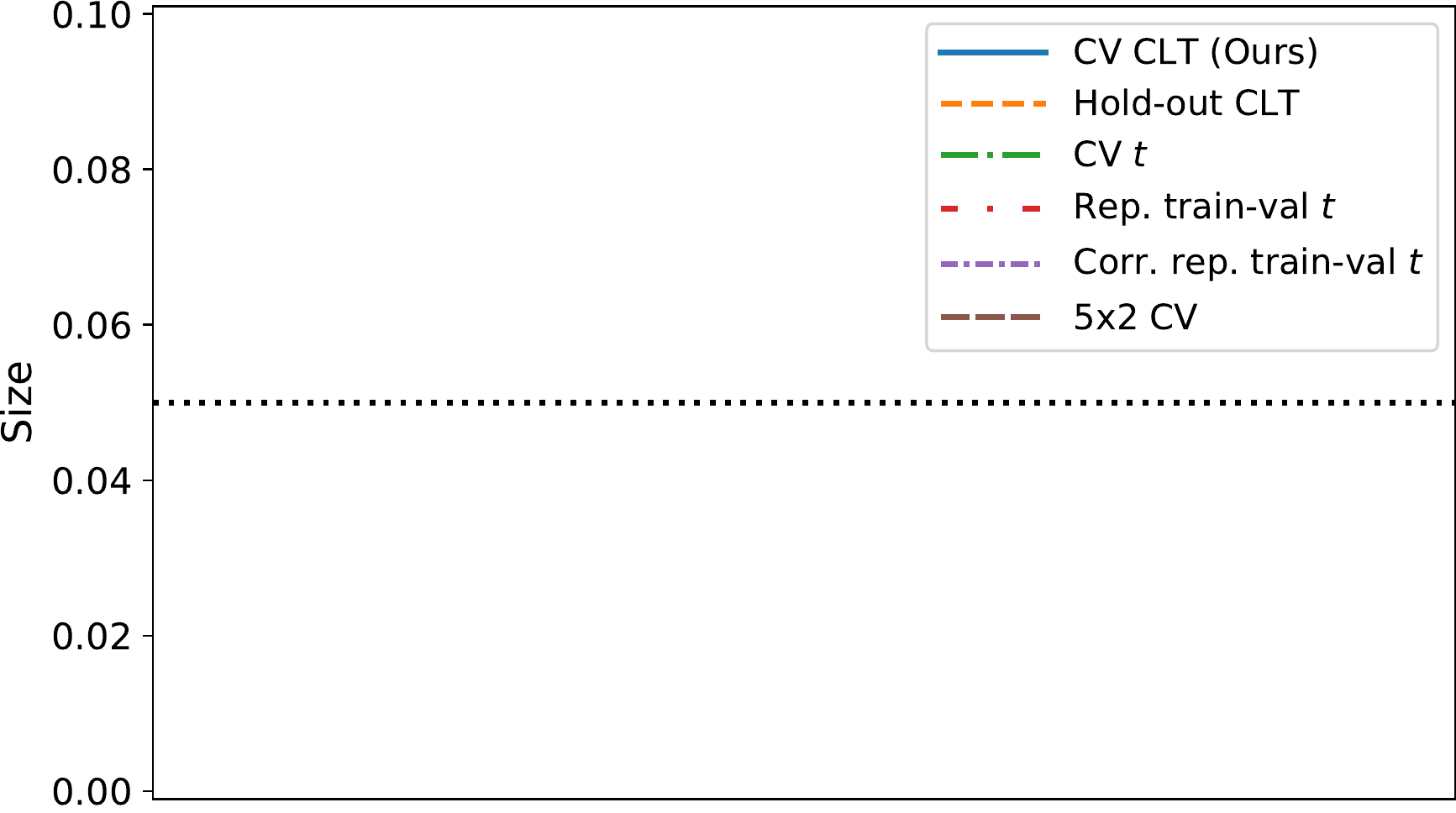}
    \end{subfigure}

    \vspace{\vgap\linewidth}
    
    \begin{subfigure}{\subfigfracin\linewidth}
        \includegraphics[width=\imgfrac\linewidth]{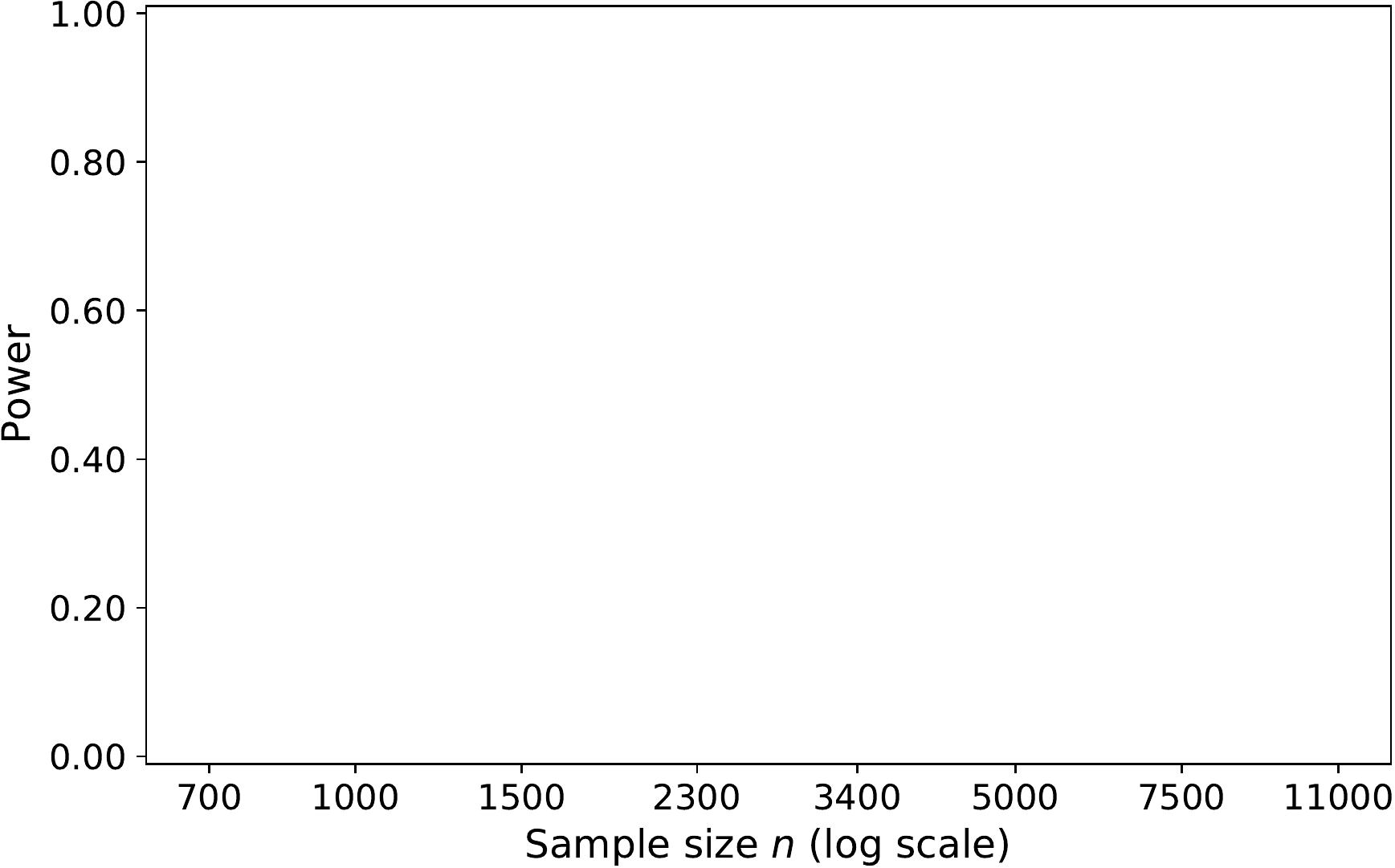}
    \end{subfigure}\hspace{\imgspace\linewidth}%
    \begin{subfigure}{\subfigfracin\linewidth}
        \includegraphics[width=\imgfrac\linewidth]{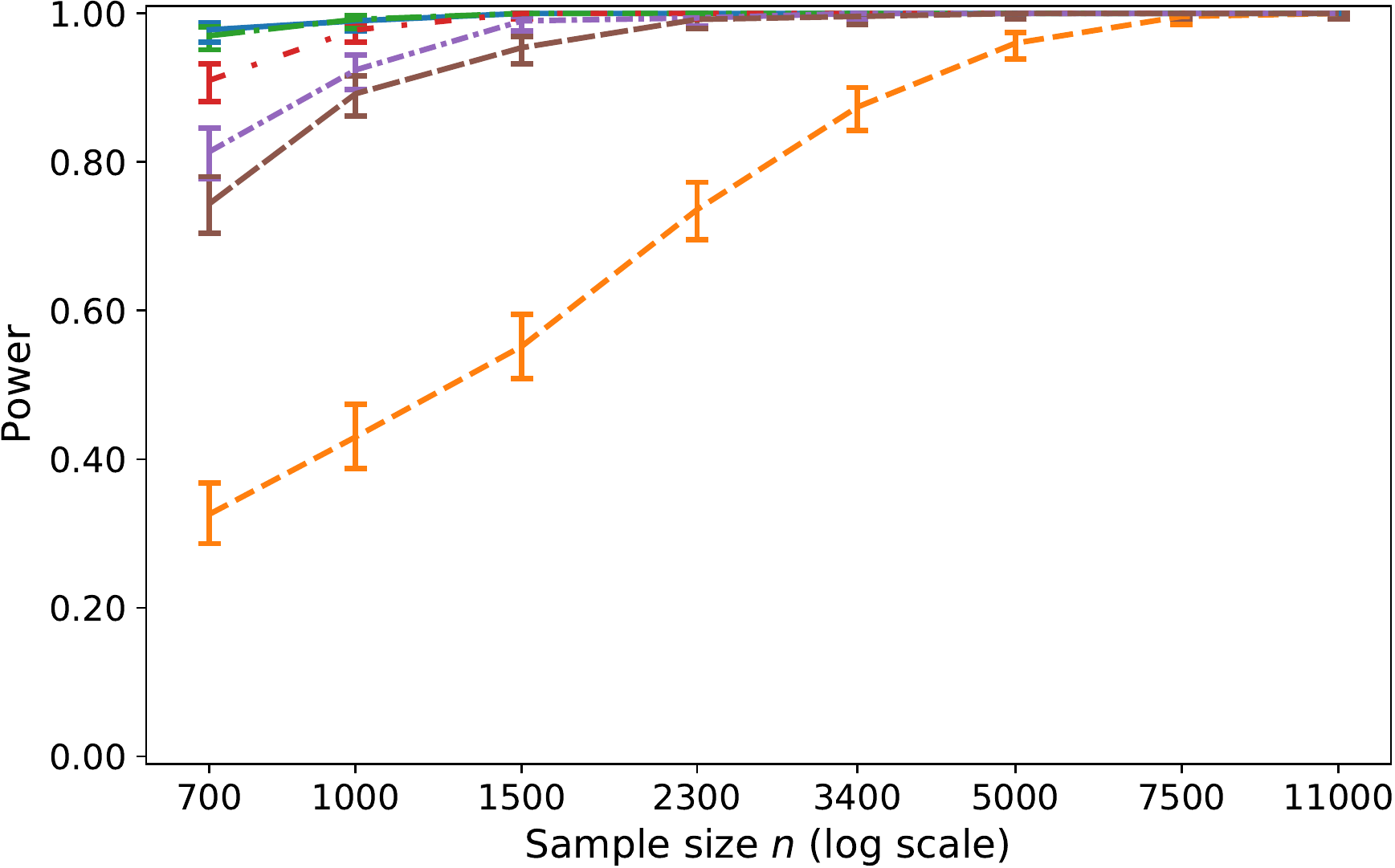}
    \end{subfigure}

    \caption{Size (top) and power (bottom) of level-$0.05$ tests for improved test error (see \cref{sec:sim:test}). \tbf{Left}: Testing $H_1$: neural network classifier improves upon random forest. \tbf{Right}: Testing $H_1$: random forest classifier improves upon neural network.}
    \label{fig:test-error-improvement-clf-NN-RF}
\end{figure}

\begin{figure}[h!]
\centering
    \begin{subfigure}{\subfigfracin\linewidth}
        \includegraphics[width=\imgfrac\linewidth]{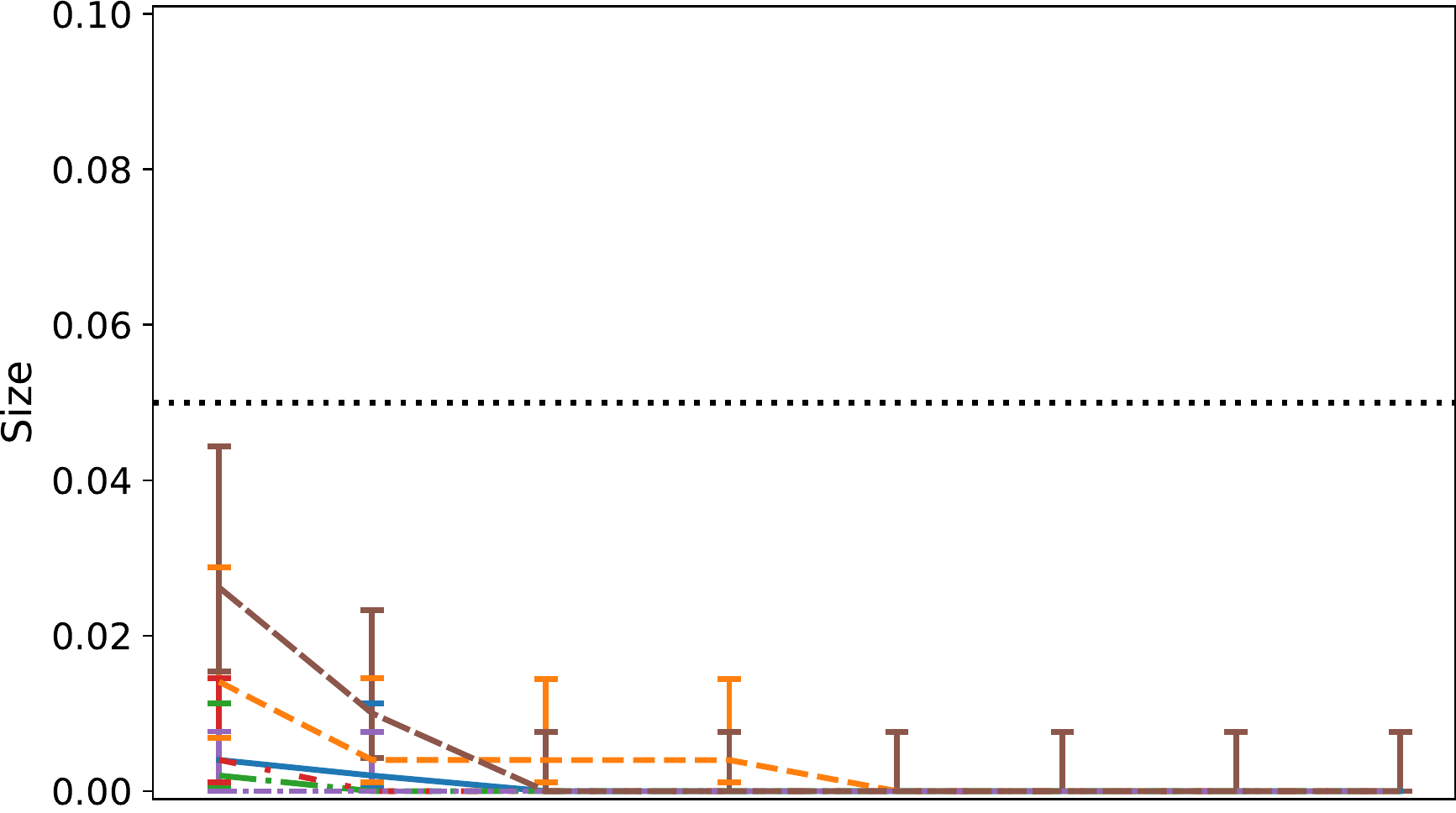}
    \end{subfigure}\hspace{\imgspace\linewidth}%
     \begin{subfigure}{\subfigfracin\linewidth}
             \includegraphics[width=\imgfrac\linewidth]{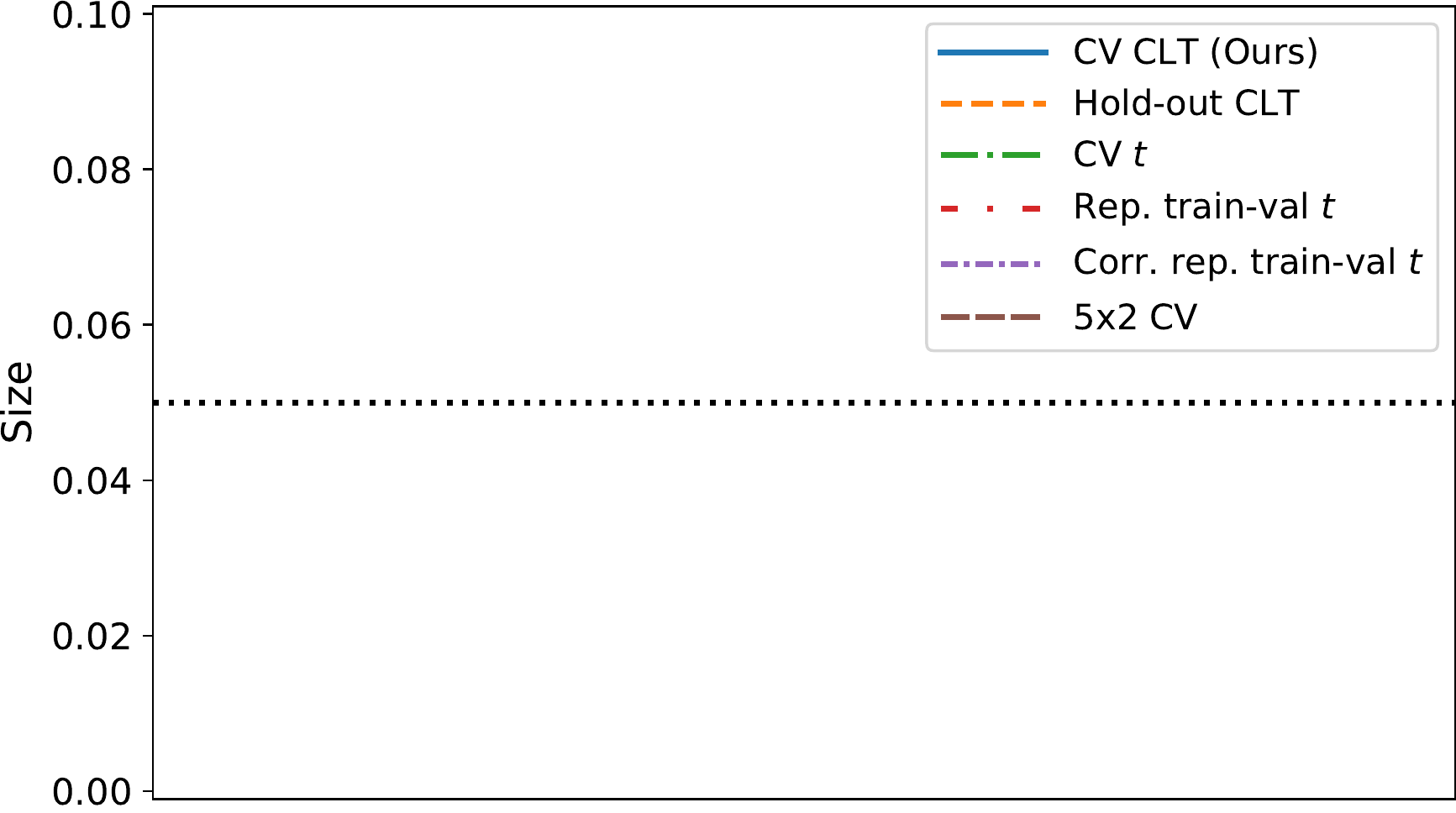}
    \end{subfigure}

    \vspace{\vgap\linewidth}
    
    \begin{subfigure}{\subfigfracin\linewidth}
        \includegraphics[width=\imgfrac\linewidth]{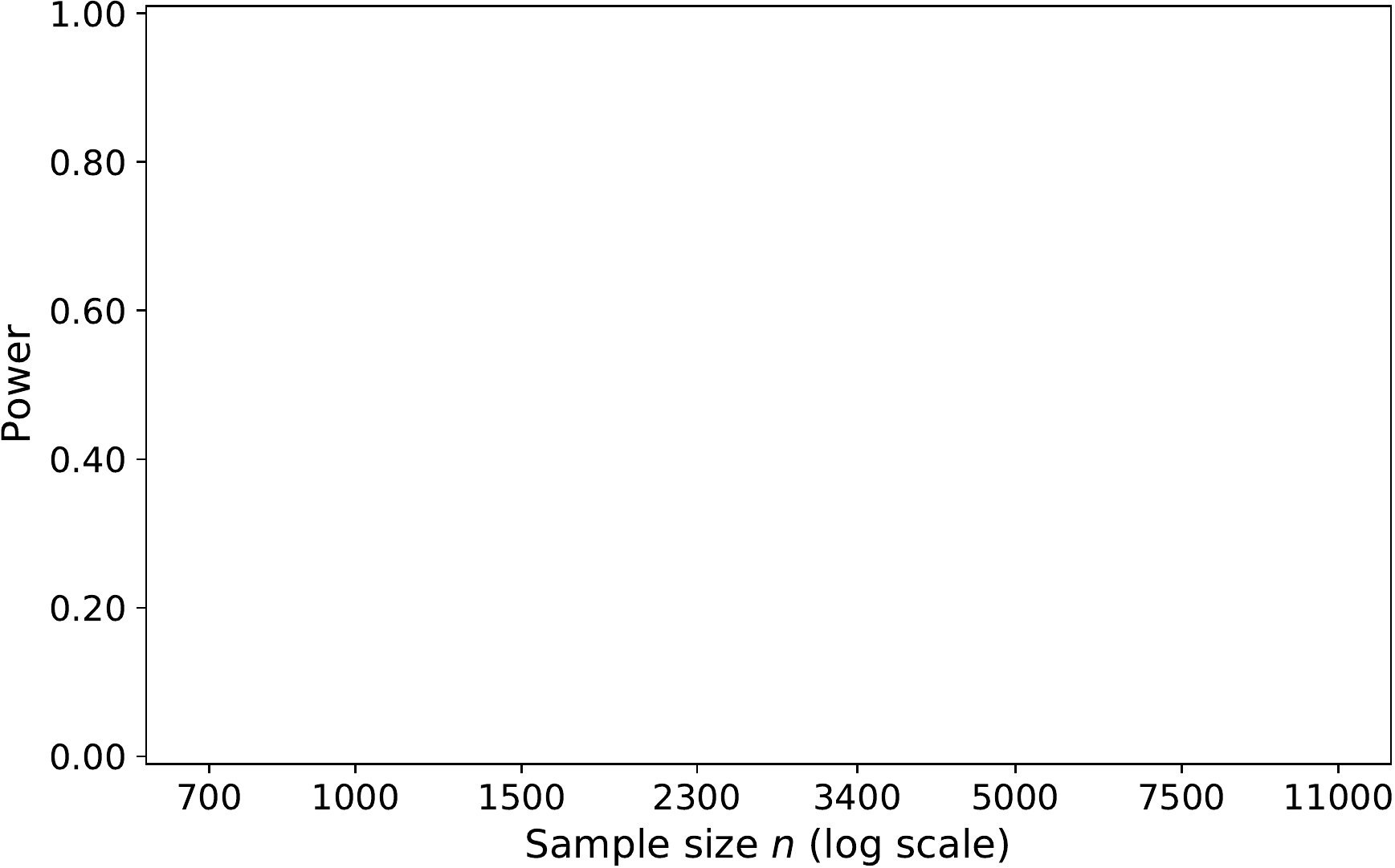}
    \end{subfigure}\hspace{\imgspace\linewidth}%
    \begin{subfigure}{\subfigfracin\linewidth}
        \includegraphics[width=\imgfrac\linewidth]{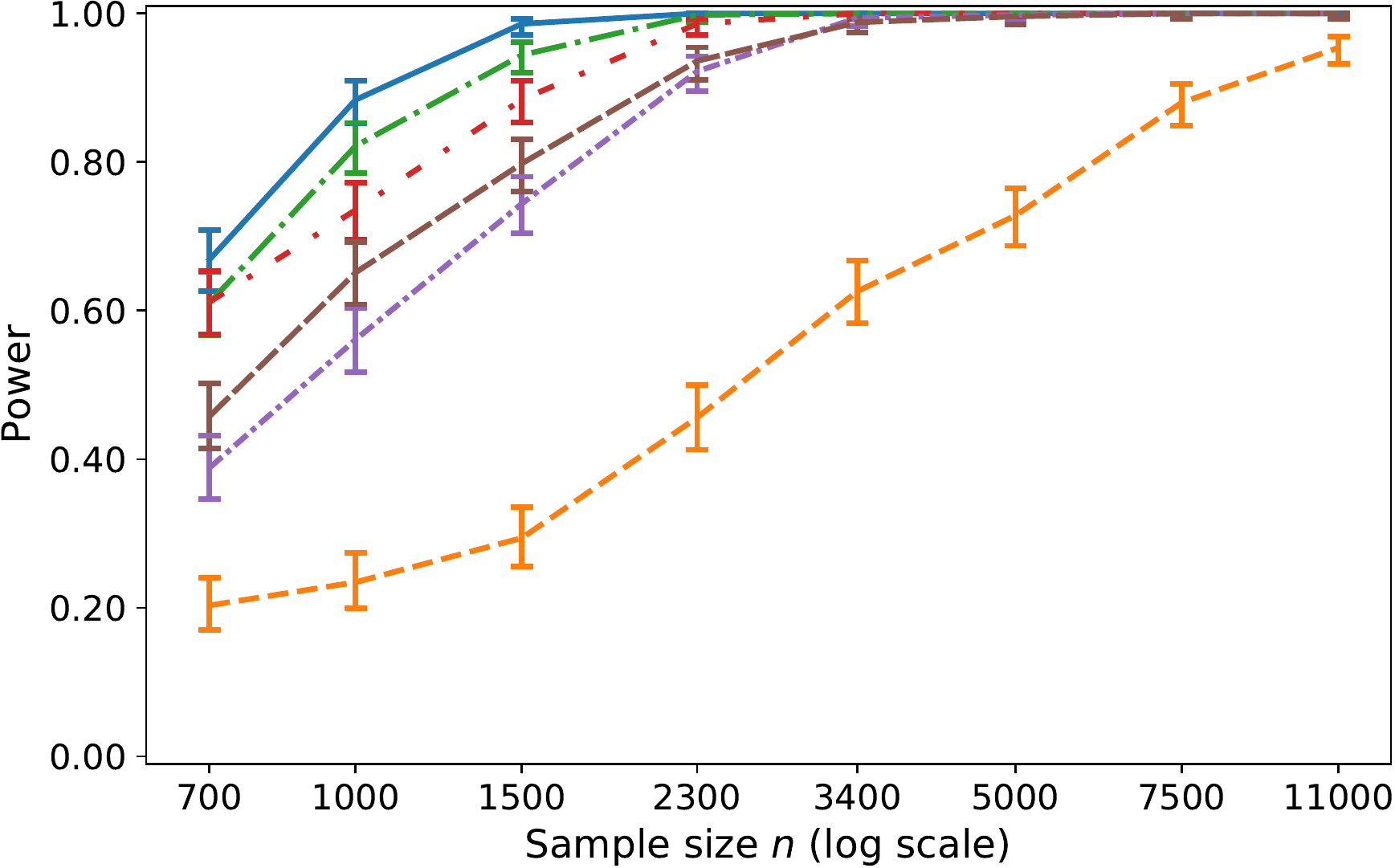}
    \end{subfigure}

    \caption{Size (top) and power (bottom) of level-$0.05$ tests for improved test error (see \cref{sec:sim:test}). \tbf{Left}: Testing $H_1$: ridge regression improves upon random forest. \tbf{Right}: Testing $H_1$: random forest improves upon ridge regression.}
    \label{fig:test-error-improvement-reg-RR-RF}
\end{figure}

\begin{figure}[h!]
\centering
    \begin{subfigure}{\subfigfracin\linewidth}
        \includegraphics[width=\imgfrac\linewidth]{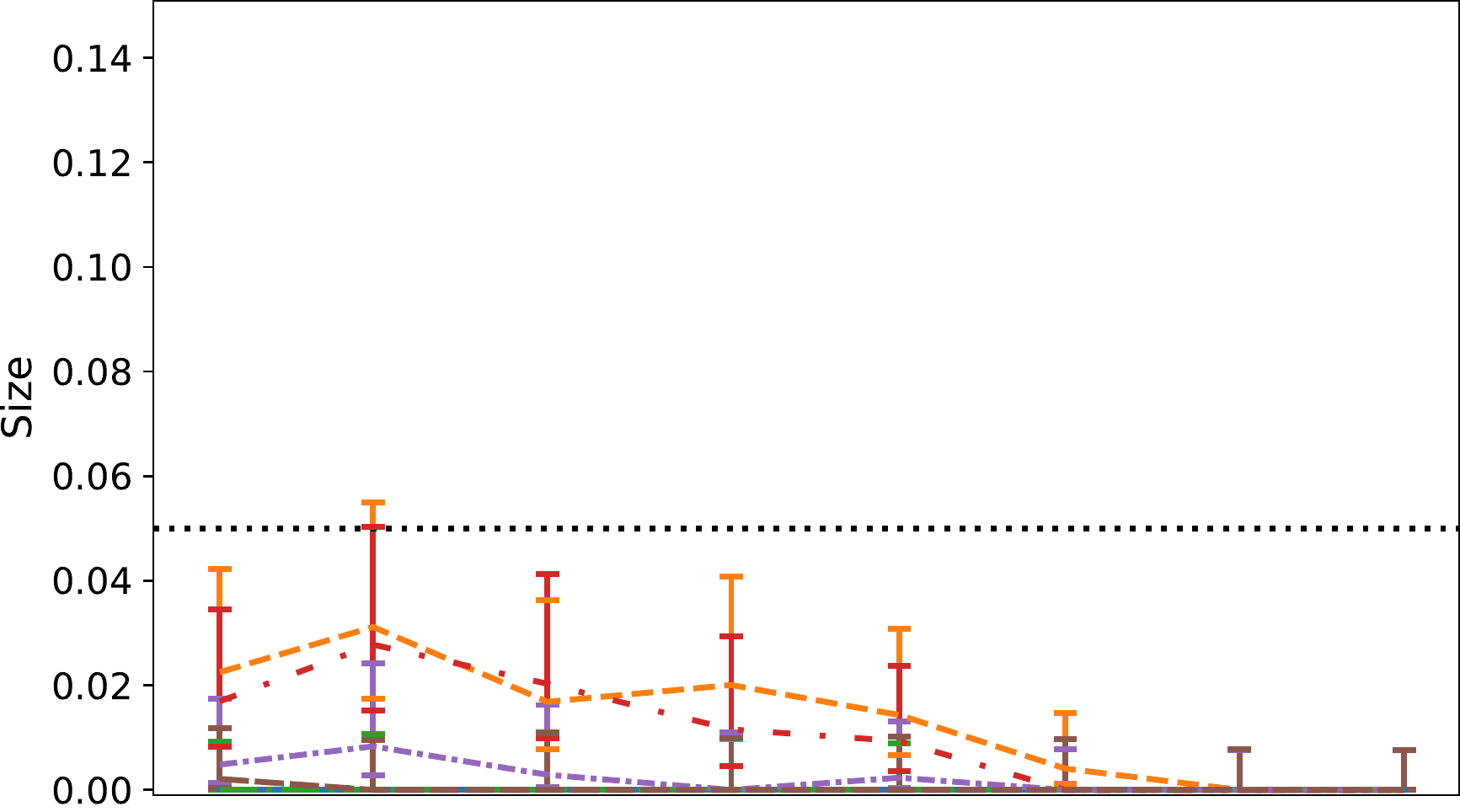}
    \end{subfigure}\hspace{\imgspace\linewidth}%
     \begin{subfigure}{\subfigfracin\linewidth}
             \includegraphics[width=\imgfrac\linewidth]{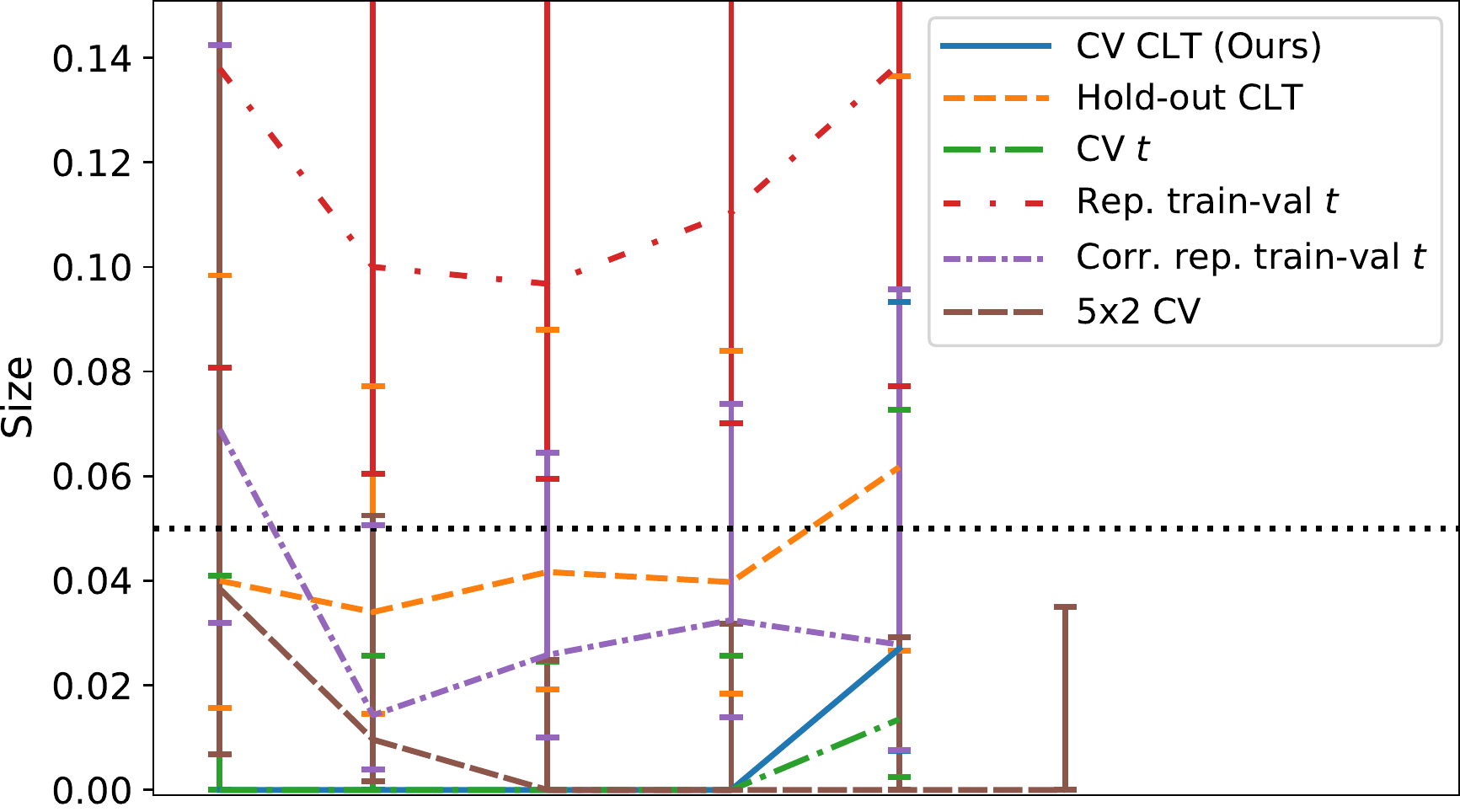}
    \end{subfigure}

    \vspace{\vgap\linewidth}
    
    \begin{subfigure}{\subfigfracin\linewidth}
        \includegraphics[width=\imgfrac\linewidth]{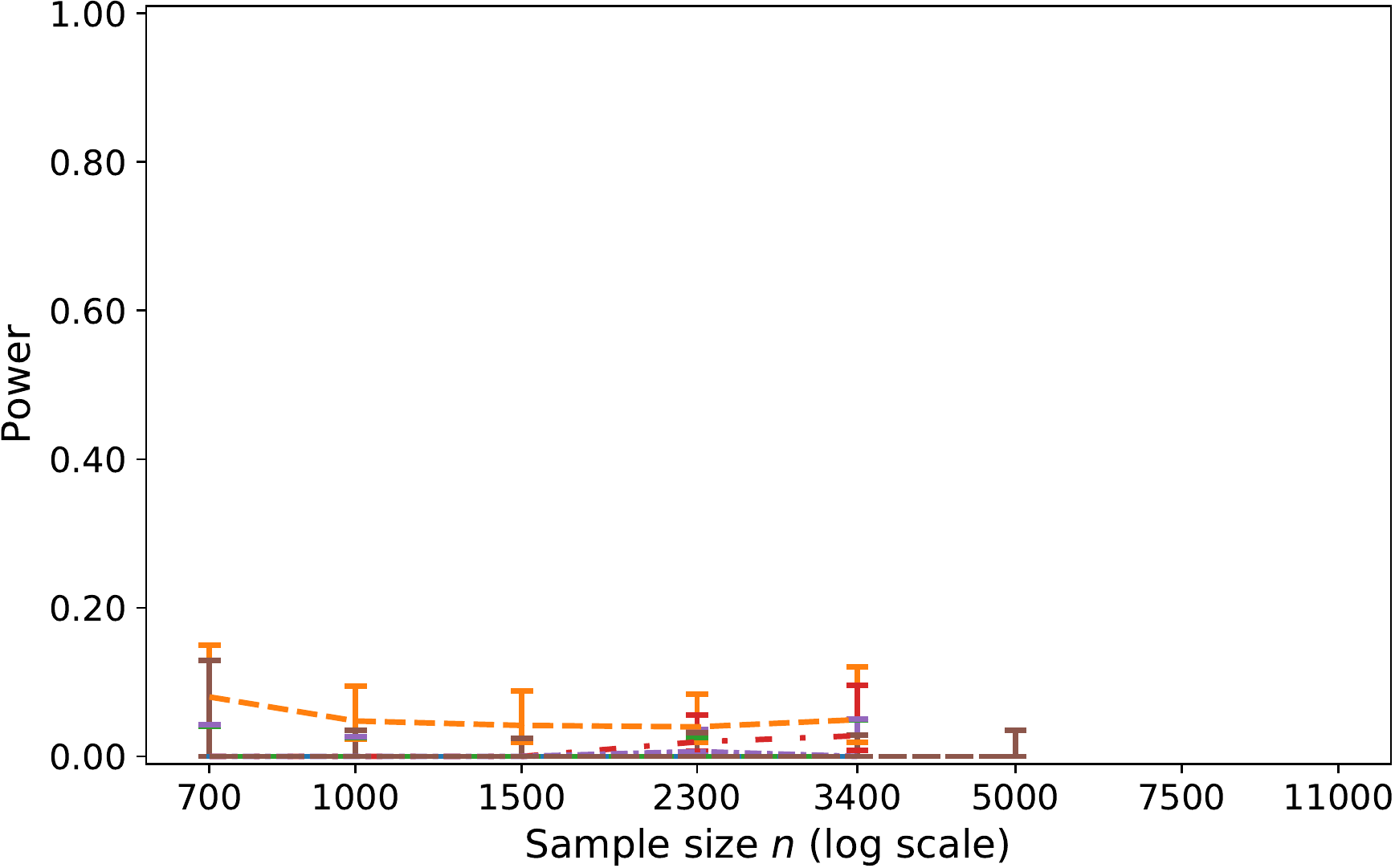}
    \end{subfigure}\hspace{\imgspace\linewidth}%
    \begin{subfigure}{\subfigfracin\linewidth}
        \includegraphics[width=\imgfrac\linewidth]{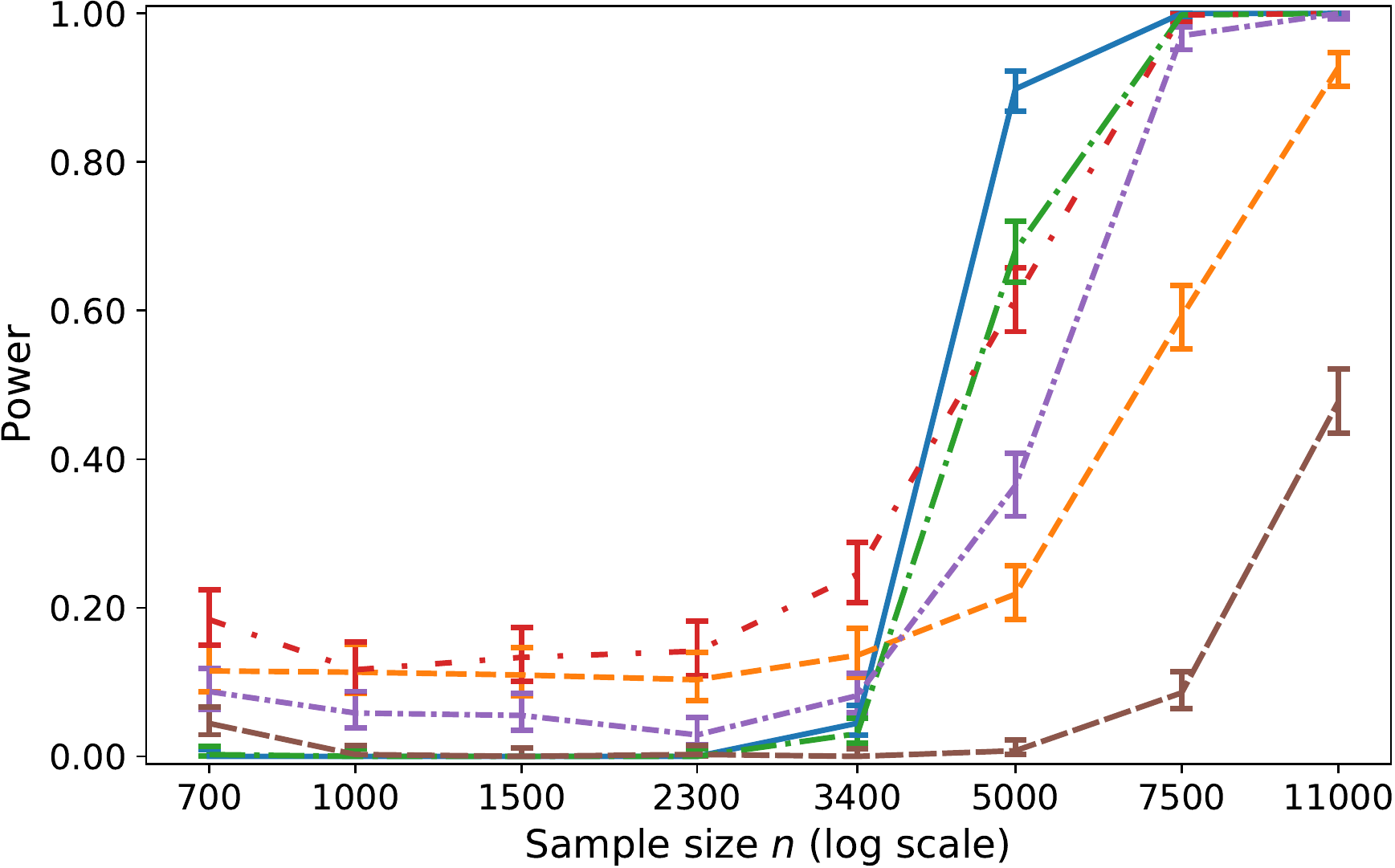}
    \end{subfigure}

    \caption{Size (top) and power (bottom) of level-$0.05$ tests for improved test error (see \cref{sec:sim:test}). \tbf{Left}: Testing $H_1$: neural network improves upon ridge regression. \tbf{Right}: Testing $H_1$: ridge regression improves upon neural network.}
    \label{fig:test-error-improvement-reg-RR-NN}
\end{figure}

\begin{figure}[h!]
\centering
    \begin{subfigure}{\subfigfracin\linewidth}
        \includegraphics[width=\imgfrac\linewidth]{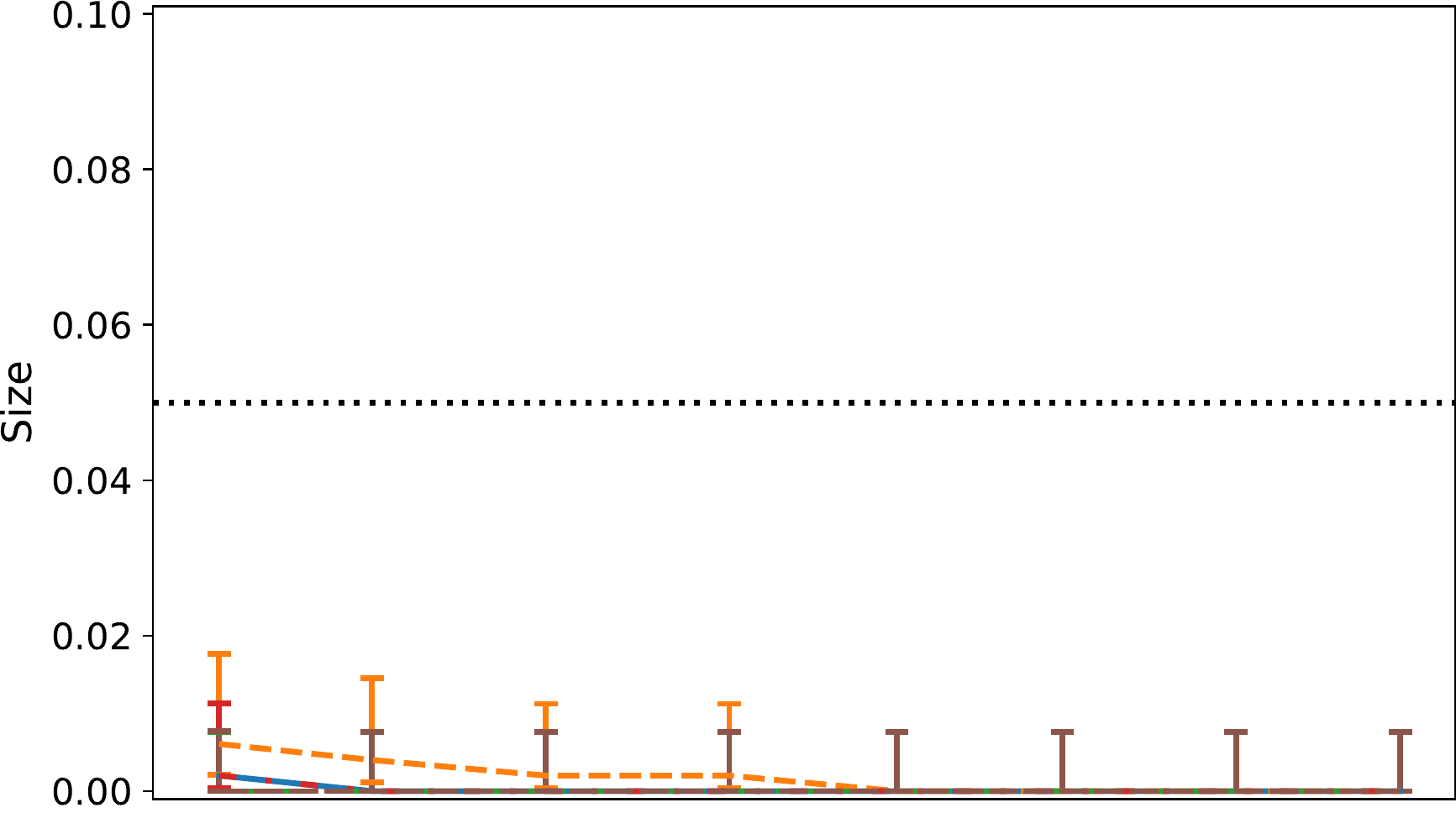}
    \end{subfigure}\hspace{\imgspace\linewidth}%
     \begin{subfigure}{\subfigfracin\linewidth}
             \includegraphics[width=\imgfrac\linewidth]{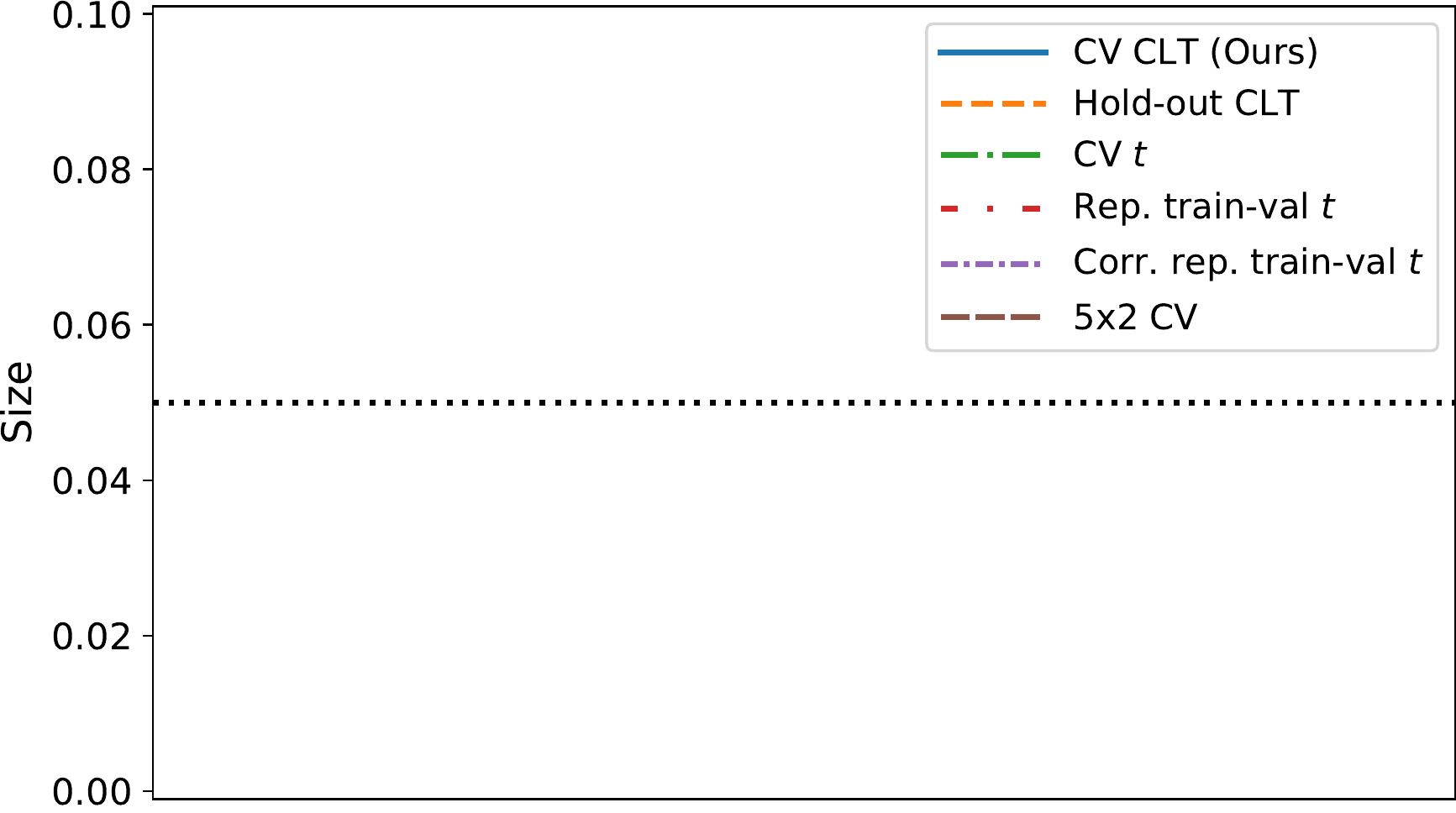}
    \end{subfigure}

    \vspace{\vgap\linewidth}
    
    \begin{subfigure}{\subfigfracin\linewidth}
        \includegraphics[width=\imgfrac\linewidth]{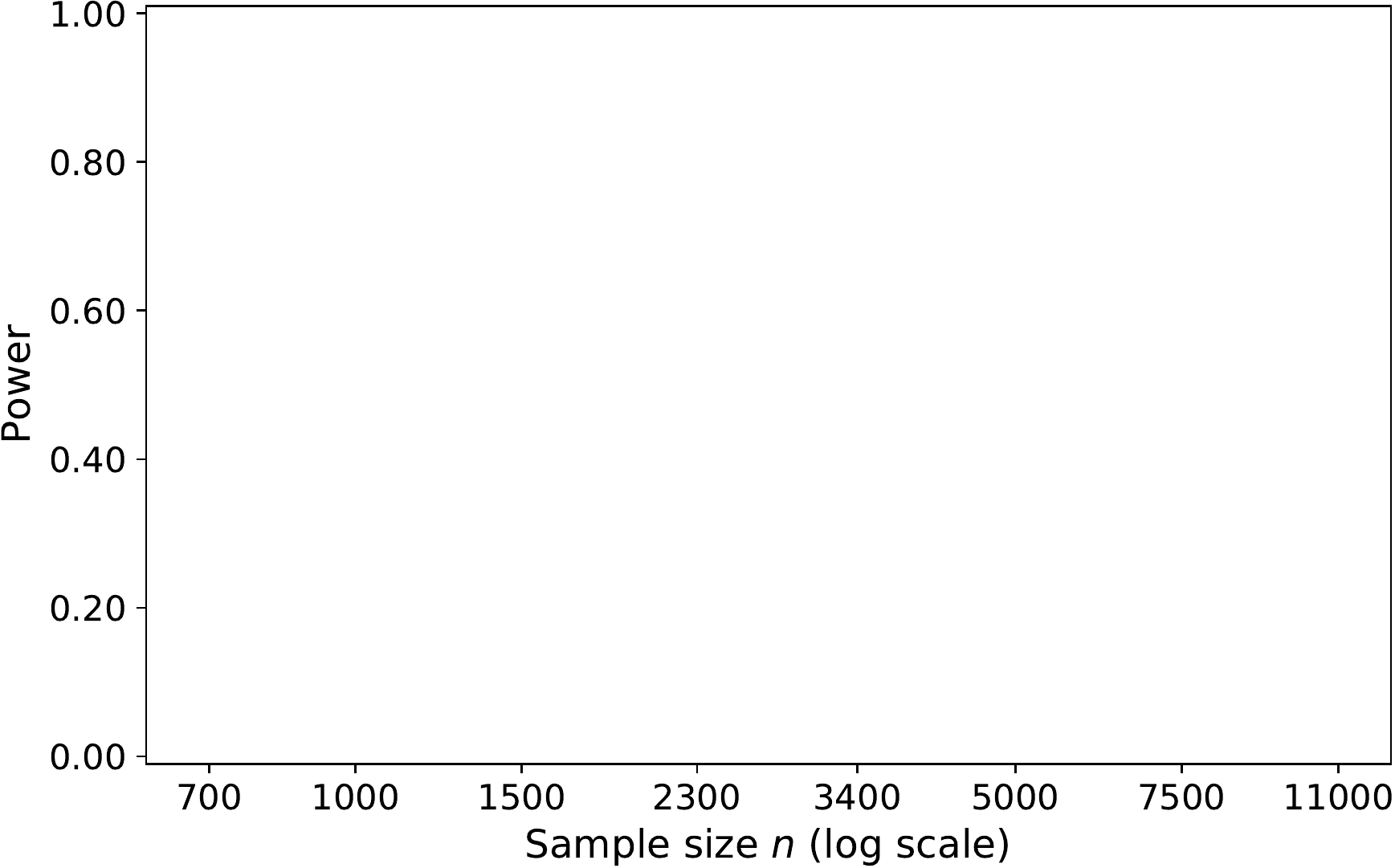}
    \end{subfigure}\hspace{\imgspace\linewidth}%
    \begin{subfigure}{\subfigfracin\linewidth}
        \includegraphics[width=\imgfrac\linewidth]{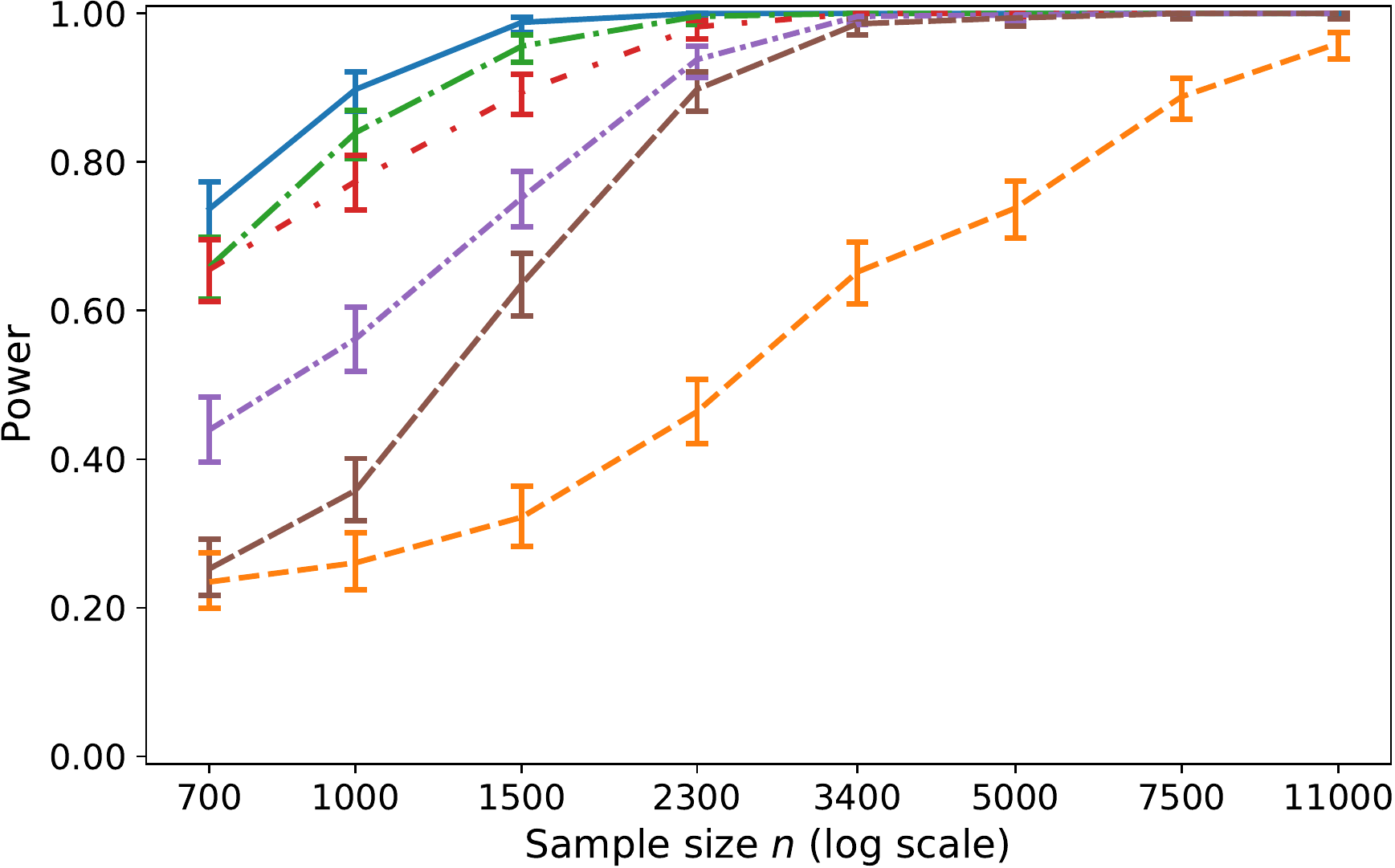}
    \end{subfigure}

    \caption{Size (top) and power (bottom) of level-$0.05$ tests for improved test error (see \cref{sec:sim:test}). \tbf{Left}: Testing $H_1$: neural network regression improves upon random forest. \tbf{Right}: Testing $H_1$: random forest regression improves upon neural network.}
    \label{fig:test-error-improvement-reg-RF-NN}
\end{figure}

\begin{figure}[h!]
\centering
    \begin{subfigure}{\subfigfracin\linewidth}
        \includegraphics[width=\imgfrac\linewidth]{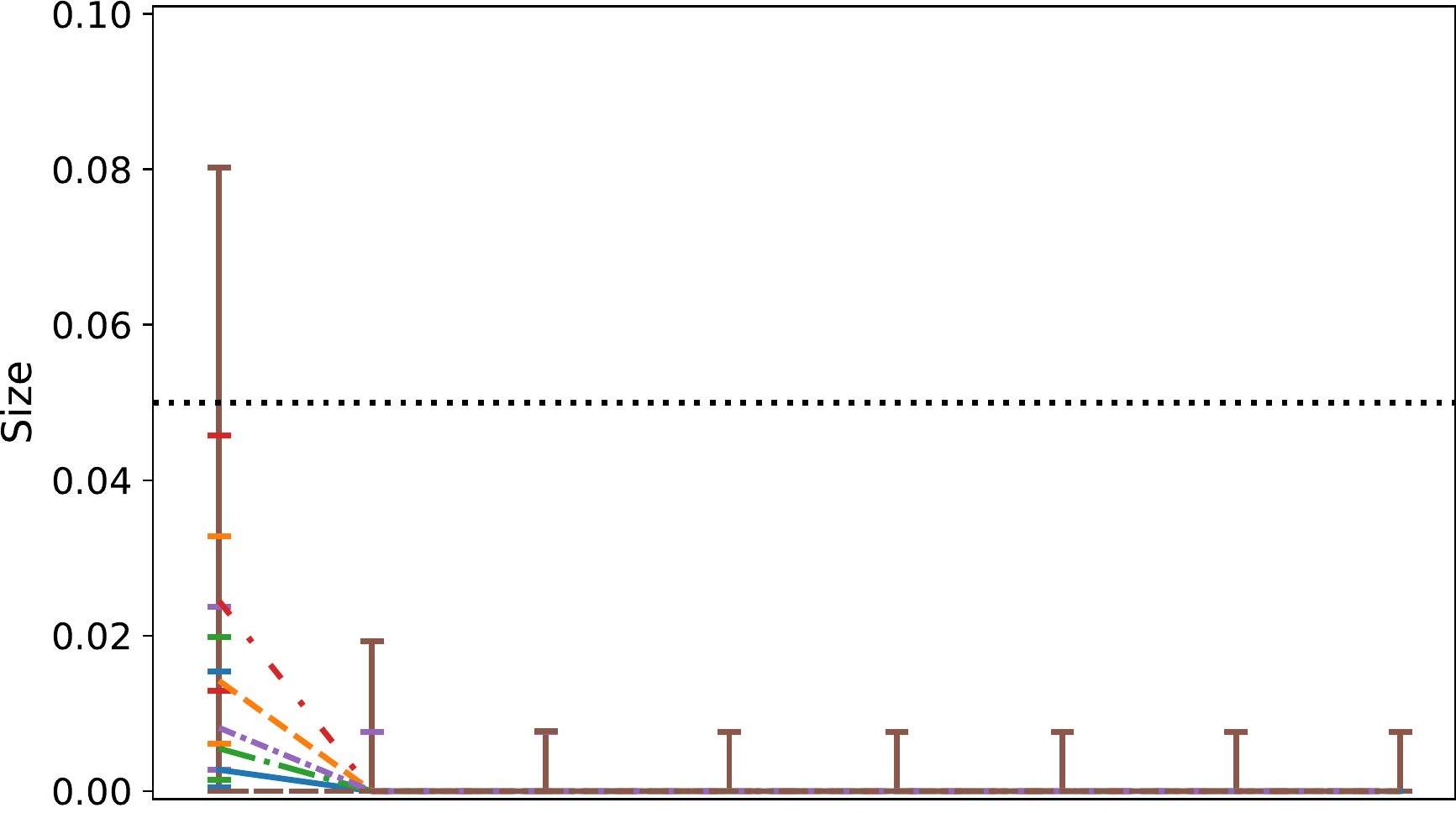}
    \end{subfigure}\hspace{\imgspace\linewidth}%
     \begin{subfigure}{\subfigfracin\linewidth}
             \includegraphics[width=\imgfrac\linewidth]{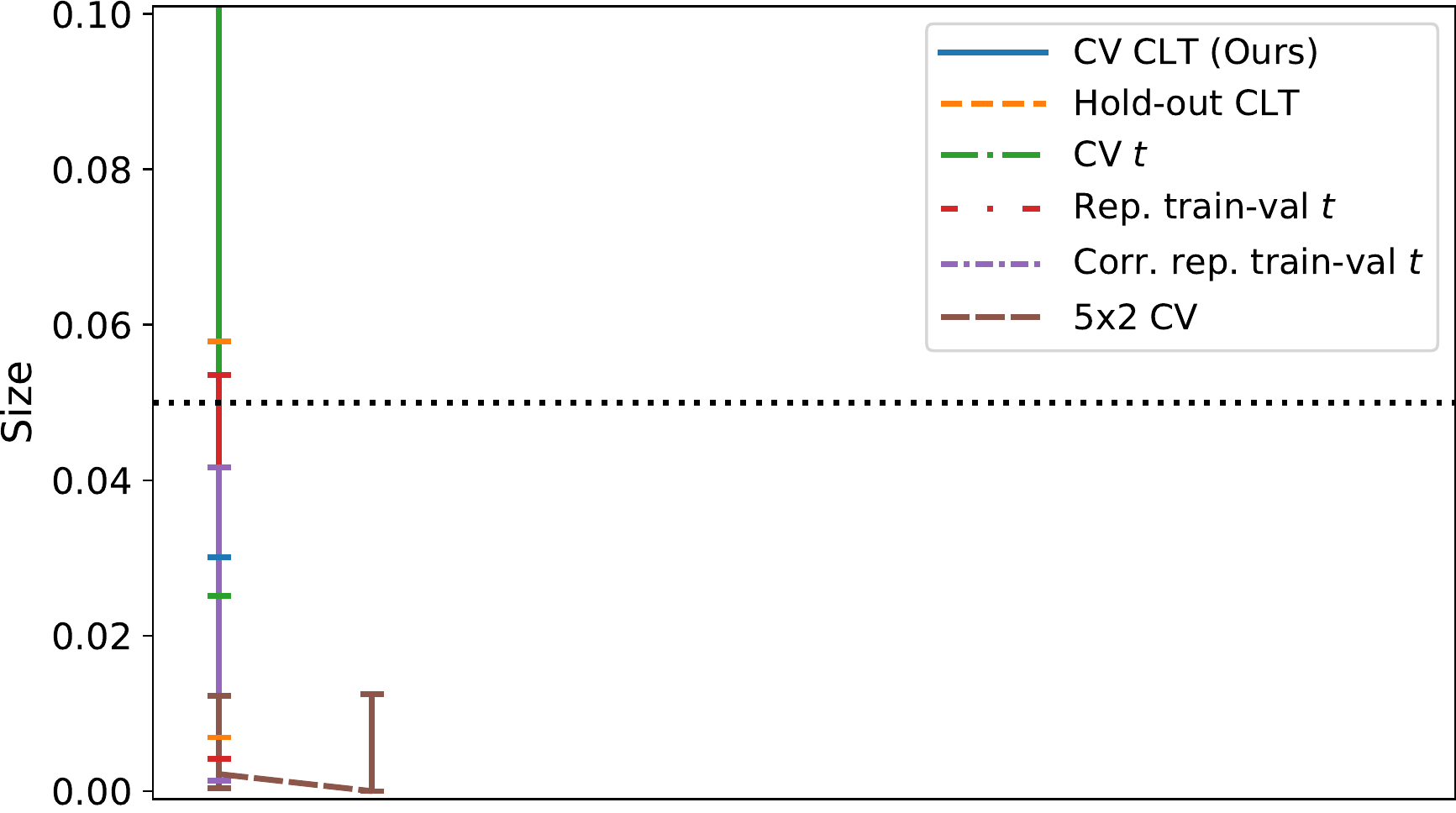}
    \end{subfigure}

    \vspace{\vgap\linewidth}
    
    \begin{subfigure}{\subfigfracin\linewidth}
        \includegraphics[width=\imgfrac\linewidth]{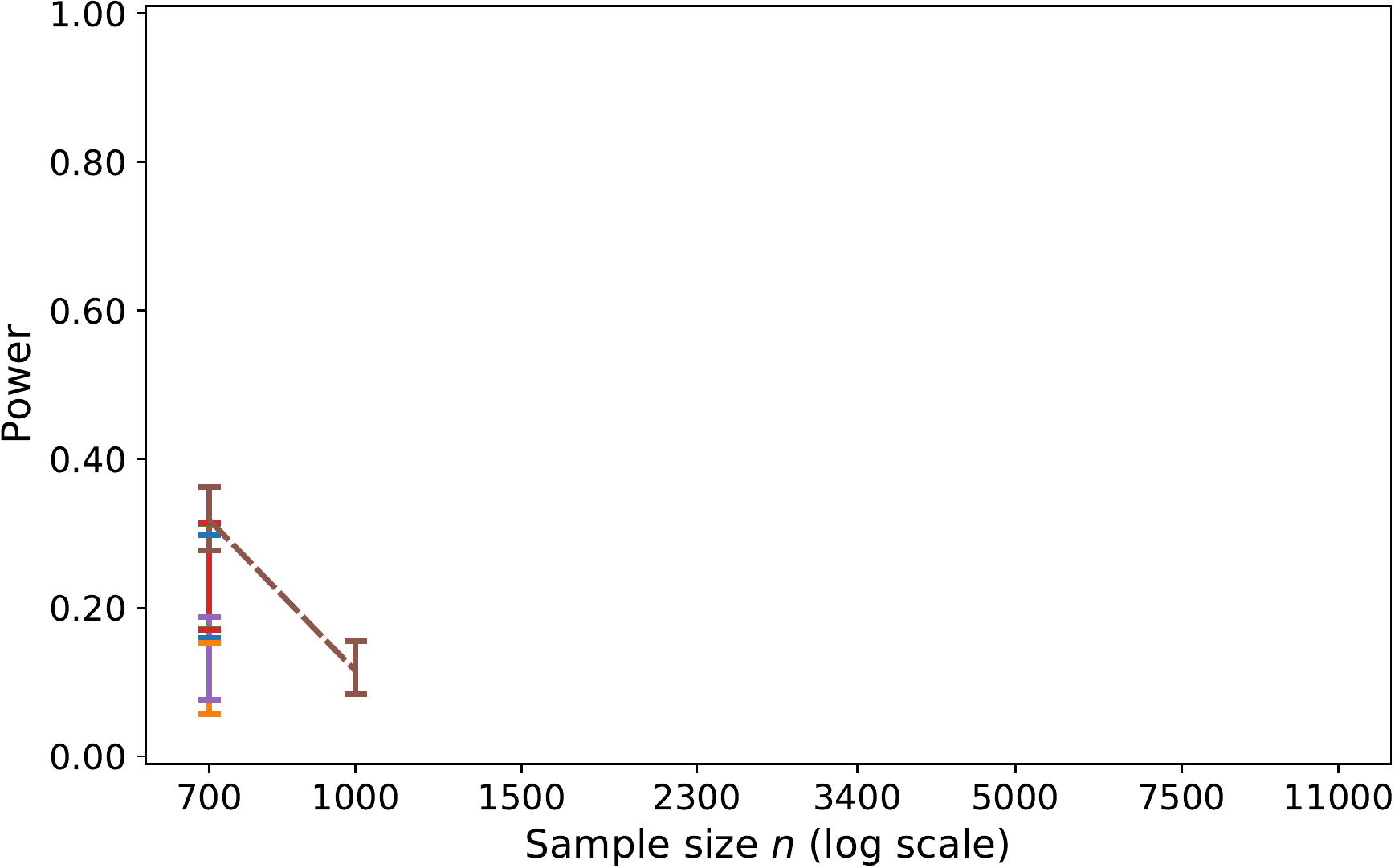}
    \end{subfigure}\hspace{\imgspace\linewidth}%
    \begin{subfigure}{\subfigfracin\linewidth}
        \includegraphics[width=\imgfrac\linewidth]{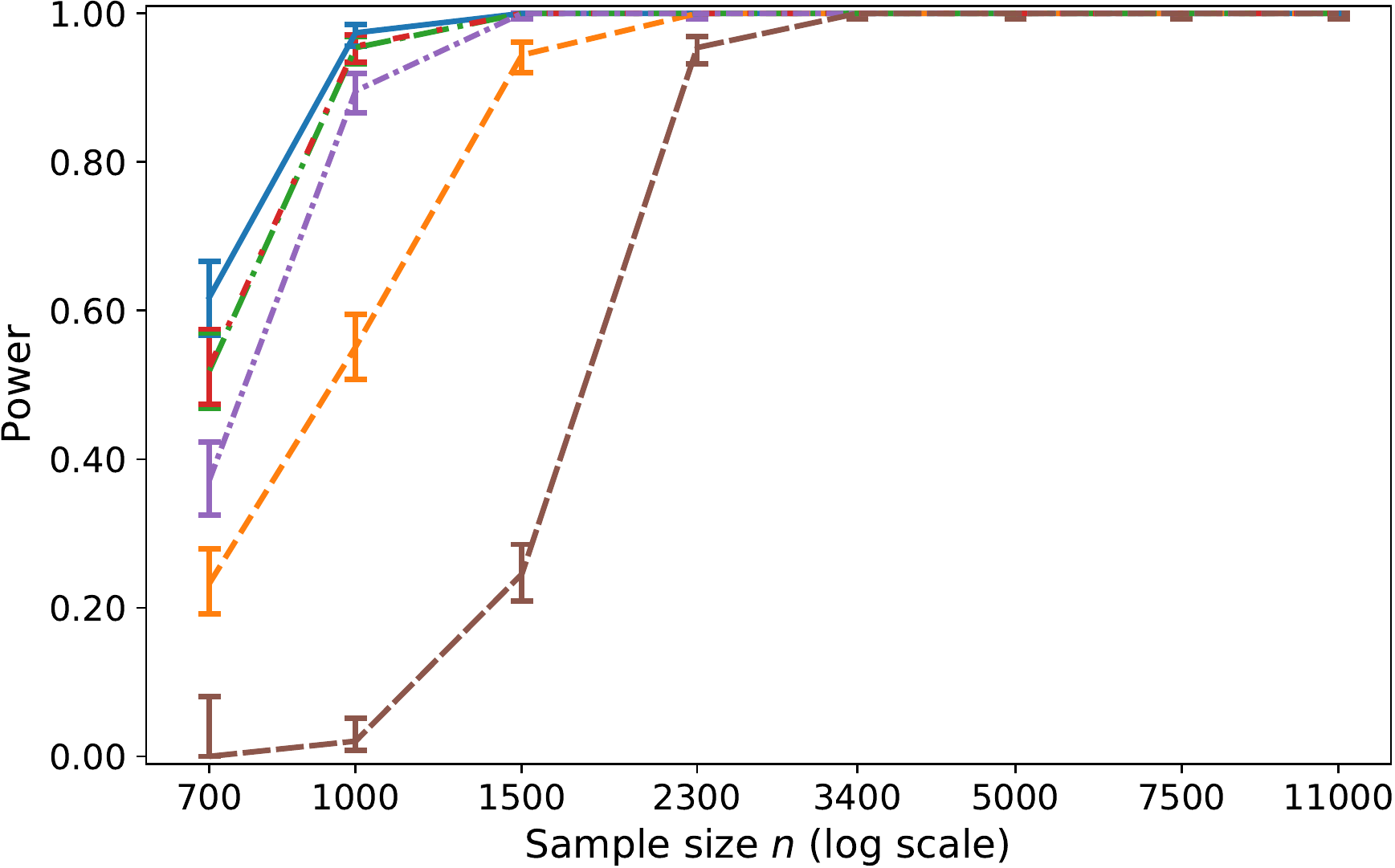}
    \end{subfigure}

    \caption{Size (top) and power (bottom) of level-$0.05$ tests for improved test error \tbf{with synthetic logistic regression labels} (see \cref{sec:synthetic-test}). \tbf{Left}: Testing $H_1$: random forest improves upon $\ell^2$-regularized logistic regression classifier. \tbf{Right}: Testing $H_1$: $\ell^2$-regularized logistic regression classifier improves upon random forest.}
    \label{fig:test-error-improvement-clf-RF-LR-simul}
\end{figure}

\subsection{Results for \cref{sec:import-stab}: Importance of stability}\label{sec:appendix-importance-stability}
In this section, we provide the figures (\cref{fig:bad-plots,fig:unstable-single-algos-plots,fig:comp-var-all}) and experimental details supporting the importance of stability experiment of \cref{sec:import-stab}.
Compared to the chosen hyperparameters described in \cref{sec:num-exp-appendix}, for this example, we used the default value of \texttt{max\_depth} for \texttt{XGBRFRegressor}, that is 6, and the default value of \texttt{alpha} for \texttt{MLPRegressor}, that is 1e-4. 
For \cref{fig:comp-var-to-one,fig:comp-var-to-one-single}, we obtain an estimate of $\sigma_n^2 = \Var(\bar h_n(Z_0))$ by computing a Monte Carlo approximation of $\bar h_n(Z_0) = \E[h_n(Z_0,Z_{1:\ntrain}) \mid Z_0]$ for each of 10,000 $Z_0$ values and then reporting the empirical variance of these 10,000 approximated values. For each value of $Z_0$ we employ the Monte Carlo approximation of 
\begin{align}
    \bar h_n(Z_0)
    \approx
    \frac{1}{500} \sum_{\ell = 1}^{500} \frac{1}{k}\sum_{j=1}^k h_n(Z_0, Z^{(\ell)}_{B_j})
\end{align}
where $(Z_{1:n}^{(\ell)})_{\ell=1}^{500}$ are the $500$  datasets of size $n$ described in \cref{sec:num-exp-appendix}.

\begin{figure}[h!]
\centering
    \begin{subfigure}{\subfigfracin\linewidth}
        \includegraphics[width=\imgfrac\linewidth]{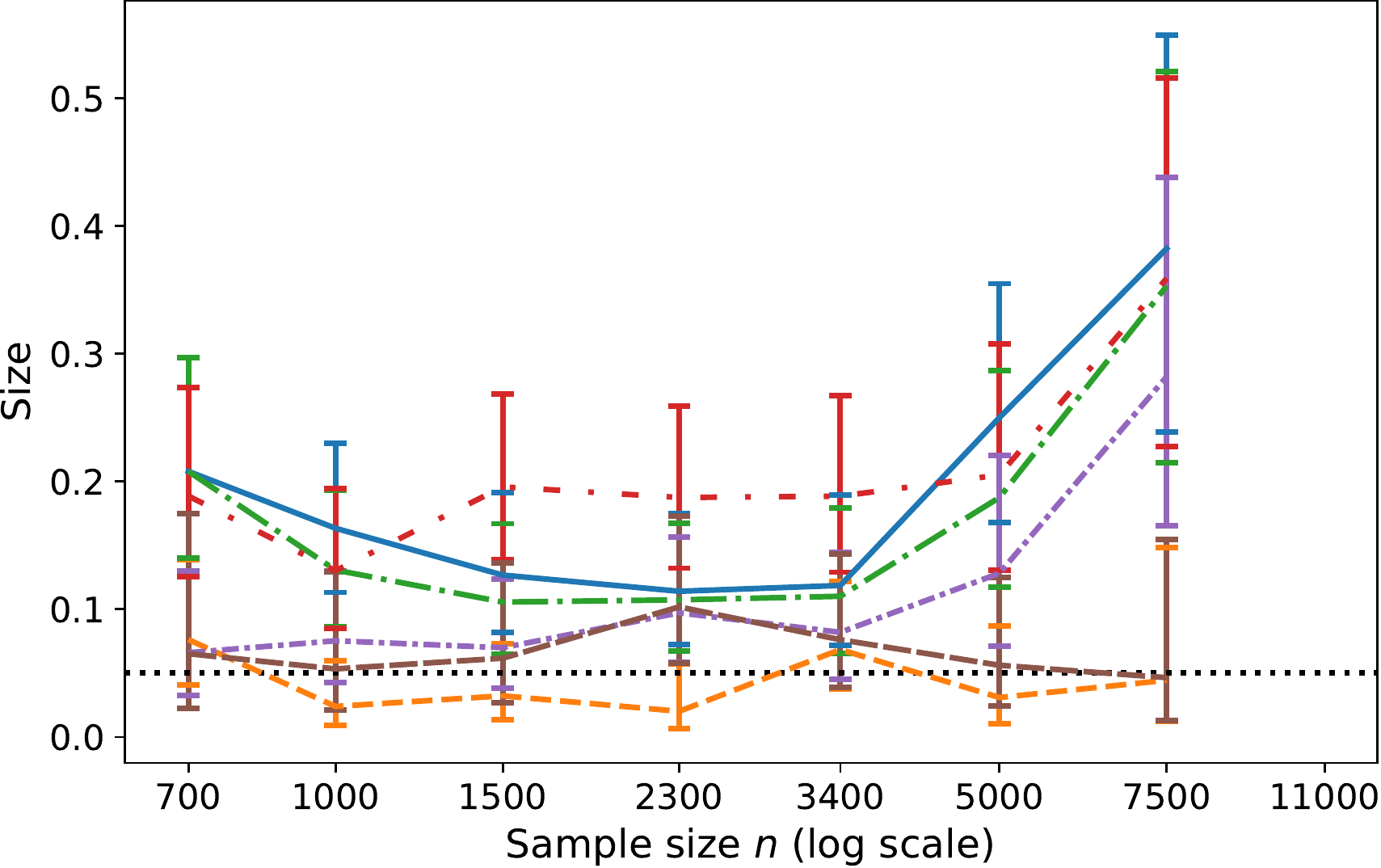}
    \end{subfigure}\hspace{\imgspace\linewidth}%
        \begin{subfigure}{\subfigfracin\linewidth}
        \includegraphics[width=\imgfrac\linewidth]{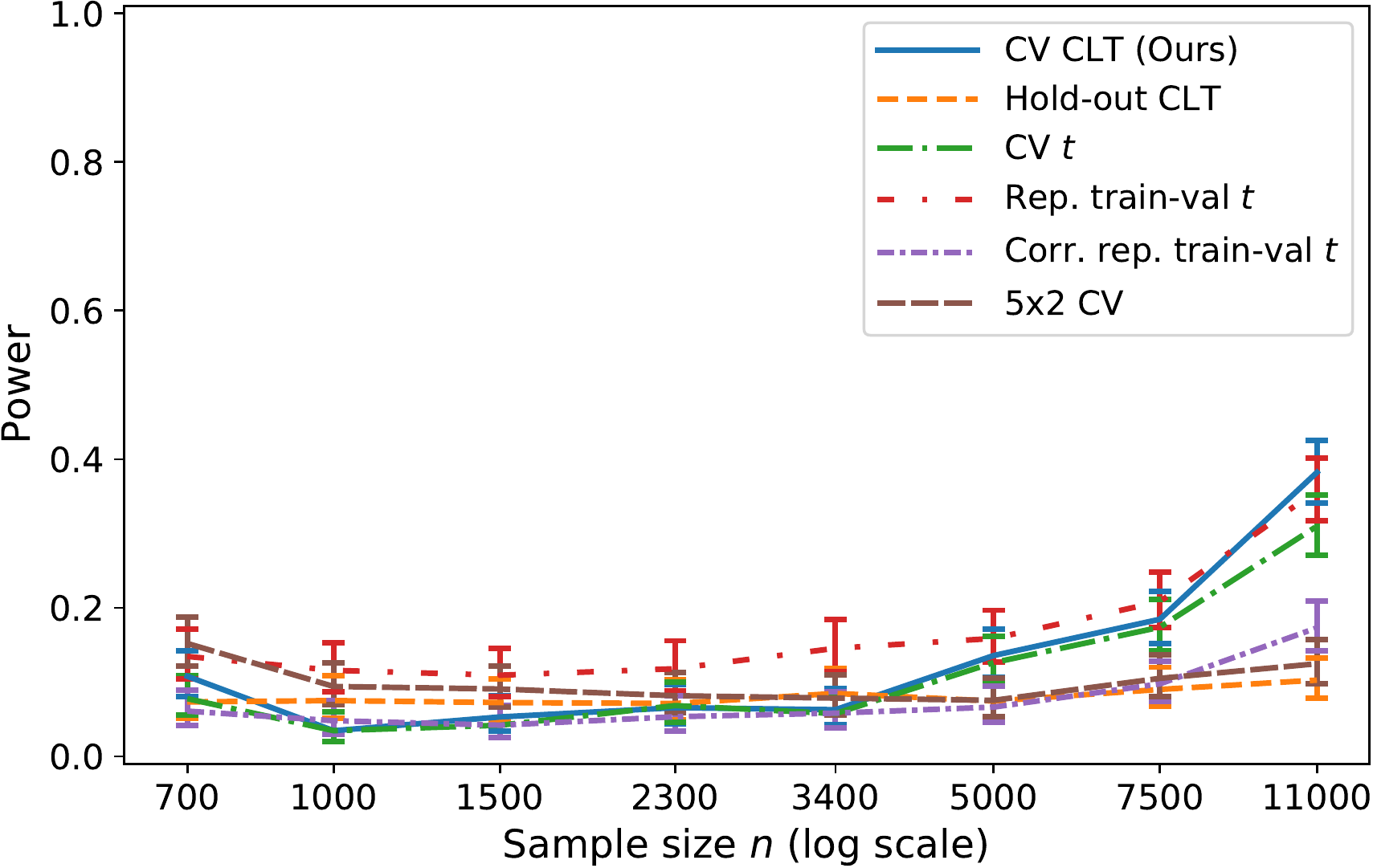}
    \end{subfigure}

    \caption{\tbf{Impact of instability} on size (left) and power (right) of level-$0.05$ tests for improved test error (see \cref{sec:import-stab}). Testing $H_1$: less stable neural network regression improves upon less stable random forest.}
    \label{fig:bad-plots}
\end{figure}

\begin{figure}[tb!]
\centering
    \begin{subfigure}{\subfigfracin\linewidth}
        \includegraphics[width=\imgfrac\linewidth]{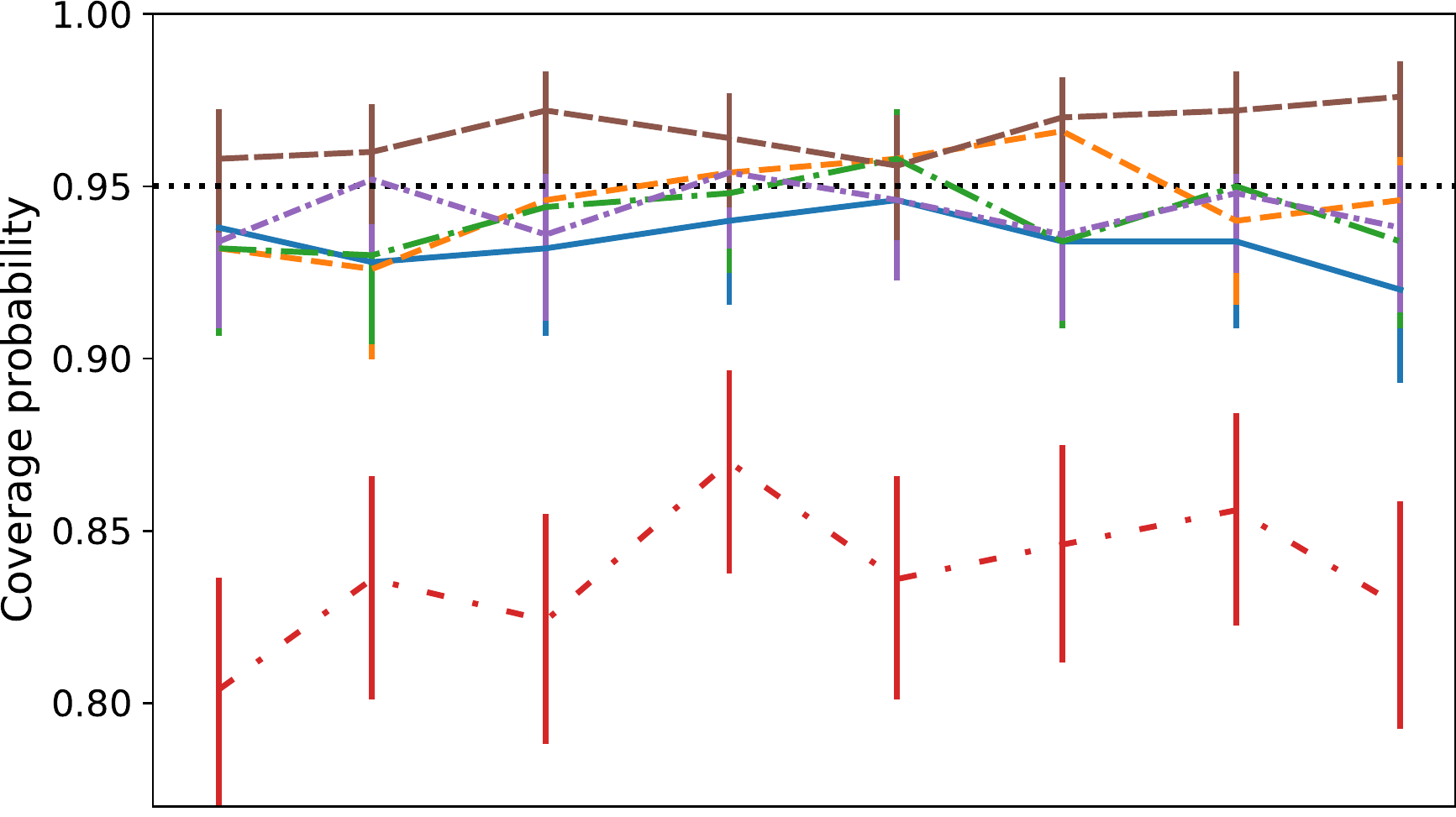}
    \end{subfigure}\hspace{\imgspace\linewidth}%
     \begin{subfigure}{\subfigfracin\linewidth}
             \includegraphics[width=\imgfrac\linewidth]{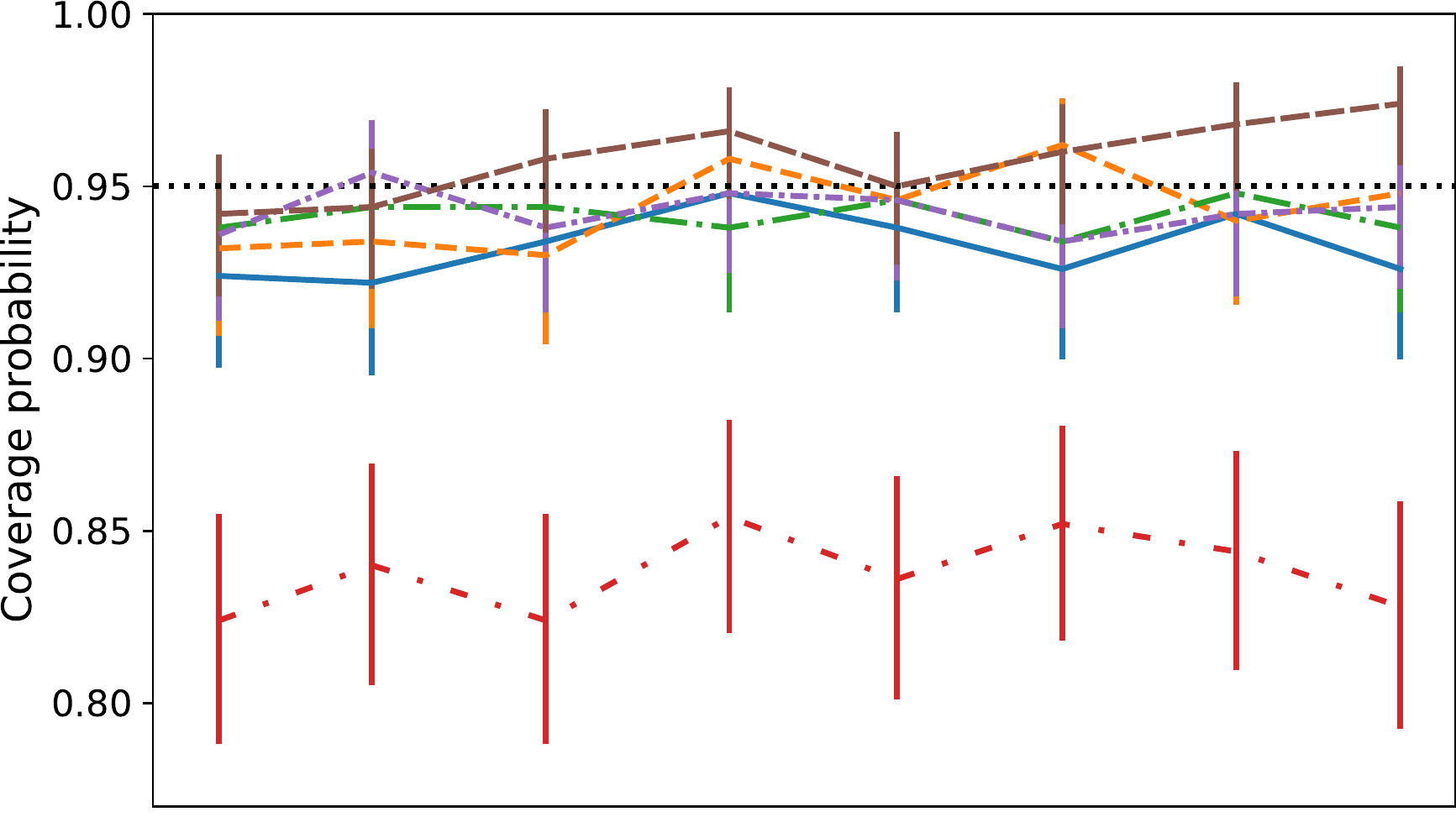}
    \end{subfigure}

    \vspace{\vgap\linewidth}
    \begin{subfigure}{\subfigfracin\linewidth}
        \includegraphics[width=\imgfrac\linewidth]{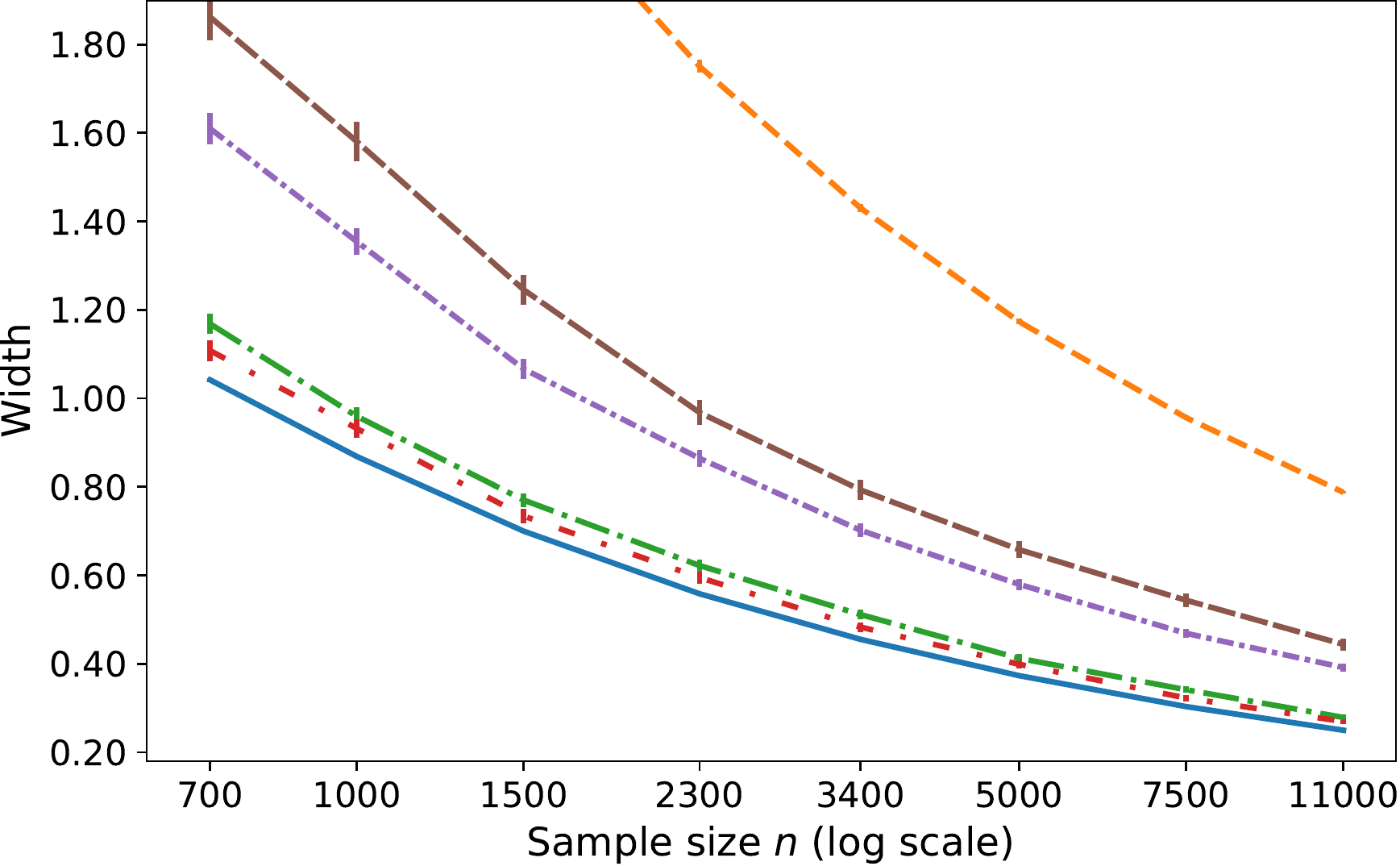}
    \end{subfigure}\hspace{\imgspace\linewidth}%
    \begin{subfigure}{\subfigfracin\linewidth}
        \includegraphics[width=\imgfrac\linewidth]{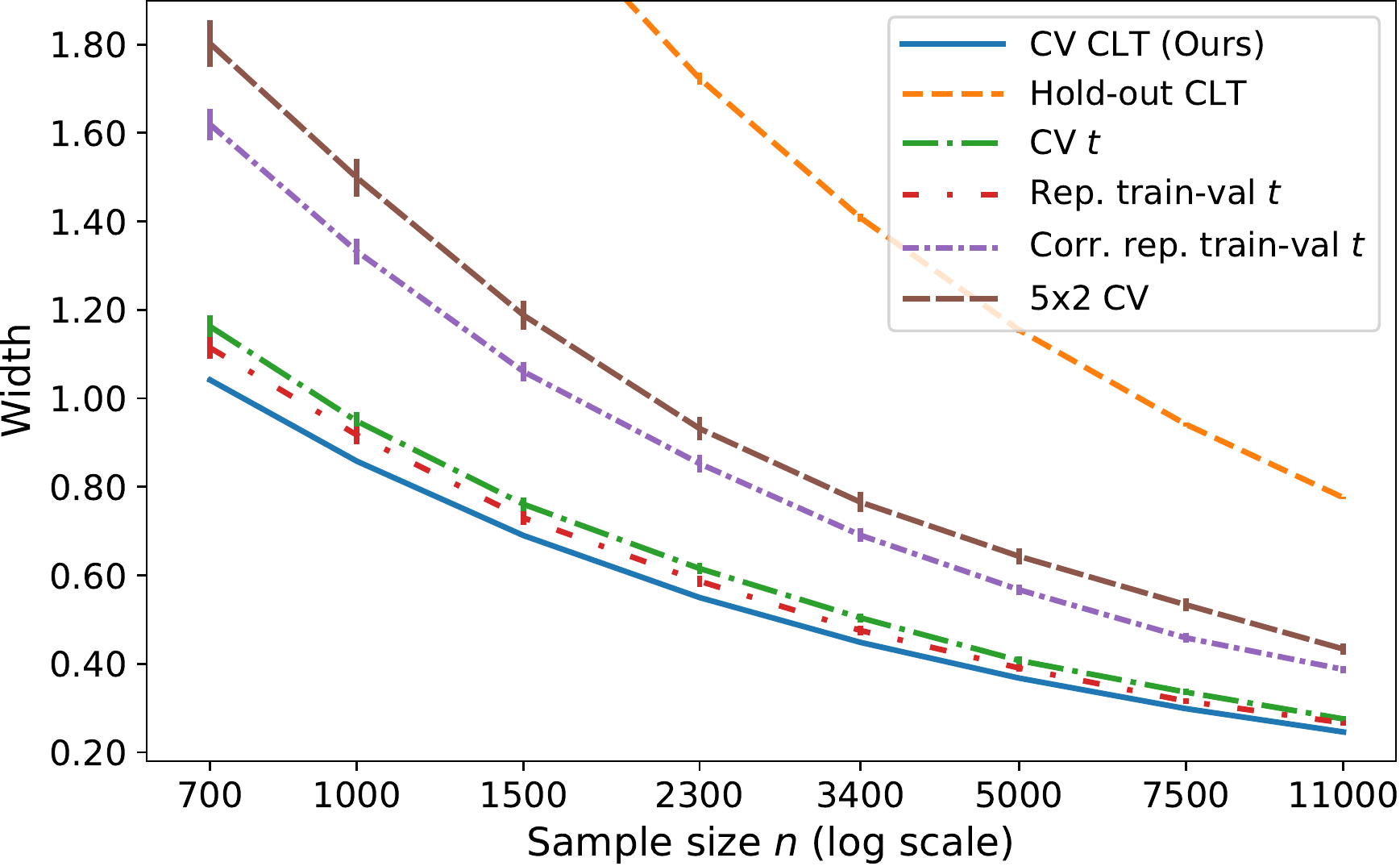}
    \end{subfigure}
    \caption{\tbf{Impact of instability} on test error coverage (top) and width (bottom) of $95\%$ confidence intervals (see \cref{sec:import-stab}). \tbf{Left:} Less stable neural network regression. \tbf{Right:} Less stable random forest regression.}
    \label{fig:unstable-single-algos-plots}
\end{figure}

\begin{figure}[h!]
\centering
\begin{subfigure}{\subfigfracin\linewidth}
\includegraphics[width=\imgfrac\linewidth]{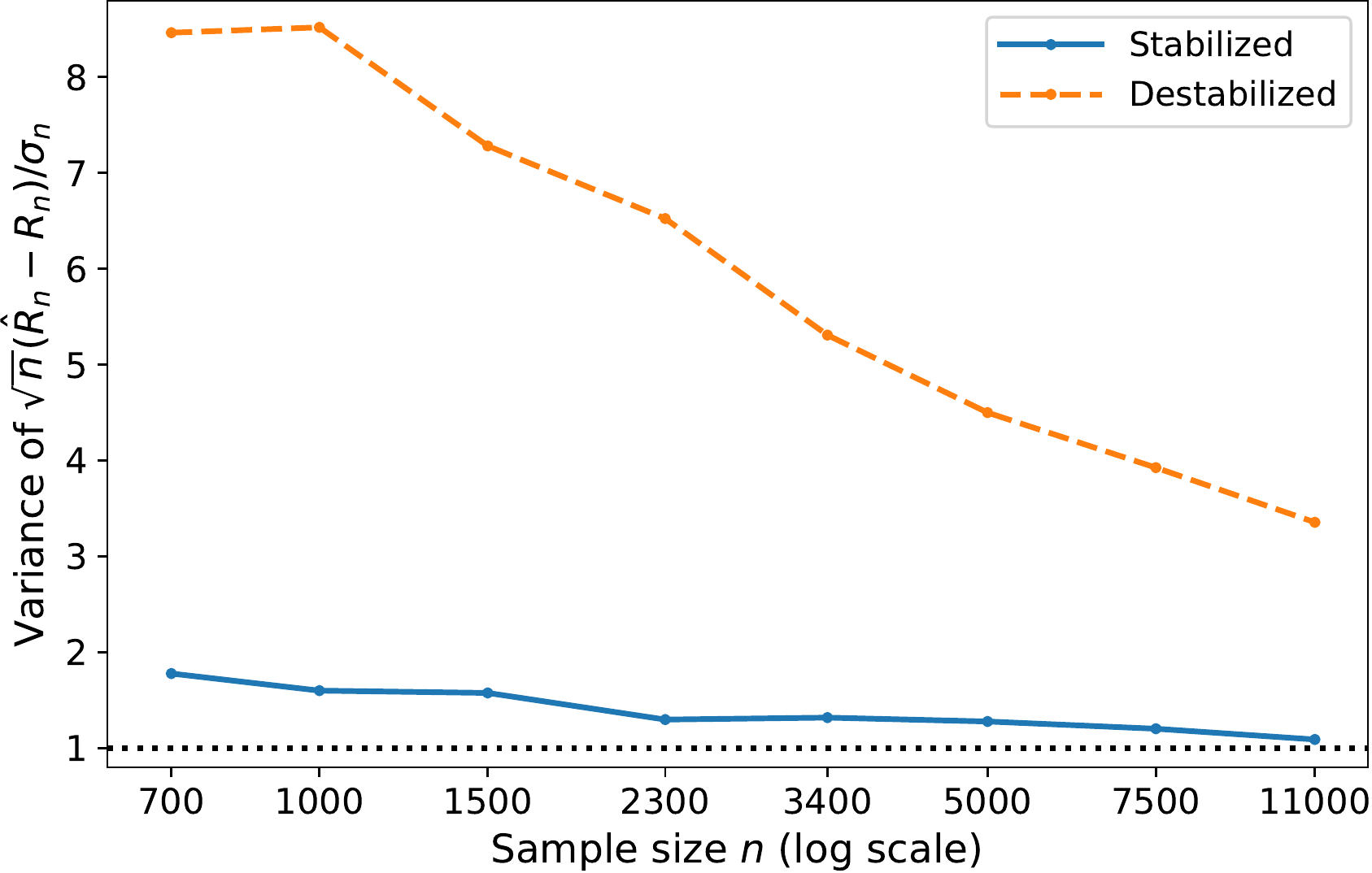}
\caption{{Algorithm comparison} 
}
\label{fig:comp-var-to-one}
\end{subfigure}\hspace{\imgspace\linewidth}%
\begin{subfigure}{\subfigfracin\linewidth}
\includegraphics[width=\imgfrac\linewidth]{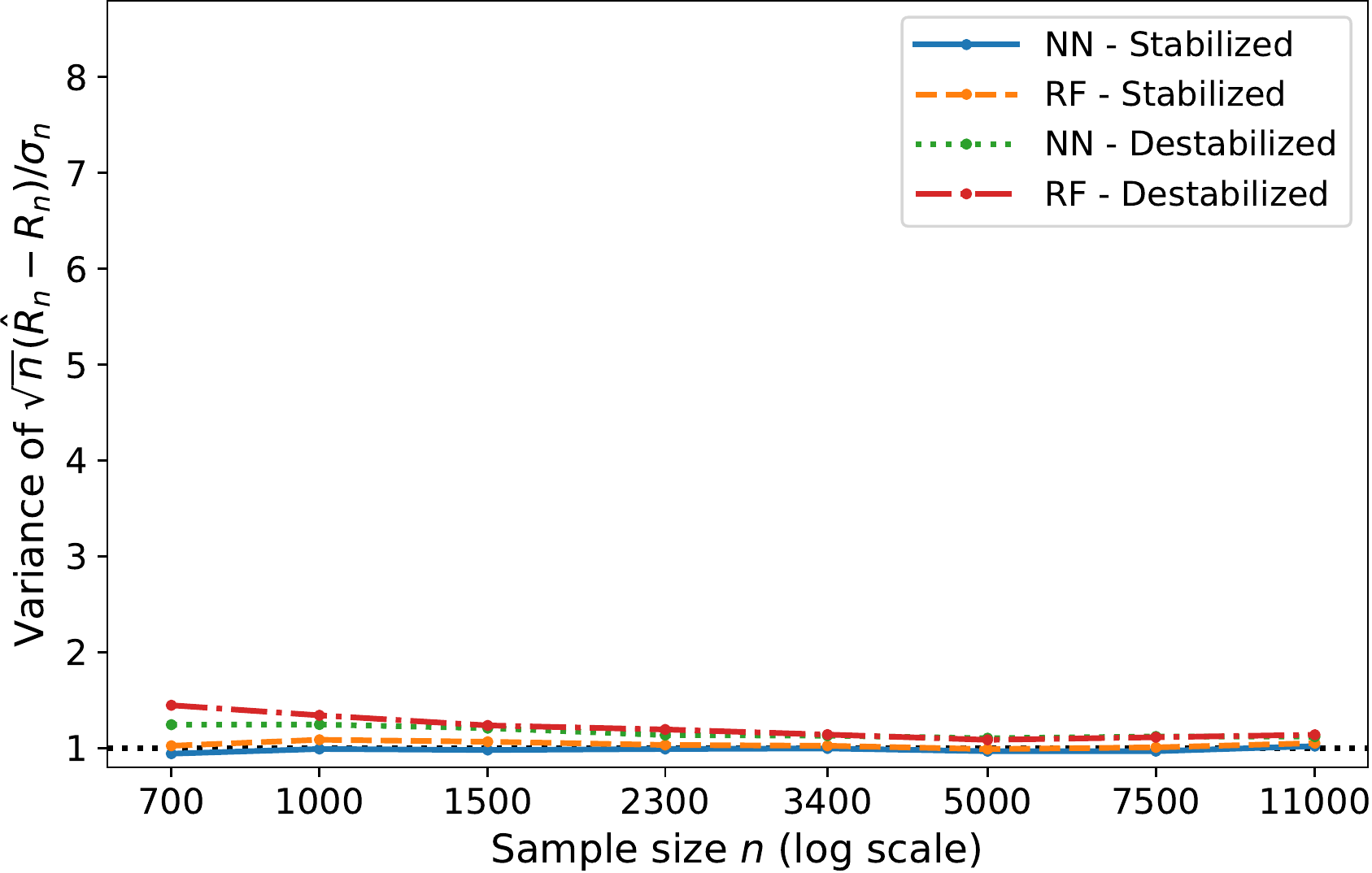}
\caption{
{Single algorithm assessment} 
}
\label{fig:comp-var-to-one-single}
\end{subfigure}
\caption{\tbf{Impact of instability} on variance of $\frac{\sqrt{n}}{\sigma_n}(\Rhat - \Rcondcv)$ (see \cref{sec:import-stab}).
\tbf{Left:} $h_n(Z_0, Z_B) = (Y_0 - \hat f_1(X_0; Z_B))^2 - (Y_0 - \hat f_2(X_0; Z_B))^2$
for neural network and random forest prediction rules, $\hat{f_1}$ and $\hat{f_2}$.
As predicted in \cref{iid-cv-normal,asymp-from-lstability}, the variance is close to $1$ when $h_n$ is stable, but the variance can be much larger when $h_n$ is unstable.
\tbf{Right:} $h_n(Z_0, Z_B) = (Y_0 - \hat f(X_0; Z_B))^2$
for neural network or random forest prediction rule, $\hat{f}$.
The same destabilized algorithms produce relatively stable $h_n$ in the context of single algorithm assessment, as the variance parameter $\sigma_n^2 = \Var(\bar{h}_n(Z_0))$ is larger.
}
\label{fig:comp-var-all}
\end{figure}

\end{document}